\documentclass[twoside,11pt]{article}
\usepackage{amsthm}
\PassOptionsToPackage{sort}{natbib}
\usepackage[hyperref]{jmlr2e}

\usepackage{amsmath,amsfonts,amssymb}
\usepackage{textcomp}
\usepackage{mathrsfs} 
\usepackage{mathtools}
\usepackage{dsfont} 
\usepackage{bm}
\usepackage{xurl,xcolor,enumerate}
\hypersetup{pdfauthor=Shijun Zhang,
pdftitle={Deep Network Approximation: Beyond ReLU to Diverse Activation Functions}
}
\usepackage{tabularx,multirow,array,booktabs}
\usepackage{colortbl}
\usepackage{graphicx}
\graphicspath{{figures/}}%
\usepackage{float}
\usepackage{subcaption} 
\usepackage{etoolbox,xparse,dsfont,bm,amssymb}
\newcommand{\myMFabc}[4]{\expandafter#1\csname#3#4\endcsname{{#2{#4}}}}
\newcommand{\myMFcmd}[4]{\expandafter#1\csname#3#4\endcsname{{#2{\csname#4\endcsname}}}}

\newcommand{\MFabc}[3][\newcommand]{
    \def\doOld##1##2{\forcsvlist{\myMFabc{#1}{##1}{##2}}{#3}}
    \providecommand{\do}{do}
    \RenewDocumentCommand \do { >{\SplitList{,}} m } { \doOld##1 }
    \docsvlist{#2}
}
\newcommand{\MFcmd}[3][\newcommand]{
    \def\doOld##1##2{\forcsvlist{\myMFcmd{#1}{##1}{##2}}{#3}}
    \providecommand{\do}{do}
    \RenewDocumentCommand \do { >{\SplitList{,}} m } { \doOld##1 }
    \docsvlist{#2}
}

\newcommand{\bmzero}{{\bm{0}}}

\let\one\bbone
\MFabc{ {\mathfrak,frak}, }{T,u,U,o,O,E,m}
\MFabc{ {\mathbb, }, }{A,N,Z,R,C,Q}

\newcommand{\hatbm}[1]{\widehat{\bm{#1}}}
\newcommand{\tildebm}[1]{\widetilde{\bm{#1}}}
\newcommand{\bmcal}[1]{\bm{\mathcal{#1}}}
\newcommand{\caltilde}[1]{\mathcal{\widetilde{#1}}}
\newcommand{\calhat}[1]{\mathcal{\widehat{#1}}} 
\newcommand{\scrtilde}[1]{\widetilde{\mathscr{#1}\mspace{1mu}\mspace{-1mu}}}
\newcommand{\scrhat}[1]{\mathscr{\widehat{#1}\mspace{1mu}\mspace{-1mu}}}
\newcommand{\bmcalhat}[1]{\bm{\mathcal{\widehat{#1}}}}
\newcommand{\bmcaltilde}[1]{\bm{\mathcal{\widetilde{#1}}}}
\MFabc{ {\mathcal,cal}, {\mathscr,scr}, {\mathbb,bb}, {\bmcal,bmcal}, {\bmcal,calbm}, {\caltilde,caltilde}, {\caltilde,tildecal}, {\calhat,calhat}, {\calhat,hatcal}, {\scrtilde,scrtilde}, {\scrtilde,tildescr}, {\scrhat,scrhat}, {\scrhat,hatscr}, {\bmcalhat,bmcalhat}, {\bmcalhat,bmhatcal}, 
{\bmcalhat,calbmhat}, {\bmcalhat,calhatbm}, {\bmcalhat,hatbmcal}, {\bmcalhat,hatcalbm}, {\bmcaltilde,bmcaltilde}, {\bmcaltilde,bmtildecal}, {\bmcaltilde,calbmtilde}, {\bmcaltilde,caltildebm}, {\bmcaltilde,tildebmcal}, {\bmcaltilde,tildecalbm}}{A,B,C,D,E,F,G,H,I,J,K,L,M,N,O,P,Q,R,S,T,U,V,W,X,Y,Z}

\let\tildescrA\scrtildeA

\MFabc{{\bm,bm}, {\mathsf,sf}, {\widehat,hat}, {\widetilde,tilde}, {\hatbm,hatbm}, {\hatbm,bmhat}, {\tildebm,tildebm}, {\tildebm,bmtilde}}{a,b,c,d,e,f,g,h,i,j,k,l,m,n,o,p,q,r,s,t,u,v,w,x,y,z,A,B,C,D,E,F,G,H,I,J,K,L,M,N,O,P,Q,R,S,T,U,V,W,X,Y,Z}

\let\eps\varepsilon
\MFcmd{{\bm,bm}, {\widehat,hat}, {\widetilde,tilde}, {\tildebm, tildebm}, {\tildebm, bmtilde}, {\hatbm, hatbm}, {\hatbm, bmhat}}{alpha,beta,gamma,delta,epsilon,zeta,eta,theta,iota,kappa,lambda,mu,nu,xi,omicron,pi,rho,sigma,tau,upsilon,phi,chi,psi,omega,Alpha,Beta,Gamma,Delta,Epsilon,Zeta,Eta,Theta,Iota,Kappa,Lambda,Mu,Nu,Xi,Omicron,Pi,Rho,Sigma,Tau,Upsilon,Phi,Chi,Psi,Omega,varrho,varphi,vartheta,varepsilon,varsigma,ell,eps}

\newcommand{\actdef}[1]{\expandafter\def\csname#1\endcsname{{\ensuremath{\mathtt{#1}}}}}
\forcsvlist{\actdef}{ReLU, LReLU, LeakyReLU, ELU, GELU, SiLU, Softplus, dGELU, dSiLU, dSoftplus, Tanh, Sigmoid, Arctan, Softsign, SRS, dSRS, Swish, dSwish, Mish, dMish, SELU, CELU, dSELU}

\newlength{\myLength}

\newcommand{\nn}[6][]{\ensuremath{
		{\hspace{0.6pt}\mathcal{N}\hspace{-1.98pt}\mathcal{N}\hspace{-0.725pt}}_{\hspace{-0.75pt}#2\hspace{-0.02pt}}
		#1\{#3,\hspace{1.987pt} #4;\hspace{3.0297pt} \R^{#5}\hspace{-1.0298pt}\to\hspace{-0.98pt}\R^{#6}#1\}
}}

\newcommand{\nnOneD}[6][]{\ensuremath{
		{\hspace{0.6pt}\mathcal{N}\hspace{-1.98pt}\mathcal{N}\hspace{-0.725pt}}_{#2}
		#1\{#3,\hspace{1.987pt} #4;\hspace{3.0297pt} {#5}\hspace{-1.0298pt}\to\hspace{-0.98pt}{#6}#1\}
}}

\let\ts\intercal
\newcommand{\dprime}{{\prime\prime}}

\newcommand{\mycase}[2]{\par \vspace{0.25cm}\noindent\textbf{\hspace{8pt}Case }$#1\colon$ #2\par \vspace{0.18cm} \par}

\usepackage{tikz}
\newcommand{\myto}[2][1]{\mathop{
		\vcenter{\hbox{\scalebox{1}[#1]{\tikz{\draw[->,line width=0.72pt] (0,0.5) to (0.69*#2,0.5);}}}}
}}

\usepackage[left,mathlines]{lineno}
\usepackage{refcount}

\definecolor{mylinenumbercolor}{HTML}{BEBEBE}

\makeatletter
\newcommand*\patchAmsMathEnvironmentForLineno[1]{%
	\expandafter\let\csname old#1\expandafter\endcsname\csname #1\endcsname
	\expandafter\let\csname oldend#1\expandafter\endcsname\csname end#1\endcsname
	\renewenvironment{#1}%
	{\linenomath\csname old#1\endcsname}%
	{\csname oldend#1\endcsname\endlinenomath}}%
\newcommand*\patchBothAmsMathEnvironmentsForLineno[1]{%
	\patchAmsMathEnvironmentForLineno{#1}%
	\patchAmsMathEnvironmentForLineno{#1*}}%
\patchBothAmsMathEnvironmentsForLineno{equation}%
\patchBothAmsMathEnvironmentsForLineno{align}%
\patchBothAmsMathEnvironmentsForLineno{flalign}%
\patchBothAmsMathEnvironmentsForLineno{alignat}%
\patchBothAmsMathEnvironmentsForLineno{gather}%
\patchBothAmsMathEnvironmentsForLineno{multline}%
\makeatother
\makeatletter
\long\def\@makefntext#1{\@setpar{\@@par\@tempdima \hsize 
		\advance\@tempdima-15pt\parshape \@ne 15pt \@tempdima}\par
	\parindent 2em\noindent \hbox to \z@{\hss{\textsuperscript{\@thefnmark}} \hfil}#1}
\makeatother

\newcommand{\mailto}[2][]{\href{mailto:#2?cc=#1}{\color{black}#2}}

\usepackage{doi}

\let\epsilon\varepsilon
\let\eps\varepsilon

\let\tn\textnormal

\let\cdots\customcdots
\let\myforall\forall
\def\forall{{\myforall\, }}
\let\myexists\exists
\def\exists{{\myexists\, }}
\let\emptyset\varnothing

\definecolor{mygray}{RGB}{230,230,230}
\definecolor{myorange}{HTML}{ff7f0e}

\makeatletter
\renewenvironment{proof}[1][\proofname]{\par
    \pushQED{\qed}%
    \normalfont \topsep6\p@\@plus6\p@\relax
    \trivlist
    \item\relax
    {\itshape
    #1\@addpunct{.}}\hspace\labelsep\ignorespaces
    }{%
    \popQED\endtrivlist\@endpefalse
    }
\makeatother

\let\cite\citep
\usepackage{doi}

\usepackage{subcaption} 
\usepackage{caption}
\usepackage{lastpage}
\jmlrheading{25}{2024}{1-\pageref{LastPage}}{7/23; Revised
10/23}{1/24}{23-0912}{Shijun Zhang, Jianfeng Lu, and Hongkai Zhao}
\ShortHeadings{Beyond ReLU to Diverse Activation Functions}{Zhang, Lu, and Zhao}
\firstpageno{1}

\begin{document}

\title{Deep Network Approximation: Beyond ReLU to Diverse Activation Functions}

\author{\name  \href{https://shijunzhang.top/}{\color{black}Shijun Zhang}\thanks{\hspace{-2pt}Corresponding author.}
	\email \mailto[shijun.math@outlook.com]{shijun.zhang@duke.edu}\\ 
	\addr   Department of Mathematics\\  
	Duke University
	\AND  \name Jianfeng Lu
		\email \mailto{jianfeng@math.duke.edu} \\
		\addr   Department of Mathematics\\
		  Duke University
	\AND \name Hongkai Zhao 
	\email \mailto{zhao@math.duke.edu}  \\
	\addr   Department of Mathematics\\  
	Duke University
		}
  
\editor{Joan Bruna}

\maketitle

\begin{abstract}%

    This paper explores the expressive power of deep neural networks for a diverse range of activation functions. An activation function set $\mathscr{A}$ is defined to encompass the majority of commonly used activation functions, such as $\mathtt{ReLU}$, $\mathtt{LeakyReLU}$, $\mathtt{ReLU}^2$, $\mathtt{ELU}$, $\mathtt{CELU}$, $\mathtt{SELU}$, $\mathtt{Softplus}$, $\mathtt{GELU}$, $\mathtt{SiLU}$, $\mathtt{Swish}$, $\mathtt{Mish}$, $\mathtt{Sigmoid}$, $\mathtt{Tanh}$, $\mathtt{Arctan}$, $\mathtt{Softsign}$, $\mathtt{dSiLU}$, and $\mathtt{SRS}$. We demonstrate that for any activation function $\varrho\in \mathscr{A}$, a $\mathtt{ReLU}$ network of width $N$ and depth $L$ can be approximated to arbitrary precision by a $\varrho$-activated network of width $3N$ and depth $2L$ on any bounded set. This finding enables the extension of most approximation results achieved with $\mathtt{ReLU}$ networks to a wide variety of other activation functions, albeit with slightly increased constants.
    Significantly,  we establish that the (width,$\,$depth) scaling factors can be further reduced from $(3,2)$ to $(1,1)$  if $\varrho$ falls within a specific subset of $\mathscr{A}$. This subset includes activation functions such as $\mathtt{ELU}$, $\mathtt{CELU}$, $\mathtt{SELU}$, $\mathtt{Softplus}$, $\mathtt{GELU}$, $\mathtt{SiLU}$, $\mathtt{Swish}$, and $\mathtt{Mish}$.

\end{abstract}



\section{Introduction}
\label{sec:intro}

In the realm of artificial intelligence, deep neural networks have emerged as a powerful tool. By harnessing the potential of interconnected nodes organized into multiple layers, deep neural networks have showcased notable success in many challenging applications and new territories.
The foundation of deep neural networks consists of an affine linear transformation followed by an activation function. The activation function plays an important role in the successful training of deep neural networks. 
In recent years, the Rectified Linear Unit (\ReLU) \cite{10.5555/3104322.3104425} has experienced a surge in popularity and demonstrated its effectiveness as an activation function.

The adoption of \ReLU\  
has led to significant improvements in results
on challenging datasets in supervised learning \cite{NIPS2012_c399862d}. 
Optimizing deep networks activated by \ReLU\  is simpler compared to networks utilizing other activation functions such as \Sigmoid\ or \Tanh, since gradients can propagate when the input to \ReLU\  is positive. It was also shown in the recent work \cite{ZZZZ-23} 
that using \ReLU\ has a weaker regularizing effect than using smoother
activation functions in practice.
The effectiveness and simplicity of \ReLU\ have positioned it as the preferred default activation function in the deep learning community. A significant number of publications have extensively investigated the expressive capabilities of deep neural networks, with the majority of them primarily focusing on the \ReLU\ activation function.


In recent developments, various alternative activation functions have been proposed as replacements for \ReLU. Notable examples include 
the Leaky \ReLU\ (\LeakyReLU) \cite{maas2013rectifier}, the Exponential Linear Unit (\ELU) \cite{DBLP:journals/corr/ClevertUH15}, and the Gaussian Error Linear Unit (\GELU) \cite{2016arXiv160608415H}.  These alternative activation functions have exhibited improved performance in specific neural network architectures.
Among these alternatives, \GELU\ has gained significant popularity in deep learning models, especially in the realm of natural language processing 
tasks. It has been successfully employed in prominent models such as GPT-3 \cite{NEURIPS2020_1457c0d6}, BERT \cite{devlin-etal-2019-bert}, XLNet \cite{NEURIPS2019_dc6a7e65}, and various other transformer models.
While these recently proposed activation functions have demonstrated promising empirical results, their theoretical underpinnings are still being developed. This paper aims to investigate the expressive capabilities of deep neural networks utilizing these activation functions. In doing so, we establish connections between these functions and \ReLU, allowing us to extend most existing approximation results for \ReLU\ networks to encompass other activation functions such as \ELU\ and \GELU. More precisely, we will define an activation function set, denoted as $\scrA$, which contains the majority of commonly used activation functions.
\subsection{Definition of Activation Function Set}

To the best of our knowledge,
 the majority of commonly used activation functions can be generally classified into three distinct categories.
 The initial category primarily comprises piecewise smooth functions,
	e.g., \ReLU,  \LeakyReLU,
 $\ReLU^2$ (\ReLU\ squared) \cite{SIEGEL20221}, \ELU, \CELU\ (Continuously Differentiable \ELU) \cite{2017arXiv170407483B}, and \SELU\ (Scaled \ELU) \cite{NIPS2017_5d44ee6f}.
 All these activation functions are included in
     $\bigcup_{k=0}^{\infty} \scrA_{1,k}$, 
    where $\scrA_{1,k}$, for each smoothness index $k\in\N$, is defined as
	\begin{equation*}
		\begin{split}
			\scrA_{1,k}\coloneqq \big\{\varrho:\R\to\R
			\    \big|\,\ \scrK_k(\varrho)\neq \emptyset\big\},
		\end{split}
	\end{equation*}
 where $\scrK_k(\varrho)$ represents the set of $k$-th order ``kinks'' of $\varrho:\R\to\R$. A point $x_0\in\R$ is referred to as a $k$-th order ``kink'' of $\varrho$ if there exist $a_0,b_0\in\R$ such that $a_0<x_0<b_0$, $\varrho\in C^k\big((a_0,b_0)\big)$, and
\begin{equation*}
			\R\ni\lim_{t\to 0^-}\frac{\varrho^{(k)}(x_0+t)-\varrho^{(k)}(x_0)}{t}\neq \lim_{t\to 0^+}\frac{\varrho^{(k)}(x_0+t)-\varrho^{(k)}(x_0)}{t}\in\R.
\end{equation*}
 It is worth noting that  $\varrho\in C^k\big((a_0,b_0)\big)\backslash C^{k+1}\big((a_0,b_0)\big)$ is necessary to ensure $\varrho\in\scrA_{1,k}$. Specifically, at $x_0\in (a_0,b_0)$, the left and right derivatives of $\varrho^{(k)}\in C\big((a_0,b_0)\big)$ must exist and be distinct. However, there are no specific requirements placed on $\varrho$ outside $(a_0,b_0)$.
 Here and in the sequel, we use $f^{(k)}$ to represent the $k$-th derivative of a function $f:\Omega\subseteq\R\to\R$. For instance, $f^{(0)}$ refers to the function itself, and $f^{(1)}$ represents the first derivative.
 Let $\N$ denote the set of natural numbers, i.e., $\N\coloneqq \{0,1,2,\cdots\}$, and set $\N^+\coloneqq \N\backslash \{0\}$.
 Given a function $f:\Omega\subseteq \R^d\to\R$, we denote $\partial^\bmalpha f$ as the partial derivative $\bmx\mapsto \tfrac{\partial^\bmalpha }{\partial \bmx^\bmalpha}f(\bmx)=\tfrac{\partial^{\alpha_1}}{\partial x_1^{\alpha_1}}\tfrac{\partial^{\alpha_2}}{\partial x_2^{\alpha_2}}\cdots \tfrac{\partial^{\alpha_d}}{\partial x_d^{\alpha_d}}f(\bmx)$ for any $\bmx=(x_1,\cdots,x_d)\in \Omega$ and $\bmalpha=(\alpha_1,\cdots,\alpha_d)\in\N^d$.
Let $C^k(\Omega)$ denote the set of all functions $f:\Omega\subseteq \R^d\to\R$, in which the partial derivatives 
$\partial^\bmalpha  f$
exist and are continuous for any $\bmalpha\in \N^d$ with $\sum_{i=1}^d \alpha_i\le k$.
In particular, when $k=0$, we denote $C^0(\Omega)$ as $C(\Omega)$, which represents the set of continuous functions on $\Omega$.

The second category primarily encompasses smooth variations of \ReLU,
e.g., \Softplus\ \cite{pmlr-v15-glorot11a}, \GELU,
 \SiLU\ (Sigmoid Linear Unit) \cite{2016arXiv160608415H,ELFWING20183}, \Swish\  \cite{2017arXiv171005941R}, and \Mish\ \cite{DBLP:conf/bmvc/Misra20}. 
 All these activation functions are encompassed in the set $\scrA_{2}$, which is defined via
\begin{equation*}
    \begin{split}
        \scrA_{2}\coloneqq \Big\{\varrho:\R\to\R
        \    \Big| \ &   
        \forall x\in\R,\  \varrho(x)\coloneqq(x+b_0)\cdot h(x)+b_1,\quad 
        b_0,b_1\in\R,\quad h\in \scrS
        \Big\},
    \end{split}
\end{equation*}
where $\scrS$ represents a collection of functions referred to as S-shaped functions, defined as
\begin{equation*}
    \begin{split}
        \scrS\coloneqq \Big\{h:\R\to\R\  \Big|\ 
     \sup_{x\in \R}|h(x)|<\infty,\quad \R\ni\lim_{x\to -\infty} h(x)\neq\lim_{x\to \infty} h(x)\in \R
        \Big\}.
    \end{split}
\end{equation*}
Evidently, activation functions such as \GELU, \SiLU, \Swish, and \Mish\ are members of $\scrA_2$. It is worth highlighting that \Softplus, \ELU, \CELU, and \SELU\  also belong to $\scrA_2$, even though this may not be immediately apparent.  
Let us take \Softplus\ as an example to illustrate this point, and similar reasoning applies to the other cases.  Define
$h(x)\coloneqq \tfrac{\ln(1+e^x)-\ln 2}{x}$ for any $x\neq 0$ and $h(0)=\tfrac{1}{2}$, where $e$  represents the base of the natural logarithm.  Consequently, we have $\Softplus(x)=\ln(1+e^x)=x\cdot h(x)+\ln 2$ for any $x\in\R$. It is then straightforward to verify that \Softplus\ is indeed a member of $\scrA_2$. 
We would like to point out that the primary idea behind defining $\scrA_2$ is to replace the step function $\one_{\{x>0\}}$ in $\ReLU(x)=x\cdot \one_{\{x>0\}}$ with a (smooth) S-shaped function. This insight allows us to create numerous examples within $\scrA_2$. For instance, one can define $\varrho(x)\coloneqq x\cdot h(x)$, where $h:\R\to [0,1]$ represents a cumulative distribution function of a real-valued random variable. Notably,  \GELU\ is defined in this manner, with $h$ being the (standard) Gaussian cumulative distribution function.

The final category is primarily composed of S-shaped activation functions with particular regularity,
e.g.,
\Sigmoid, \Tanh, \Arctan, and
\Softsign\    
  \cite{10.5555/1620853.1620921}. 
All these functions are included in the set $\scrA_3$, which is defined via 
\begin{equation*}
    \begin{split}
        \scrA_3\coloneqq \big\{\varrho:\R\to\R
        \    \big|\,\  \varrho\in \scrS,\quad \exists  x_0\in\R,\ \varrho^\dprime(x_0)\neq 0\big\}.
    \end{split}
\end{equation*}
The set $\scrA_3$ includes a wide range of activation functions, with certain ones featuring discontinuities. 
In addition to the examples mentioned earlier, there exist numerous functions in the set $\scrA_3$, such as \dSiLU\ (the derivative of \SiLU) as introduced in \cite{ELFWING20183}, and \SRS\ (Soft-Root-Sign) discussed in \cite{10.1007/978-3-030-60636-7_26}. Furthermore, the derivatives of \Softplus, \GELU, \SiLU, \Swish, and \Mish\ fall into the category of $\scrA_3$.
Then the activation function set $\scrA$ is defined as the union of $\scrA_{1,0}$, $\scrA_{1,1}$, $\scrA_{2}$, and $\scrA_3$, which can be expressed as
\begin{equation*} 
     \scrA\coloneqq\big(\scrA_{1,0}\cup\scrA_{1,1}\big)
     \cup \scrA_{2}\cup\scrA_3.
\end{equation*}
Throughout the entirety of this paper, the definitions of $\scrA$, $\scrA_{1,k}$ for $k\in\N$, $\scrA_{2}$, and $\scrA_3$ will remain consistent.
It is worth noting that if $\varrho\in\scrA$, then its variant $x\mapsto w_1\varrho(w_0 x+ b_0)+b_1$ is also in $\scrA$ provided $w_0 w_1\neq  0$. 
Notably, the set
$\scrA$ 
encompasses the majority of commonly used activation functions, such as
\ReLU,  \LeakyReLU, $\ReLU^2$,  \ELU, \CELU, \SELU, \Softplus, \GELU, \SiLU, \Swish, \Mish, \Sigmoid, \Tanh, \Arctan, \Softsign, \dSiLU, \SRS, and their modified versions achieved by employing translation, {non-zero} scaling, and reflection operations. 
In Section~\ref{sec:summary:activation:functions}, we will present definitions and visual representations of the activation functions mentioned above. 

Define the supremum norm of a bounded vector-valued function $\bmf: \Omega\subseteq \R^d \to \R^n$  via
\begin{equation*}
	\|\bmf\|_{\sup(\Omega)}\coloneqq \sup\big\{|f_i(\bmx)|: \bmx\in \Omega,\   i\in\{1,2,\cdots,n\}\big\},
\end{equation*}
where $f_i$ is the $i$-th component of $\bmf$ for $i=1,2,\cdots,n$. 
This paper exclusively focuses on fully connected feed-forward neural networks.
We denote $\nn{\varrho}{N}{L}{d}{n}$ as the set of vector-valued functions $\bmphi:\mathbb{R}^d\to\mathbb{R}^n$ that can be represented
by $\varrho$-activated networks of width $\le N\in \N^+$ and depth $\le L\in\N^+$. 
In our context, the width of a network refers to the maximum number of neurons in a hidden layer and the depth corresponds to the number of hidden layers. 
For instance, suppose $\bmphi:\R^d\to\R^n$ is a vector-valued function realized by a $\varrho$-activated network, where $\varrho$ is the activation function that can be applied elementwise to a vector input. Then $\bmphi$ can be expressed as
\begin{equation*}
    \bmphi =\calbmL_L\circ\varrho\circ
    		\calbmL_{L-1}\circ 
    \ \cdots \  \circ 
    		\varrho\circ
    \calbmL_1\circ\varrho\circ\calbmL_0,
\end{equation*}
where $\calbmL_\ell$ is an affine linear map given by $\calbmL_\ell(\bmy)\coloneqq \bmW_\ell \cdot \bmy +\bmb_\ell$ for $\ell=0,1,\cdots,L$. Here, $\bmW_\ell\in \R^{N_{\ell+1}\times N_{\ell}}$ and $\bm{b}_\ell\in \R^{N_{\ell+1}}$ are the weight matrix and the bias vector, respectively, with
 $N_0=d$, $N_1,N_2,\cdots,N_L\in\N^+$, and 
$N_{L+1}=n$.  Clearly, $\bmphi\in \nn{\varrho}{N}{L}{d}{n}$, where $N=\max\{N_1,N_2,\cdots,N_L\}$.

\subsection{Main Results}

Our goal is to explore the expressiveness of deep neural networks activated by $\varrho\in\scrA$. In pursuit of this goal, the following theorem establishes connections between \ReLU\ and $\varrho\in\scrA$.
This allows us to extend and generalize most existing approximation results for \ReLU\ networks to activation functions in $\scrA$.
\begin{theorem}
	\label{thm:main}
	Suppose	$\varrho\in \scrA$ and $\bmphi_\ReLU\in \nn{\ReLU}{N}{L}{d}{n}$ with $N,L,d,n\in\N^+$. Then
	for any $\eps>0$ and $A>0$,
 there exists $\bmphi_\varrho\in \nn{\varrho}{3N}{2L}{d}{n}$ such that
	\begin{equation*}
		\|\bmphi_\varrho-\bmphi_\ReLU\|_{\sup([-A,A]^d)}<\varepsilon.
	\end{equation*}
\end{theorem}

The proof of Theorem~\ref{thm:main} can be found in Section~\ref{sec:proof:thms}.
Theorem~\ref{thm:main} implies that a \ReLU\  network of width $N$ and depth $L$ can be approximated by a $\varrho$-activated network of width $3N$ and $2L$ arbitrarily well on any bounded set for any pre-specified $\varrho\in \scrA$.
In other words, 	$\nn{\varrho}{3N}{2L}{d}{n}$ is dense in $\nn{\ReLU}{N}{L}{d}{n}$ in terms of the $\|\cdot\|_{\sup([-A,A]^d)}$\\ norm for any pre-specified $A>0$ and $\varrho\in\scrA$. 
Indeed, this implies that networks activated by $\varrho\in\scrA$ possess, at the very least, a comparable level of expressive capability as \ReLU\ networks, providing valuable insights for the development of new activation functions.  Constraining $\varrho$ to $\scrA$ is a relatively simple process, and it serves to ensure the effective expressiveness of $\varrho$-activated networks. This, in turn, enables us to direct our attention more towards the learning and numerical properties of $\varrho$.

It is worth mentioning that, while  Theorem~\ref{thm:main} covers activation functions $\varrho\in \scrA_{1,k}$ only for $k=0,1$, it is possible to obtain analogous results for larger values of $k\in \N$. For more detailed analysis and discussions, please refer to Section~\ref{sec:additional:theorems}. 
Additionally, we would like to emphasize that the (width,\,depth) scaling factors appearing in Theorem~\ref{thm:main} as $(3,2)$ have the potential to be reduced to $(2,1)$ or even $(1,1)$ under certain circumstances, as elaborated in Table~\ref{tab:summary} later on.

Equipped with Theorem~\ref{thm:main}, we can expand most existing approximation results for \ReLU\ networks to encompass various alternative activation functions, albeit with slightly larger constants. To illustrate this point, we present several corollaries below.
Theorem~$1.1$ of \cite{shijun:optimal:rate:in:width:and:depth} implies that a \ReLU\  network of width $C_{d,1}N$ and depth $C_{d,2} L$
can approximate a continuous function $f\in C([0,1]^d)$ with an error 
$C_{d,3}\,\omega_f\big(\big(N^2L^2\ln (N+1)\big)^{-1/d}\big)$, where $C_{d,1}$, $C_{d,2}$,  and $C_{d,3}$ are constants\footnote{The values of $C_{d,1}$, $C_{d,2}$,  and $C_{d,3}$ are explicitly given in \cite{shijun:optimal:rate:in:width:and:depth}.} determined by $d$, and $\omega_f(\cdot)$ is the modulus of continuity of $f\in C([0,1]^d)$ defined via
\begin{equation*}
    \omega_f(t)\coloneqq \sup\big\{|f(\bmx)-f(\bmy)|: \|\bmx-\bmy\|_2\le t,\,\;\bmx,\bmy\in [0,1]^d\big\}\quad \tn{for any $t\ge 0$.}
\end{equation*}
By combining this result with Theorem~\ref{thm:main},
an immediate corollary follows.
\begin{corollary}
\label{coro:continuous:functions}
	Suppose $\varrho\in\scrA$ and $f\in C([0,1]^d)$ with $d\in\N^+$. Then for any $N,L\in\N^+$, 
	there exists 
$\phi\in \nnOneD[\big]{\varrho}{3C_{d,1}N}{2C_{d,2}L}{\R^d}{\R}$
	such that 
	\begin{equation*}
	\|\phi-f\|_{L^\infty([0,1]^d)}\le 2C_{d,3}\,\omega_f\Big(\big(N^2L^2\ln (N+1)\big)^{-1/d}\Big),
	\end{equation*}
 where $C_{d,1}$, $C_{d,2}$,  and $C_{d,3}$ are constants determined by $d$.
\end{corollary}


It is demonstrated in Theorem~$1.1$ of \cite{shijun:smooth:functions} that a \ReLU\  network of width $C_{s,d,1}N\ln(N+1)$ and depth $C_{s,d,2} L\ln(L+1)$
can approximate a smooth function $f\in C^s([0,1]^d)$ with an error 
$C_{s,d,3} \|f\|_{C^s([0, 1]^d)} N^{-2s/d}L^{-2s/d}$, where $C_{s,d,1}$, $C_{s,d,2}$,  and $C_{s,d,3}$ are constants\footnote{The values of $C_{s,d,1}$, $C_{s,d,2}$, and $C_{s,d,3}$ are explicitly provided in \cite{shijun:smooth:functions}.} determined by $s$ and $d$. Here, the norm $\|f\|_{C^s([0, 1]^d)}$ for any $f\in C^s([0,1]^d)$ is defined via
\begin{equation*}
    \|f\|_{C^s([0, 1]^d)}\coloneqq \sup\big\{\|\partial^\bmalpha f\|_{L^\infty([0,1]^d)}:\,\|\bmalpha\|_1\le s,\,\   \bmalpha\in \N^d\big\}\quad \tn{for any $f\in C^s([0,1]^d)$.}
\end{equation*}
By combining the aforementioned result with Theorem~\ref{thm:main}, we can promptly deduce the subsequent corollary.
\begin{corollary}
\label{coro:smooth:functions}
	Suppose $\varrho\in\scrA$ and $f\in  C^s([0,1]^d)$ with $s,d\in \N^+$. Then for any $N,L\in \N^+$, there exists 
$\phi\in \nnOneD[\big]{\varrho}{3C_{s,d,1}N\ln(N+1)}{\,2C_{s,d,2}L\ln(L+1)}{\R^d}{\R}$
 such that 
	\begin{equation*}
		\|\phi-f\|_{L^\infty([0,1]^d)}\le 2C_{s,d,3} \|f\|_{C^s([0, 1]^d)} N^{-2s/d}L^{-2s/d},
	\end{equation*} 
 where $C_{s,d,1}$, $C_{s,d,2}$,  and $C_{s,d,3}$ are constants determined by $s$ and $d$.
\end{corollary}


It is demonstrated in Theorem~$1$ of \cite{NEURIPS2022_2f4b6feb} that 
a continuous piecewise linear function
$f:\R^d\to\R$ with $q\in\N^+$ pieces can be exactly represented by 
a \ReLU\  network of width $\lceil 3q/2\rceil q$ and depth $2\lceil \log_2 q\rceil+1$.
By combining this result with Theorem~\ref{thm:main}, we obtain the following corollary.
\begin{corollary}
\label{coro:CPWL:functions}
    Suppose $\varrho\in \scrA$ and let $f:\R^d\to\R$ be a continuous piecewise linear function with $q$ pieces, where $d,q\in \N^+$. Then
    for any $\eps>0$ and $A>0$, there exists $\phi\in \nnOneD[\big]{\varrho}{3\lceil 3q/2\rceil q}{4\lceil \log_2 q\rceil+2}{\R^d}{\R}$, such that
    \begin{equation*}
    |\phi(\bmx)-f(\bmx)|<\eps \quad \text{for any $\bmx\in [-A,A]^d$.}
    \end{equation*}
\end{corollary}


It is demonstrated in \cite{shijun:RCNet} that
even though a single fixed-size \ReLU\ network has limited expressive capabilities, repeatedly composing it can create surprisingly expressive networks.
Specifically, Theorem~$1.1$ of \cite{shijun:RCNet}  establishes that
$\calL_2\circ \bmg^{\circ (3r+1)}\circ \calbmL_1$ can approximate a continuous function $f\in C([0,1]^d)$ with an
error $6\sqrt{d}\,\omega_f(r^{-1/d})$, where $\bmg\in\nn{\ReLU}{69d+48}{5}{5d+5}{5d+5}$, $\calbmL_1$ and $\calL_2$ are two affine
linear maps matching the dimensions, and $\bmg^{\circ   r}$
denotes the $r$-times composition of $\bmg$.
By merging this outcome with Theorem~\ref{thm:main}, we can promptly deduce the subsequent corollary.
\begin{corollary}
\label{coro:RCNet}
	Suppose $\varrho\in\scrA$ and $f\in C([0,1]^d)$ with $d\in\N^+$. Then for any $r\in \N^+$ and $p\in [1,\infty)$, 
	there exist $\bmg\in \nn{\varrho}{207d+144}{10}{5d+5}{5d+5}$
	and two affine linear maps
	$\calbmL_1:\R^d\to\R^{5d+5}$ and $\calL_2:\R^{5d+5}\to\R$
	such that 
	\begin{equation*}
		\big\|\calL_2\circ \bmg^{\circ (3r+1)}\circ \calbmL_1-f\big\|_{L^p([0,1]^d)}\le 7\sqrt{d}\,\omega_f(r^{-1/d}).
	\end{equation*}
\end{corollary}
This corollary is not immediately apparent and necessitates a derivation: $\bmg_\eps\approx \bmg$ implies $\bmg_\eps^{\circ (3r+1)}\approx \bmg^{\circ (3r+1)}$ for small $\eps>0$, where $\bmg$ and $\bmg_\eps$ represent \ReLU\ and $\varrho$-activated networks, respectively. 
The proof relies on mathematical induction and the following equation
\begin{equation*}
    \big\|\bmg_\eps^{\circ  (m+1)}-\bmg^{\circ  (m+1)}\big\|_{\sup(\calK)}\le
    \big\|\bmg_\eps\circ \bmg_\eps^{\circ  m}-\bmg \circ \bmg_\eps^{\circ  m}\big\|_{\sup(\calK)}
    +
    \big\|\bmg \circ \bmg_\eps^{\circ  m}-\bmg\circ \bmg^{\circ  m}\big\|_{\sup(\calK)}
\end{equation*}
for any compact set $\calK$ and $m\in\N$,
where the first term of the above equation is constrained by  $\bmg_\eps\approx \bmg$, and the second term is controlled by the induction hypothesis and the uniform continuity of $\bmg$ on a compact set.
It is worth highlighting that the approximation error in Corollary~\ref{coro:RCNet} is 
measured using the $L^p$-norm for any $p\in [1,\infty)$. Nevertheless, it is feasible to
generalize this result to the $L^\infty$-norm as well, 
though it comes with larger associated constants.
To accomplish this, we only need to combine Theorem~$1.3$ of \cite{shijun:RCNet} with Theorem~\ref{thm:main}.

The remainder of this paper is organized as follows. 
In Section~\ref{sec:further:discussion}, we explore some additional related topics.
We present four supplementary theorems, Theorems~\ref{thm:main:kth:derivative}, \ref{thm:main:scrA:1k}, \ref{thm:main:scrA:2}, and \ref{thm:main:scrA:2:new} in Section~\ref{sec:additional:theorems} to complement Theorem~\ref{thm:main}, and we summarize main results of this paper in Table~\ref{tab:summary}.  We also discuss related work in Section~\ref{sec:related:work} and provide definitions and illustrations of common activation functions in Section~\ref{sec:summary:activation:functions}.
Moving forward to Section~\ref{sec:proof:thms}, we establish the proofs of Theorems~\ref{thm:main}, \ref{thm:main:kth:derivative}, \ref{thm:main:scrA:1k}, \ref{thm:main:scrA:2}, and \ref{thm:main:scrA:2:new}. In Section~\ref{sec:notation}, we introduce the notations used throughout this paper. In Section~\ref{sec:props:proof:thms}, we present several propositions, namely Propositions~\ref{prop:activation:replace}, \ref{prop:approx:f:nth:D}, \ref{prop:approx:ReLU:scrA:1k}, and \ref{prop:approx:ReLU:scrA:2:3}, 
outlining the underlying ideas for proving Theorems~\ref{thm:main}, \ref{thm:main:kth:derivative},  \ref{thm:main:scrA:1k}, \ref{thm:main:scrA:2}, and \ref{thm:main:scrA:2:new}. Subsequently, by assuming the validity of propositions, we provide the proof of Theorem~\ref{thm:main} in Section~\ref{sec:proof:thm:main}, followed by the subsequent proofs of Theorems~\ref{thm:main:kth:derivative}, \ref{thm:main:scrA:1k}, \ref{thm:main:scrA:2}, and \ref{thm:main:scrA:2:new} in Section~\ref{sec:proof:thms:main:others}. Finally,  we prove Propositions~\ref{prop:activation:replace}, \ref{prop:approx:f:nth:D}, \ref{prop:approx:ReLU:scrA:1k}, and \ref{prop:approx:ReLU:scrA:2:3}  in Sections~\ref{sec:proof:prop:activation:replace}, \ref{sec:proof:prop:approx:f:nth:D}, \ref{sec:proof:prop:approx:ReLU:scrA:1k}, and \ref{sec:proof:prop:approx:ReLU:scrA:23}, respectively.


\section{Further Discussions}
\label{sec:further:discussion}

In this section, we explore some additional related topics.
We first present four supplementary theorems, namely Theorems~\ref{thm:main:kth:derivative}, \ref{thm:main:scrA:1k}, \ref{thm:main:scrA:2}, and \ref{thm:main:scrA:2:new},
which complement Theorem~\ref{thm:main} and are
covered in detail
in Section~\ref{sec:additional:theorems}.
Additionally, we discuss related work in Section~\ref{sec:related:work} and provide 
comprehensive explanations and visual examples
of commonly used activation functions in Section~\ref{sec:summary:activation:functions}.

\subsection{Additional Results}
\label{sec:additional:theorems}

It is important to note that Theorem~\ref{thm:main} specifically focuses on activation functions $\varrho\in \scrA_{1,k}$ with $k=0,1$. However, we can also obtain similar results for larger values of $k\in \N$, where $\varrho\in \scrA_{1,k}$ exhibits even smoother properties. 
In particular, we establish that for any $\varrho\in  C^k(\R)$ with $k\in\N$, a $\varrho^{(k)}$-activated network of width $N$ and depth $L$ can be approximated to arbitrary precision by a $\varrho$-activated network of width $(k+1)N$ and depth $L$ on any bounded set.
\begin{theorem}
	\label{thm:main:kth:derivative}
	Given any $k\in\N$ and $\varrho\in C^k(\R)$, suppose
	$\bmphi_{\varrho^{(k)}}\in \nn{\varrho^{(k)}}{N}{L}{d}{n}$ with $N,L,d,n\in\N^+$. Then for any $\eps>0$ and $A>0$, there exists $\bmphi_\varrho\in \nn{\varrho}{(k+1)N}{L}{d}{n}$ such that
	\begin{equation*}
		\|\bmphi_\varrho-\bmphi_{\varrho^{(k)}}\|_{\sup([-A,A]^d)}<\varepsilon.
	\end{equation*}
\end{theorem}

Furthermore, the following theorem  specifically addresses $\varrho \in \scrA_{1,k}$ for any $k\in\N$.
Specifically,
we demonstrate that for any $\varrho\in \scrA_{1,k}$ with $k\in\N$, a \ReLU\  network of width $N$ and depth $L$  can be approximated with arbitrary precision by a $\varrho$-activated network of width $(k+2)N$ and depth $L$ on any bounded set.
\begin{theorem}
	\label{thm:main:scrA:1k}
	Suppose
	$\bmphi_\ReLU\in \nn{\ReLU}{N}{L}{d}{n}$ with $N,L,d,n\in\N^+$.
	Then for any $\eps>0$, $A>0$, $k\in\N$, and $\varrho\in \scrA_{1,k}$, there exists $\bmphi_\varrho\in \nn{\varrho}{(k+2)N}{L}{d}{n}$ such that
	\begin{equation*}
		\|\bmphi_\varrho-\bmphi_\ReLU\|_{\sup([-A,A]^d)}<\varepsilon.
	\end{equation*}
\end{theorem}

Moving forward, let us delve into the concept of optimality, particularly concerning the potential for further reducing 
the (width,\,depth) scaling factors
in Theorem~\ref{thm:main}.
Despite our diligent efforts, we have yet to establish lower bounds that correspond to the (width,\,depth) scaling factors we have identified in this study. Consequently, it remains an open question whether the (width,\,depth) scaling factors we have found are the best possible in the general context. Nonetheless, by targeting specific categories of activation functions, we have succeeded in deriving improved scaling factors.

Next, we introduce two specific scenarios aimed at diminishing the (width,\,depth) scaling factors in Theorem~\ref{thm:main}. The first scenario is directed towards the situation where $\varrho$ belongs to the set $\scrA_{2}$. In this particular case, we demonstrate that the  (width,\,depth) scaling factors can be reduced to $(2,1)$, as exemplified in Theorem~\ref{thm:main:scrA:2} below.

\begin{theorem}
	\label{thm:main:scrA:2}
	Suppose
	$\bmphi_\ReLU\in \nn{\ReLU}{N}{L}{d}{n}$ with $N,L,d,n\in\N^+$.
	Then for any $\eps>0$, $A>0$, and $\varrho\in \scrA_{2}$, there exists $\bmphi_\varrho\in \nn{\varrho}{2N}{L}{d}{n}$ such that
	\begin{equation*}
		\|\bmphi_\varrho-\bmphi_\ReLU\|_{\sup([-A,A]^d)}<\varepsilon.
	\end{equation*}
\end{theorem}

The second scenario revolves around a specific subset of $\scrA_{2}$, denoted as $\tildescrA_{2}$, and it is defined as
\begin{equation*}
    \begin{split}
        \tildescrA_{2}\coloneqq \Big\{&\varrho:\R\to\R
        \    \Big| \   
        \forall x\in\R,\  \varrho(x)\coloneqq(x+b_0)\cdot h(x)+b_1,\quad 
        b_0,b_1\in\R,\quad
       h\in\tildescrS\Big\},
    \end{split}
\end{equation*}
where $\tildescrS$ represents a refined subset of $\scrS$ and is defined as
\begin{equation*}
    \begin{split}
        \tildescrS\coloneqq \Big\{ h:\R\to\R
        \    \Big| \   
        h\in \scrS,\quad \Big(\lim_{x\to -\infty} h(x)\Big) \cdot \Big(\lim_{x\to \infty} h(x)\Big)  =0
        \Big\}.
    \end{split}
\end{equation*}
The only difference between $\scrA_{2}$ and $\tildescrA_{2}$ lies in the limits defined for $h$ therein. In $\scrA_{2}$, it only requires the existence and distinctness of $\lim_{x\to -\infty} h(x)$ and $\lim_{x\to \infty} h(x)$.  In contrast, $\tildescrA_{2}$ introduces an additional requirement where either $\lim_{x\to -\infty} h(x)$ or $\lim_{x\to \infty} h(x)$ must equal $0$.
As illustrated in the theorem below, the (width,\,depth) scaling factors in this particular instance can be reduced to (1,1), which 
is the best scaling obtained by our construction.

\begin{theorem}
	\label{thm:main:scrA:2:new}
	Suppose
	$\bmphi_\ReLU\in \nn{\ReLU}{N}{L}{d}{n}$ with $N,L,d,n\in\N^+$.
	Then for any $\eps>0$, $A>0$, and $\varrho\in \tildescrA_{2}$, there exists $\bmphi_\varrho\in \nn{\varrho}{N}{L}{d}{n}$ such that
	\begin{equation*}
		\|\bmphi_\varrho-\bmphi_\ReLU\|_{\sup([-A,A]^d)}<\varepsilon.
	\end{equation*}
\end{theorem}
The proofs of Theorems~\ref{thm:main:kth:derivative}, \ref{thm:main:scrA:1k}, \ref{thm:main:scrA:2}, and \ref{thm:main:scrA:2:new} are available in Section~\ref{sec:proof:thms}. To facilitate the reading, we provide a summary and comparison of our primary findings in Table~\ref{tab:summary}. 

\begin{table}[htbp!] 
\setlength{\tabcolsep}{10pt}
\def\arraystretch{1.22}
	\caption{Summary of main results.} 
	\label{tab:summary}
	\centering  
	\resizebox{0.9998\textwidth}{!}{ 
		\begin{tabular}{ccccccc} 
			\toprule
			    &   conditions on $\varrho$  &   (width,\,depth) scaling factors  \\
			
			\midrule
			Theorem~\ref{thm:main}    &  $\varrho\in\scrA=\big( \scrA_{1,0}\cup\scrA_{1,1}\big)\cup\scrA_{2}\cup\scrA_3$,  e.g., \Sigmoid\ and \Tanh & $(3,2)$ \\
   \midrule
Theorem~\ref{thm:main:scrA:1k} & $\varrho\in \scrA_{1,k}$ for any $k\in\N$, e.g., $\ReLU^2\in \scrA_{1,1}$ and $\ReLU^3\in \scrA_{1,2}$ &  $(k+2, 1)$ \\

\midrule
Theorem~\ref{thm:main:scrA:2} & $\varrho\in \scrA_{2}$, e.g., $\varrho(x)\coloneqq x\cdot \Softsign(x)$ and $\varrho(x)\coloneqq  x\cdot \Arctan(x)$  &  $(2, 1)$ \\

\midrule
Theorem~\ref{thm:main:scrA:2:new} & \quad 
 $\varrho\in \tildescrA_{2}\subseteq \scrA_{2}$, e.g.,   \ELU, \CELU, \SELU, \Softplus, \GELU, \SiLU, \Swish, and \Mish\quad   &  $(1, 1)$ \\			
			\bottomrule
		\end{tabular} 
	}
\end{table} 



It is important to highlight that Theorem~\ref{thm:main:scrA:2:new} suggests that the activation function within $\tildescrA_2$ is, at the very least, not inferior to \ReLU\  in terms of approximation. 
Refer to Table~\ref{tab:summary:approx:ReLU} for the comparison of approximation errors when using a single active neuron activated by $\varrho\in\tildescrA_2$ to approximate \ReLU. 
Definitions and visual depictions of the activation functions referenced in Table~\ref{tab:summary:approx:ReLU} are provided in Section~\ref{sec:summary:activation:functions}.

\begin{table}[htbp!] 
\setlength{\tabcolsep}{10pt}
\def\arraystretch{1.22}
	\caption{Comparison of approximation errors when using a single active neuron activated by $\varrho\in\tildescrA_2$ to approximate the \ReLU\ activation function. } 
	\label{tab:summary:approx:ReLU}
	\centering  
	\resizebox{0.998\textwidth}{!}{ 
		\begin{tabular}{ccccccc} 
			\toprule
			  $\varrho$  &  \qquad  approximation error (for any $x\in\R$ and $K>0$)\qquad \qquad  &    constant estimate  
 \\			
			\midrule
			$\ELU\ (\alpha>0)$ or $\CELU\ (\alpha>0)$    &  $0\le \ReLU(x)-\varrho(Kx)/K\le C_1\cdot \alpha/K$ & $C_1=1 $\\
\midrule
		
			$\Softplus$    &  $0\le \ReLU(x)-\big(\varrho(Kx)-\ln 2\big)/K\le C_2/K$ & $C_2=\ln 2 \approx  0.693 $\\
   \midrule
			$\GELU\ (\mu=0,\, \sigma>0)$    &  $0\le \ReLU(x)-\varrho(Kx)/K\le C_3\cdot\sigma/K$ & $C_3\approx  0.170 $\\  
   \midrule
			$\SiLU$    &  $0\le \ReLU(x)-\varrho(Kx)/K\le C_4/K$ & $C_4\approx
 0.278 $\\
   \midrule
			$\Swish\ (\beta>0)$    &  $0\le \ReLU(x)-\varrho(Kx)/K\le C_5/(\beta\cdot K)$ & $C_5=C_4\approx  0.278 $\\	
   \midrule 
			$\Mish$    &  $0\le \ReLU(x)-\varrho(Kx)/K\le C_6/K$ & $C_6\approx  0.309 $\\  
    \midrule 
        $\varrho(x)\coloneqq x\cdot\dSiLU(x)$  
        &  $-\tildeC_7/K\le \ReLU(x)-\varrho(Kx)/K\le C_7/K$ & $\tildeC_7\approx  0.265,\; C_7\approx 0.131 $\\ 
  \midrule 
        $\varrho(x)\coloneqq x\cdot\big(\Softsign(x)/2+1/2\big)$
        &  $0\le \ReLU(x)-\varrho(Kx)/K\le C_8/K$ & $C_8 = 0.5 $\\   
  \midrule 
        $\varrho(x)\coloneqq x\cdot\big(\Arctan(x)/\pi+1/2\big)$
        &  $0\le \ReLU(x)-\varrho(Kx)/K\le C_9/K$ & $C_9=1/\pi \approx 0.318 $\\ 
			\bottomrule
		\end{tabular} 
	}
\end{table} 

Let us briefly discuss how to estimate the approximation errors in Table~\ref{tab:summary:approx:ReLU}.
According to the definition of $\varrho\in\tildescrA_2$, up to suitable affine input/output transformations,  it can be represented as $\varrho(x) = x \cdot h(x)$ for any $x\in\R$, where $h$ is an S-shaped function with either $\lim_{x\to -\infty} h(x)$ or $\lim_{x\to \infty} h(x)$ equal to $0$.
Without loss of generality, through scaling and reflection if necessary, it can be assumed that $\lim_{x\to -\infty} h(x)=0$ and $\lim_{x\to \infty} h(x)=1$.
Then for any $x\in\R$ and $K>0$, we have
\begin{equation*}
    \ReLU(x)-\frac{\varrho(Kx)}{K}
    =x\cdot \one_{\{x>0\}}-\frac{Kx\cdot h(Kx)}{K}
    =\frac{Kx\cdot\big(\one_{\{Kx>0\}}- h(Kx)\big)}{K}
    \in \Big[\frac{m}{K},\,\frac{M}{K}\Big],
\end{equation*}
where $m,M\in \R\cup\{-\infty,\infty\}$ are given by
\begin{equation*}
    m=\inf\big\{y\cdot\big(\one_{\{y>0\}}-h(y)\big):y\in\R\big\}\quad \tn{and}\quad
    M=\sup\big\{y\cdot\big(\one_{\{y>0\}}-h(y)\big):y\in\R\big\}.
\end{equation*}
To ensure $-\infty < m \le  M < \infty$, we just need the set $\big\{y \cdot \big(\one_{\{y > 0\}} - h(y)\big) : y \in \R\big\}$ to be bounded.
We can then estimate the values of $m$ and $M$ by examining the characteristics of the derivatives, including higher-order ones.
Refer to Figure~\ref{fig:act:in:tildescrA:2:approx:ReLU} for 
visual representations demonstrating that the activation functions in $\tildescrA_2$ require only a single active neuron to approximate the \ReLU\ activation function.

We would like to emphasize that, as we will demonstrate later in Proposition~\ref{prop:approx:ReLU:scrA:2:3}, any activation function $\varrho\in\tildescrA_2$ can effectively employ just one active neuron to provide an arbitrarily accurate approximation of \ReLU\ within a \textit{bounded} subset. It is worth noting that, by introducing a mild condition, this approximation extends to the entire real number line $\R$, rather than being limited to a bounded subset.
To elucidate this, when considering any activation function $\varrho\in\tildescrA_2$, we can identify two affine linear maps $\calL_1,\calL_2:\R\to\R$ 
such that $\calL_2\circ\varrho\circ\calL_1(y)=y\cdot \tildeh(y)$ for any $y\in\R$,
where $\tildeh:\R\to\R$ is an S-shaped function with 
\begin{equation*}
    \sup_{x\in\R}|\tildeh(x)|<\infty,\quad \lim_{x\to-\infty}\tildeh(x)=0,\quad \tn{and}\quad \lim_{x\to\infty}\tildeh(x)=1. 
\end{equation*}
Therefore, by ensuring that the set $\big\{y \cdot \big(\one_{\{y > 0\}} - \tildeh(y)\big) : y \in \R\big\}$ remains bounded, we provide adequate conditions for $\varrho$ to employ a single active neuron to accurately approximate the \ReLU\ activation function on the entire real number
line $\R$.

Finally, 
let us briefly identify the types of functions that do not belong to the previously mentioned activation function sets $\scrA_{1,k}$ for $k\in\N$, $\scrA_2$, and $\scrA_3$. In essence, we aim to characterize the activation function $\varrho$ for which a fixed-size $\varrho$-activated network cannot achieve an arbitrarily accurate approximation of the \ReLU\ activation function by solely adjusting its parameters. Notably, polynomials serve as evident examples in this context. In fact, all rational functions also fall into this category. 
Below, a concise outline of the proof using the method of contradiction will be provided.
Let us assume that $\varrho$ is a rational function, and a fixed-size $\varrho$-activated network has the capability to approximate \ReLU\ arbitrarily well. It is important to note that a fixed-size $\varrho$-activated network can be represented as a rational function of a certain degree. Consequently, a rational function of a certain degree can approximate \ReLU\  arbitrarily well, and by extension, the absolute value function, which is the sum of $\ReLU(x)$ and $\ReLU(-x)$. However, this scenario leads to a contradiction since the approximation error has a lower bound that depends on the degree of the rational function when using it to approximate the absolute value function \citep[e.g., see][]{2020arXiv200502736G}.

\begin{figure}[H]
    \centering	
    \begin{subfigure}[c]{0.32455\textwidth}
    \centering            \includegraphics[width=0.998055\textwidth]{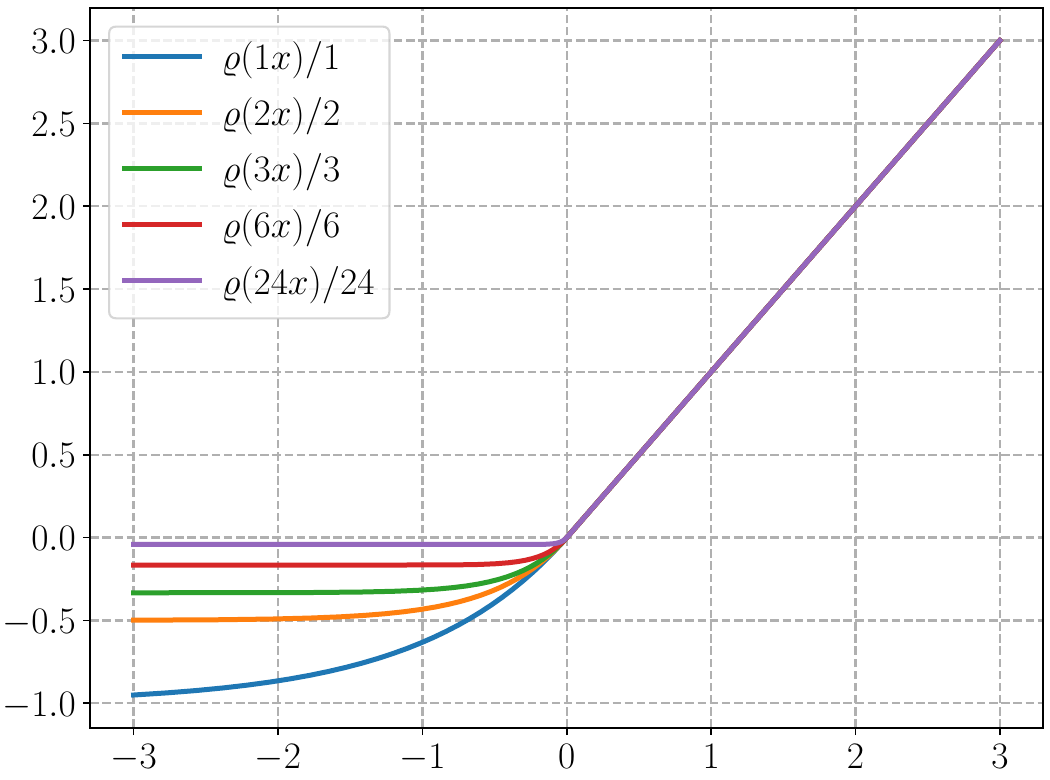}
    \subcaption{$\varrho=\ELU=\CELU\ (\alpha=1)$.}
    \end{subfigure}\hfill
             \begin{subfigure}[c]{0.32455\textwidth}
    \centering            \includegraphics[width=0.998055\textwidth]{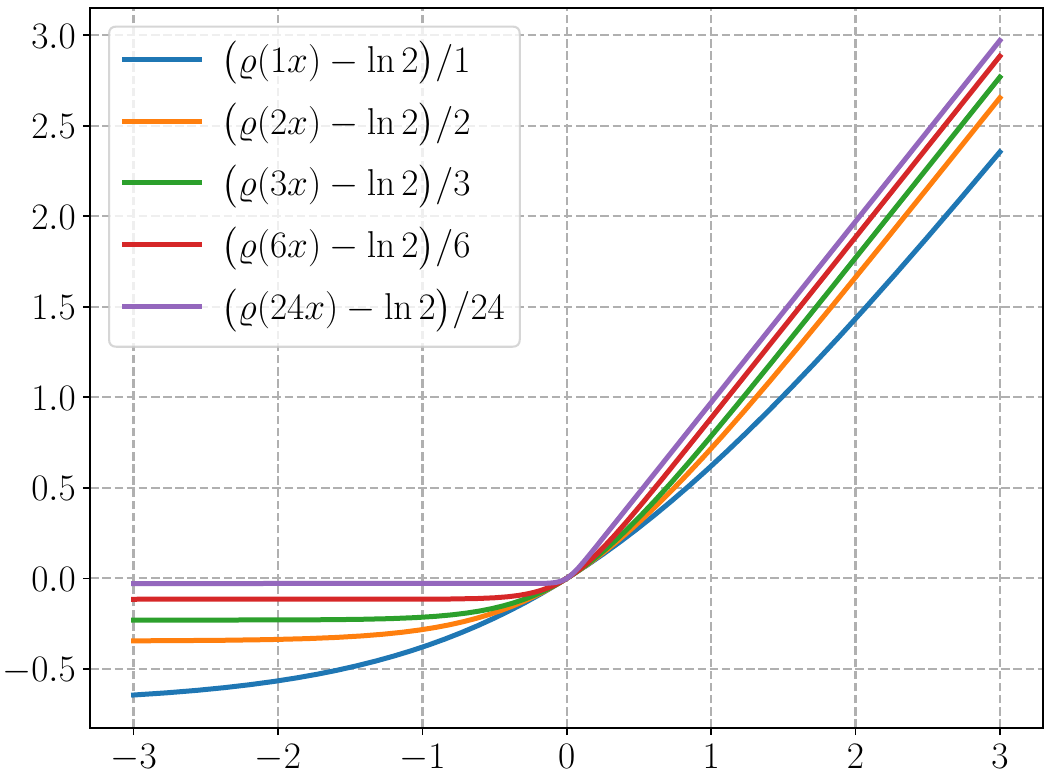}
    \subcaption{$\varrho=\Softplus$.}
    \end{subfigure}\hfill
    \begin{subfigure}[c]{0.32455\textwidth}
    \centering            \includegraphics[width=0.998055\textwidth]{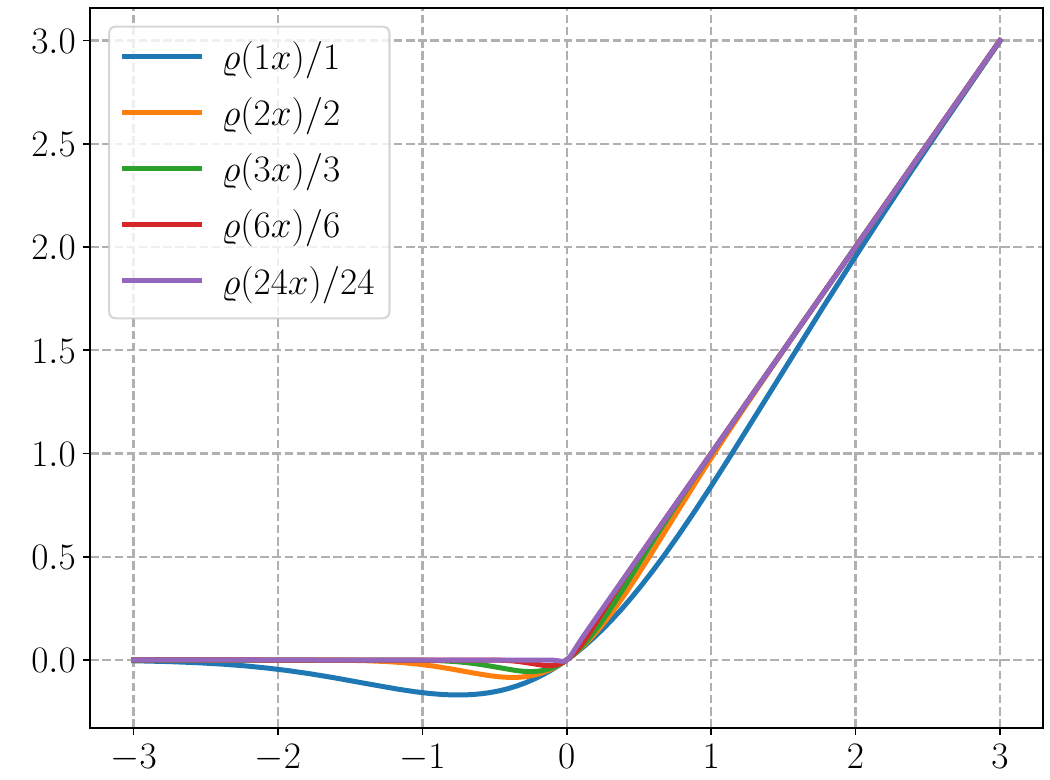}
     \subcaption{$\varrho=\GELU\  (\mu=0,\, \sigma=1)$.}
    \end{subfigure}\\  \vspace{5pt}
        \begin{subfigure}[c]{0.32455\textwidth}
    \centering            \includegraphics[width=0.998055\textwidth]{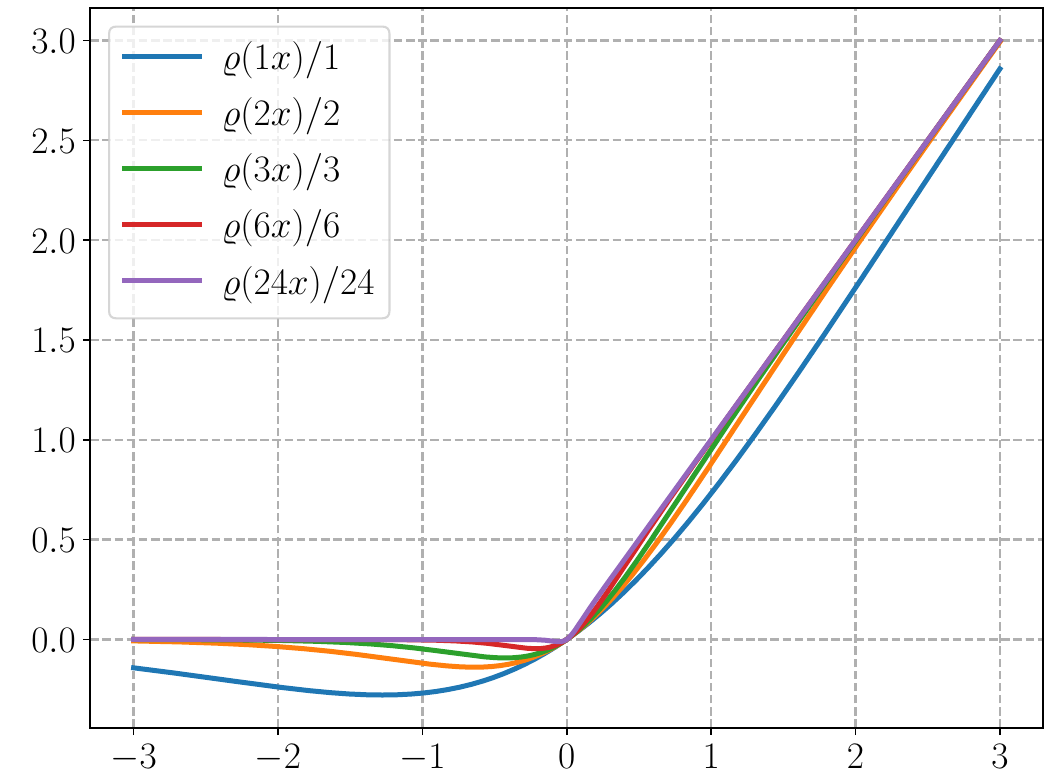}
    \subcaption{$\varrho=\SiLU$.}
    \end{subfigure}\hfill
             \begin{subfigure}[c]{0.32455\textwidth}
    \centering            \includegraphics[width=0.998055\textwidth]{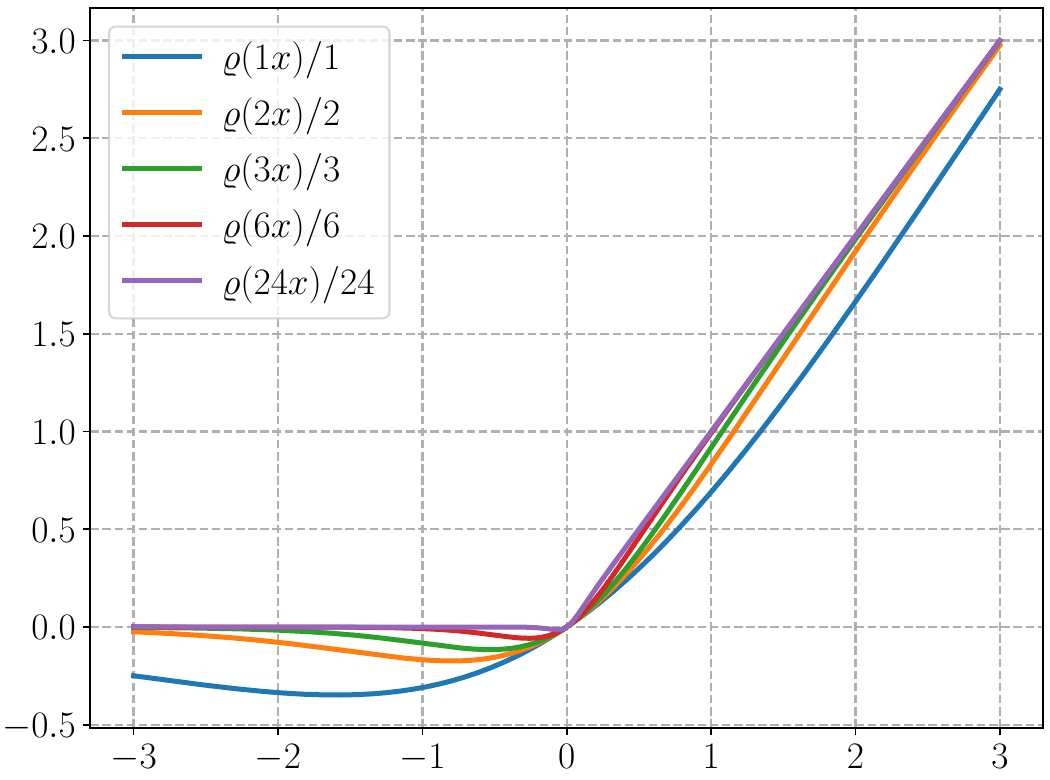}
    \subcaption{$\varrho=\Swish\ (\beta=0.8)$.}
    \end{subfigure}\hfill
    \begin{subfigure}[c]{0.32455\textwidth}
    \centering            \includegraphics[width=0.998055\textwidth]{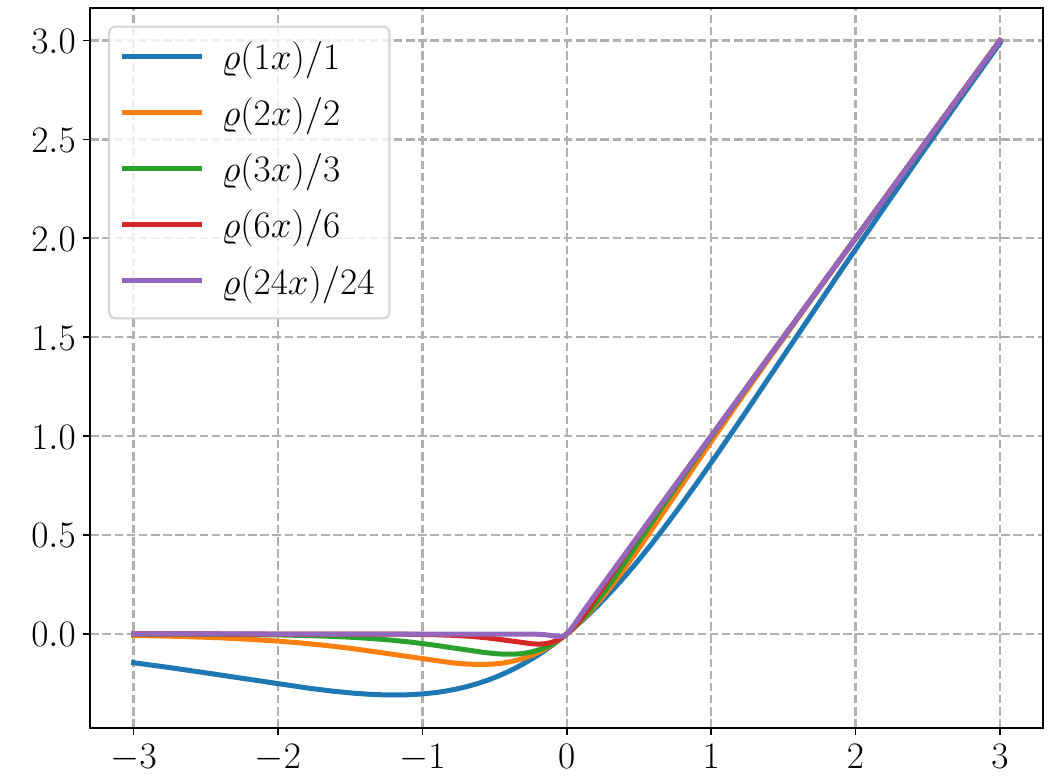}
     \subcaption{$\varrho=\Mish$.}
    \end{subfigure}\\  \vspace{5pt}
        \begin{subfigure}[c]{0.32455\textwidth}
    \centering            \includegraphics[width=0.998055\textwidth]{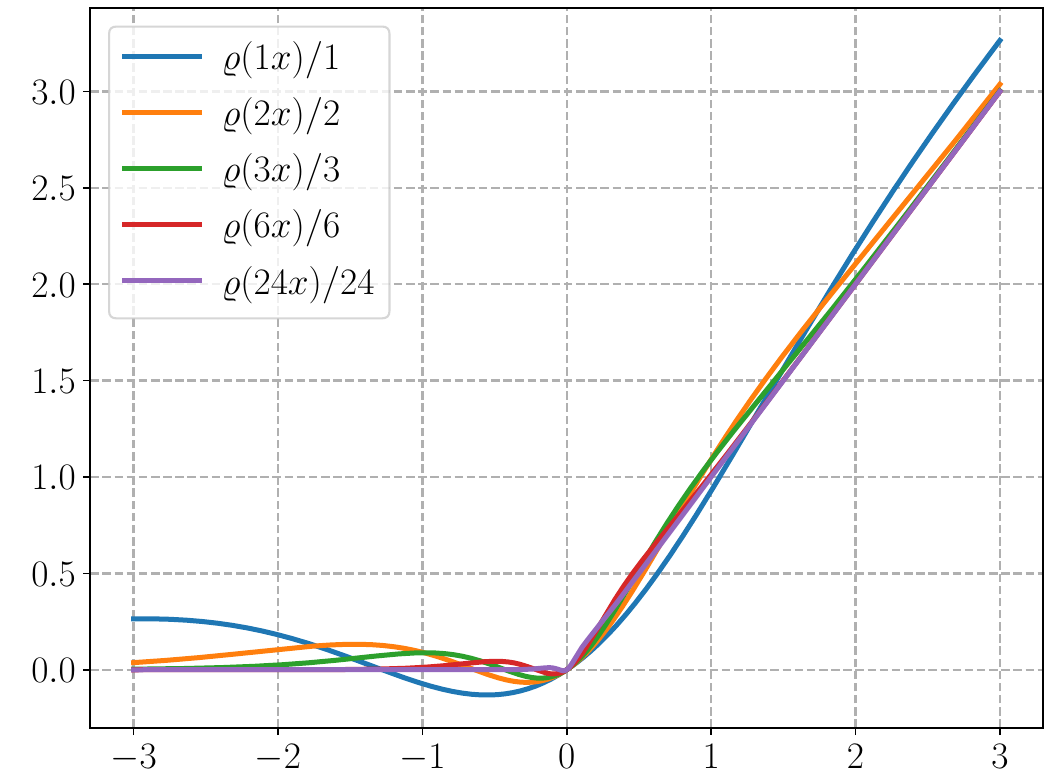}
    \subcaption{$\varrho(x)\coloneqq x\cdot\dSiLU(x)$.}
    \end{subfigure}\hfill
             \begin{subfigure}[c]{0.32455\textwidth}
    \centering            \includegraphics[width=0.998055\textwidth]{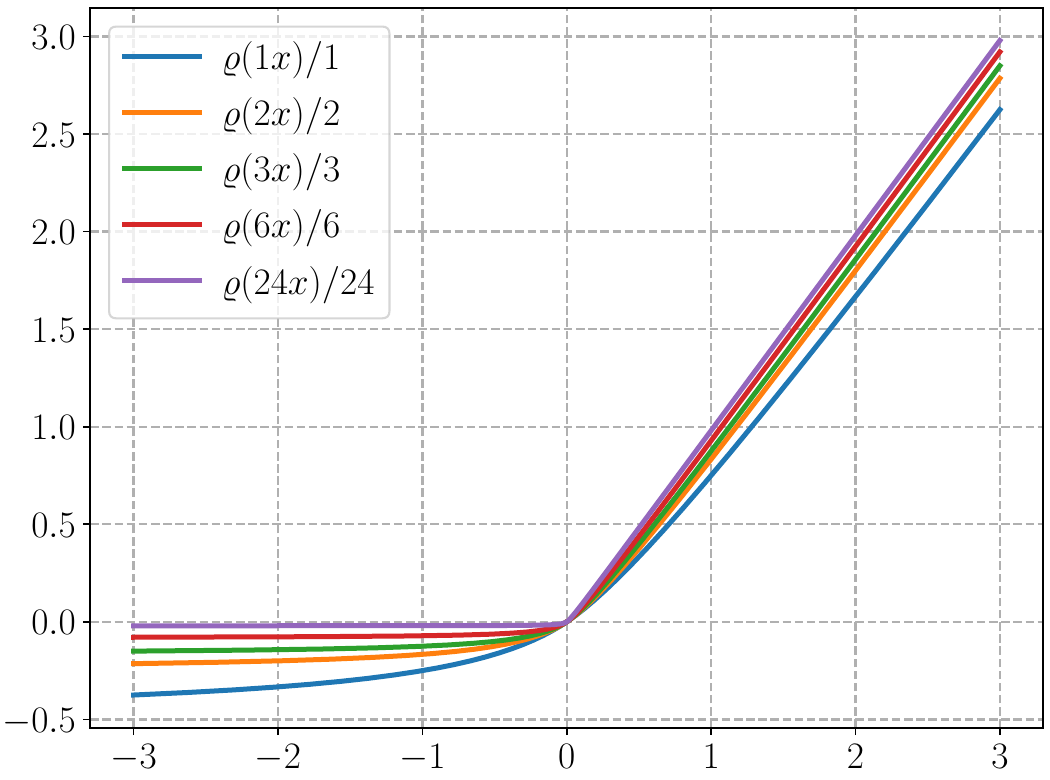}
    \subcaption{$\varrho(x)\coloneqq x\cdot\big(\tfrac{\Softsign(x)}{2}+\tfrac{1}{2}\big)$.}
    \end{subfigure}\hfill
    \begin{subfigure}[c]{0.32455\textwidth}
    \centering            \includegraphics[width=0.998055\textwidth]{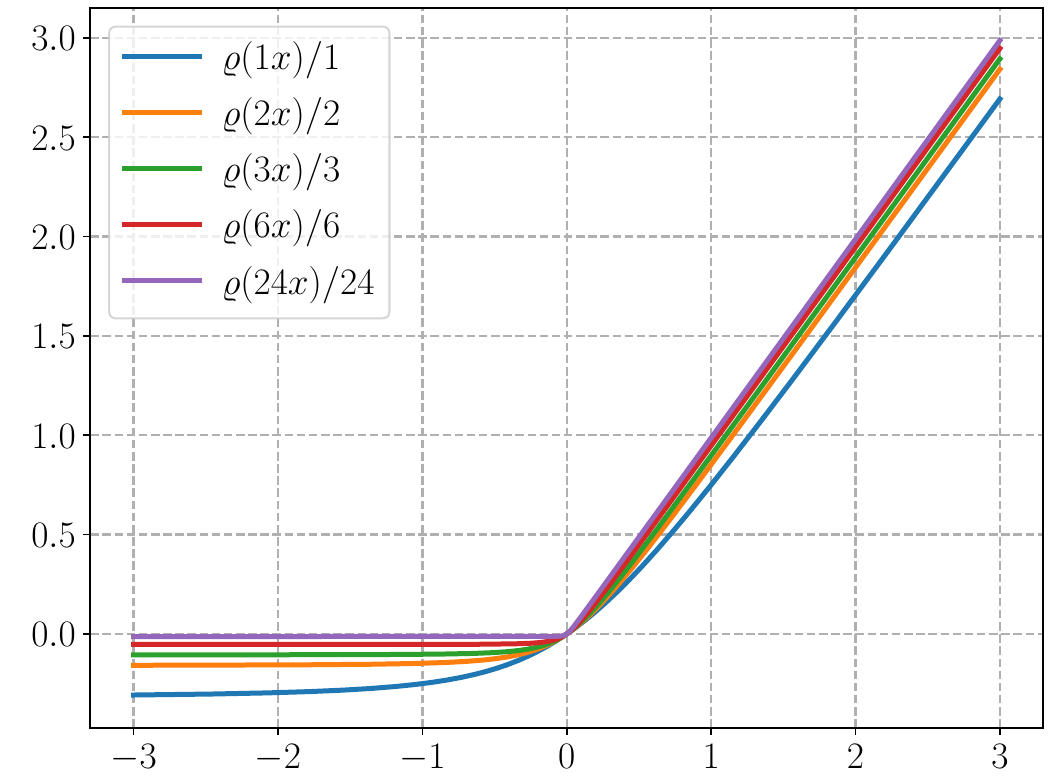}
     \subcaption{$\varrho(x)\coloneqq x\cdot\big(\tfrac{\Arctan(x)}{\pi}+\tfrac{1}{2}\big)$.}
    \end{subfigure}
    \caption{Illustrations of how a single active neuron activated by $\varrho\in \tildescrA_2$ is adequate for approximating  the \ReLU\ activation function.}
    	\label{fig:act:in:tildescrA:2:approx:ReLU}
\end{figure}

\subsection{Related Work}
\label{sec:related:work}

Extensive research has been conducted to explore the approximation capabilities of neural networks, and a multitude of publications have focused on the construction of various neural network architectures to approximate a wide range of target functions. Noteworthy examples of such studies include 
\cite{Cybenko1989ApproximationBS,HORNIK1989359,barron1993,yarotsky18a,yarotsky2017,doi:10.1137/18M118709X,ZHOU2019,10.3389/fams.2018.00014,2019arXiv190501208G,2019arXiv190207896G,suzuki2018adaptivity,Ryumei,Wenjing,Bao2019ApproximationAO,2019arXiv191210382L,MO,shijun:NonlineArpprox,shijun:Characterized:by:Numer:Neurons,shijun:smooth:functions,shijun:arbitrary:error:with:fixed:size,shijun:thesis,shijun:intrinsic:parameters}.
During the early stages of this field, the primary focus was on investigating the universal approximation capabilities of single-hidden-layer networks. The universal approximation theorem \cite{Cybenko1989ApproximationBS,HORNIK1991251,HORNIK1989359} demonstrated that when a neural network is sufficiently large, it can approximate a particular type of target function with arbitrary precision, without explicitly quantifying the approximation error in relation to the size of the network.
Subsequent research, exemplified by \cite{barron1993,barron2018approximation}, delved into analyzing the approximation error of single-hidden-layer networks with a width of $n$. These studies demonstrated an asymptotic approximation error of $\calO(n^{-1/2})$ in the $L^2$-norm for target functions possessing certain smoothness properties.

In recent years, the most widely used and effective activation function is \ReLU. 
The adoption of \ReLU\  has marked a significant improvement of results on challenging  datasets in supervised learning \cite{NIPS2012_c399862d}.
Optimizing deep neural networks activated by \ReLU\  is comparatively simpler than networks utilizing other activation functions such as \Sigmoid\ or \Tanh, since gradients can propagate when the input to \ReLU\  is positive. 
The effectiveness and simplicity of \ReLU\ have positioned it as the preferred default activation function in the deep learning community.
Extensive research has investigated the expressive capabilities of deep neural networks, with a majority of studies focusing on the \ReLU\ activation function 
\cite{yarotsky18a,yarotsky2017,shijun:NonlineArpprox,shijun:Characterized:by:Numer:Neurons,shijun:smooth:functions,shijun:RCNet,shijun:net:arc:beyond:width:depth,shijun:thesis}.
In recent advancements, several alternative activation functions have emerged as potential replacements for \ReLU. Section~\ref{sec:intro} provides numerous examples of these alternatives. Although these newly proposed activation functions have shown promising empirical results, their theoretical foundations are still being developed. The objective of this paper is to explore the expressive capabilities of deep neural networks using these activation functions. By establishing connections between these functions and \ReLU, we aim to expand most existing approximation results for \ReLU\ networks to encompass a wide range of  activation functions.
\subsection{Definitions and Illustrations of Common Activation Functions}
\label{sec:summary:activation:functions}

We will provide definitions and visual representations
of activation functions mentioned in Section~\ref{sec:intro}, including
\ReLU,  \LeakyReLU, $\ReLU^2$,  \ELU, \CELU, \SELU, \Softplus, \GELU, \SiLU, \Swish, \Mish, \Sigmoid, \Tanh, \Arctan, \Softsign, \dSiLU, and \SRS.
The definitions of these 17 activation functions are presented below.
The first 6 activation functions are given by
\begin{equation*}
\ReLU(x)=\max\{0, x\},\qquad
	\LeakyReLU(x)=\begin{cases}
		x & \tn{if}\  x\ge 0,\\
		\alpha   x &\tn{if}\   x<0
	\end{cases}
 \quad \tn{with $\alpha\in\R,$}
\end{equation*}
\begin{equation*}
	\ReLU^2(x)=(\max\{0, x\})^2,\qquad	\ELU(x)=\begin{cases}
		x & \tn{if}\ x\ge 0,\\
		\alpha(e^x-1) & \tn{if}\ x<0
	\end{cases}
\quad \tn{with $\alpha\in \R$,}
\end{equation*}
\begin{equation*}
    \CELU(x)=\begin{cases}
		x & \tn{if}\ x\ge 0,\\
		\alpha(e^{x/\alpha}-1) & \tn{if}\ x<0
	\end{cases}
\quad \tn{with $\alpha\in (0,\infty)$,}
\end{equation*}
and
\begin{equation*}
    \SELU(x)=\lambda\begin{cases}
x &\tn{if}\   x\ge 0,\\
\alpha(e^x-1)&\tn{if}\   x<0
\end{cases}\quad\tn{with $\lambda\in (0,\infty)$ and $\alpha\in\R$,}
\end{equation*}
where $e$ is the base of the natural logarithm.
For the last 6 activation functions,
\Arctan\  is the inverse tangent function and the other 5 activation functions are given by
\begin{equation*}
\Sigmoid(x)=\frac{1}{1+e^{-x}},
\qquad
	\Tanh(x)=\frac{e^x-e^{-x}}{e^x+e^{-x}},\qquad
 \Softsign(x)=\frac{x}{1+|x|},
\end{equation*}
\begin{equation*}
\dSiLU(x)=\frac{1+e^{-x}+x e^{-x}}{(1+e^{-x})^2},\qquad \tn{and}\qquad
\SRS(x)=\frac{x}{x/\alpha+e^{-x/\beta}}
\quad\tn{with $\alpha,\beta\in (0,\infty)$.}
\end{equation*}
The remaining 5 activation functions are given by
\begin{equation*}
\Softplus(x)=\ln(1+e^x),
\qquad
	\SiLU(x)=\frac{x}{1+e^{-x}},
\end{equation*}
\begin{equation*}
 \Swish(x)=\frac{x}{1+e^{-\beta  x}}\quad \tn{with $\beta\in (0,\infty)$},\qquad \Mish(x)=x\cdot\Tanh\big(\Softplus(x)\big),
\end{equation*}
 and
\begin{equation*}
	\GELU(x)=x\int_{-\infty}^{x} \tfrac{1}{\sigma\sqrt{2\pi}}e^{-\frac{1}{2}(\frac{t-\mu}{\sigma})^2}d t\quad\tn{with $\mu\in\R$ and $\sigma\in (0,\infty)$.}
\end{equation*}
Refer to Figure~\ref{fig:scrA:egs:in:intro} for visual representations of all these activation functions.

\begin{figure}[htbp!]
    \centering	
    \begin{subfigure}[c]{0.32455\textwidth}
    \centering            \includegraphics[width=0.998055\textwidth]{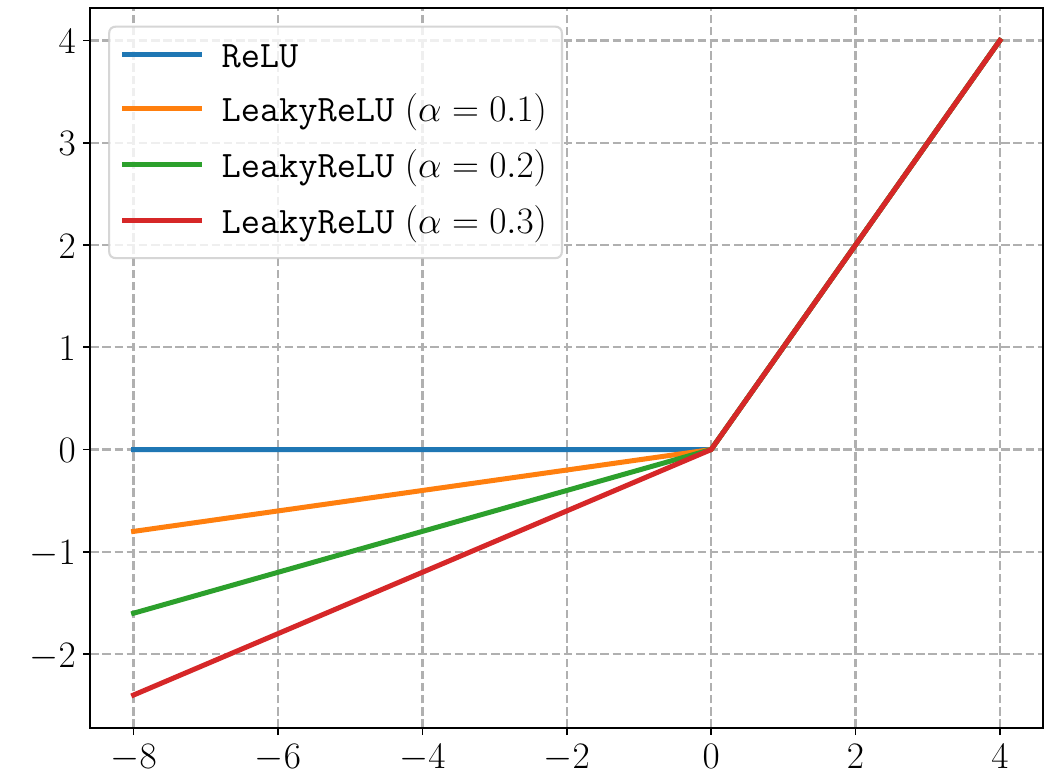}
    \end{subfigure}\hfill
             \begin{subfigure}[c]{0.32455\textwidth}
    \centering            \includegraphics[width=0.998055\textwidth]{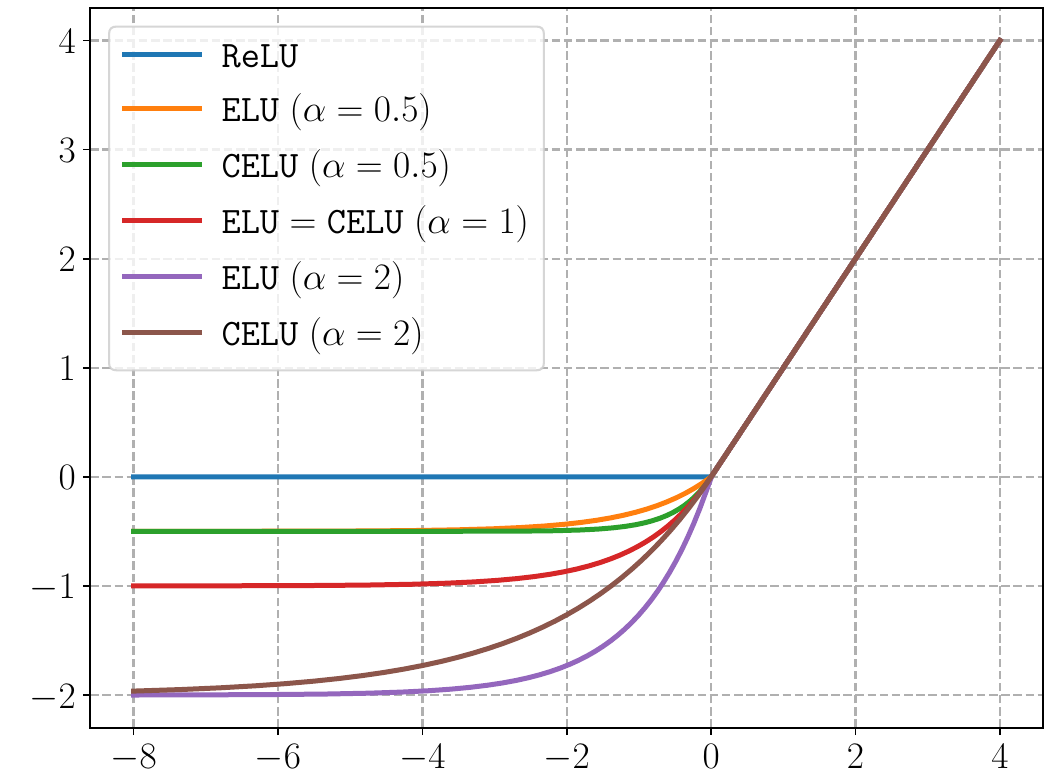}
    \end{subfigure}\hfill
    \begin{subfigure}[c]{0.32455\textwidth}
    \centering            \includegraphics[width=0.998055\textwidth]{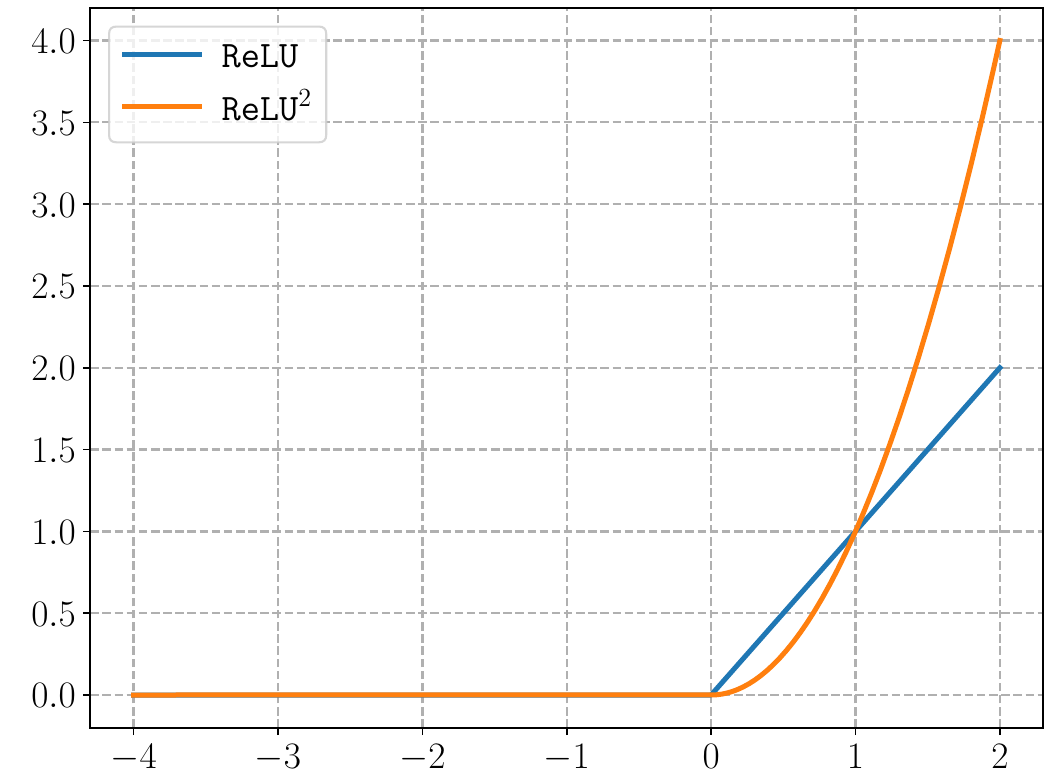}
    \end{subfigure}\\   \vspace{11.8pt}
    \begin{subfigure}[c]{0.32455\textwidth}
    \centering            \includegraphics[width=0.998055\textwidth]{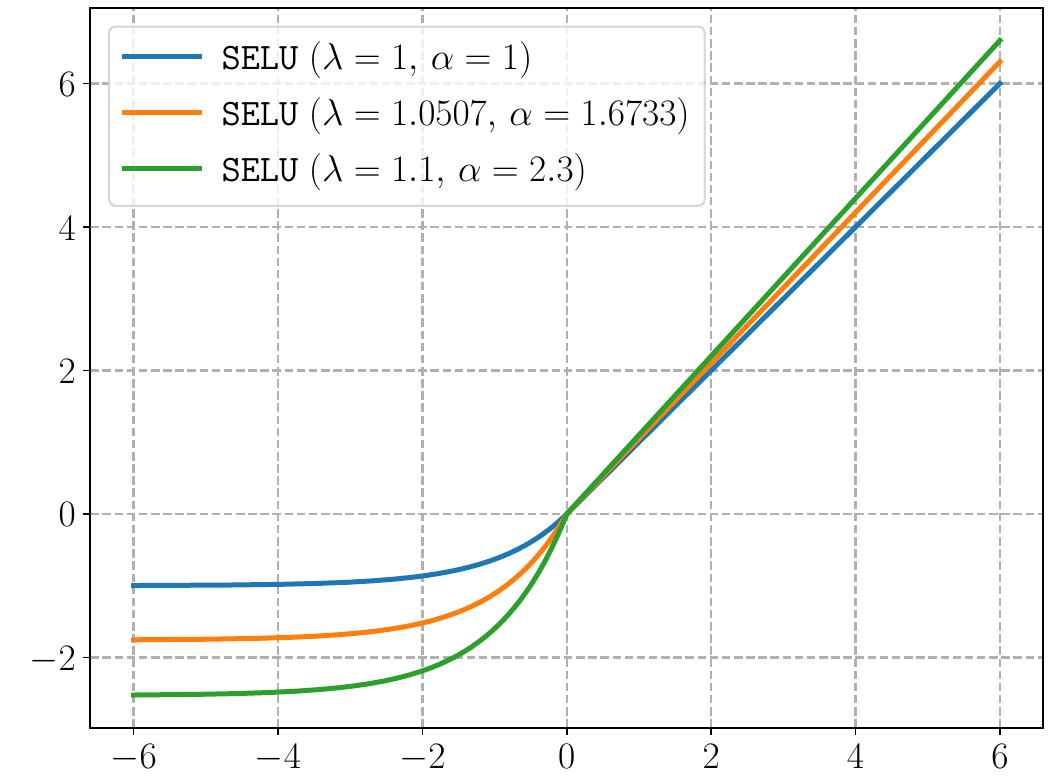}
    \end{subfigure}\hfill
    \begin{subfigure}[c]{0.32455\textwidth}
    \centering            \includegraphics[width=0.998055\textwidth]{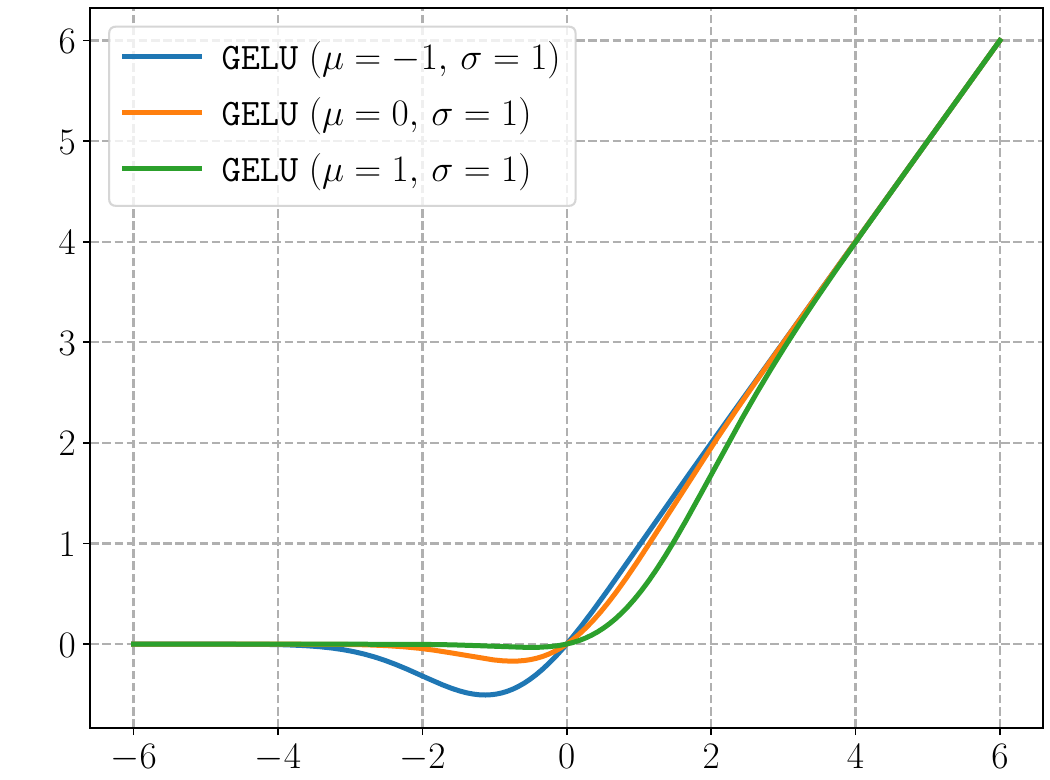}
    \end{subfigure}\hfill
    \begin{subfigure}[c]{0.32455\textwidth}
    \centering            \includegraphics[width=0.998055\textwidth]{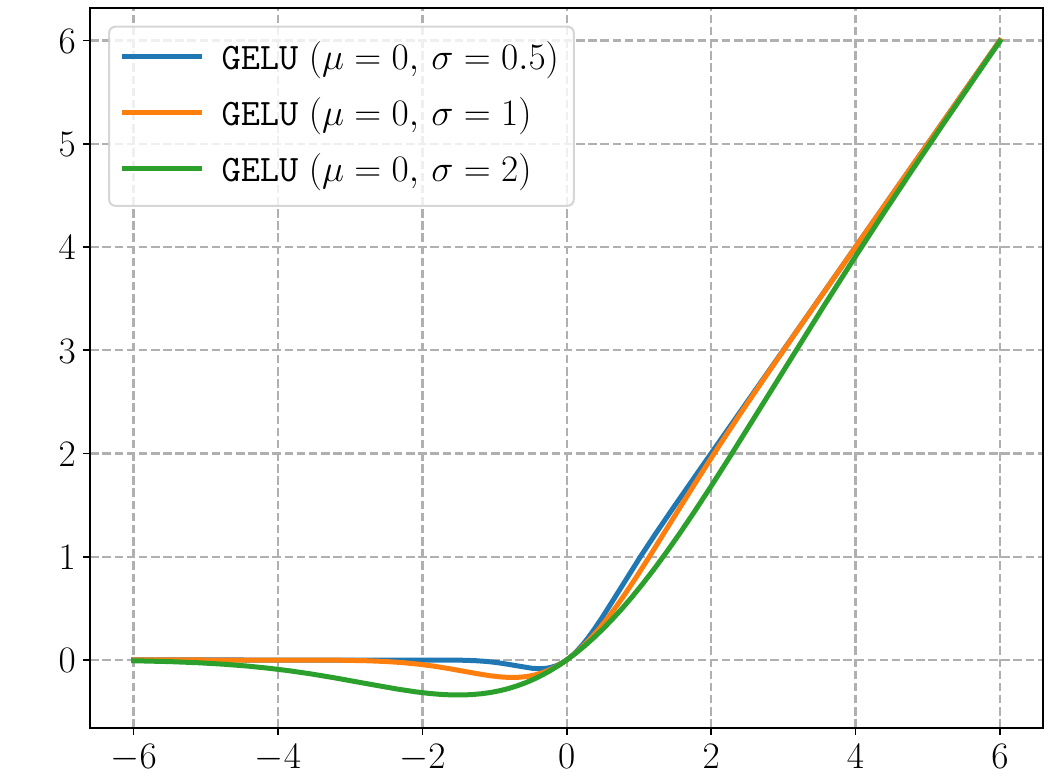}
    \end{subfigure}
\\   \vspace{11.8pt}
    \begin{subfigure}[c]{0.32455\textwidth}
    \centering            \includegraphics[width=0.998055\textwidth]{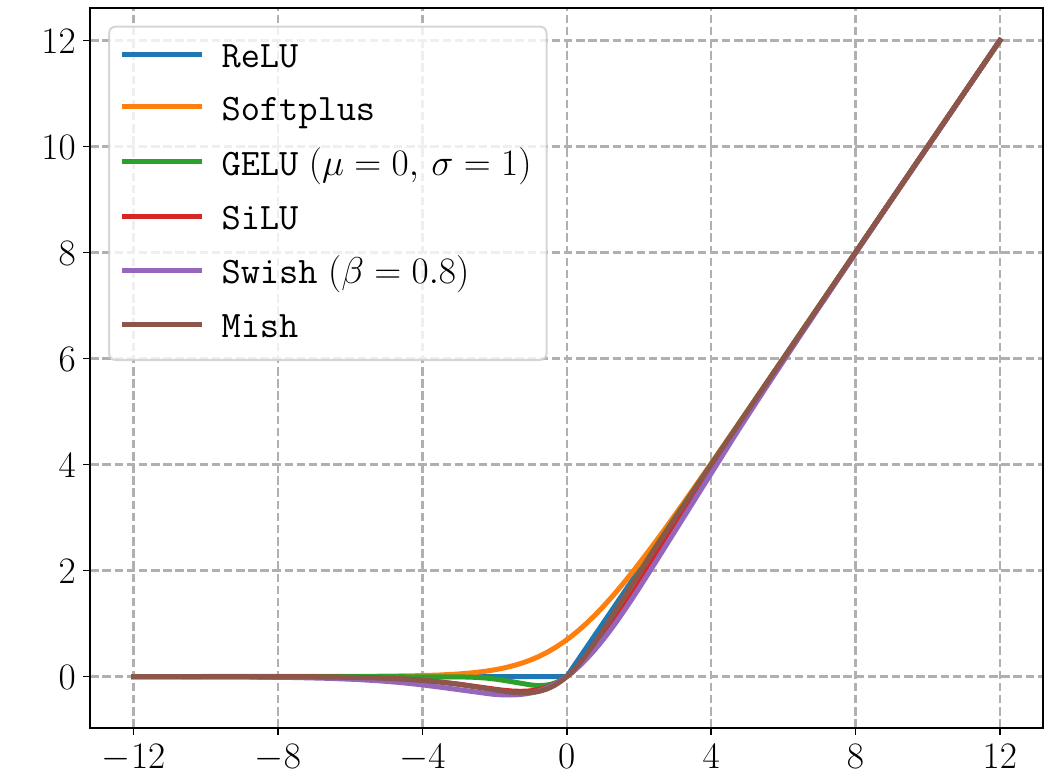}
    \end{subfigure}\hfill
        \begin{subfigure}[c]{0.32455\textwidth}
    \centering            \includegraphics[width=0.998055\textwidth]{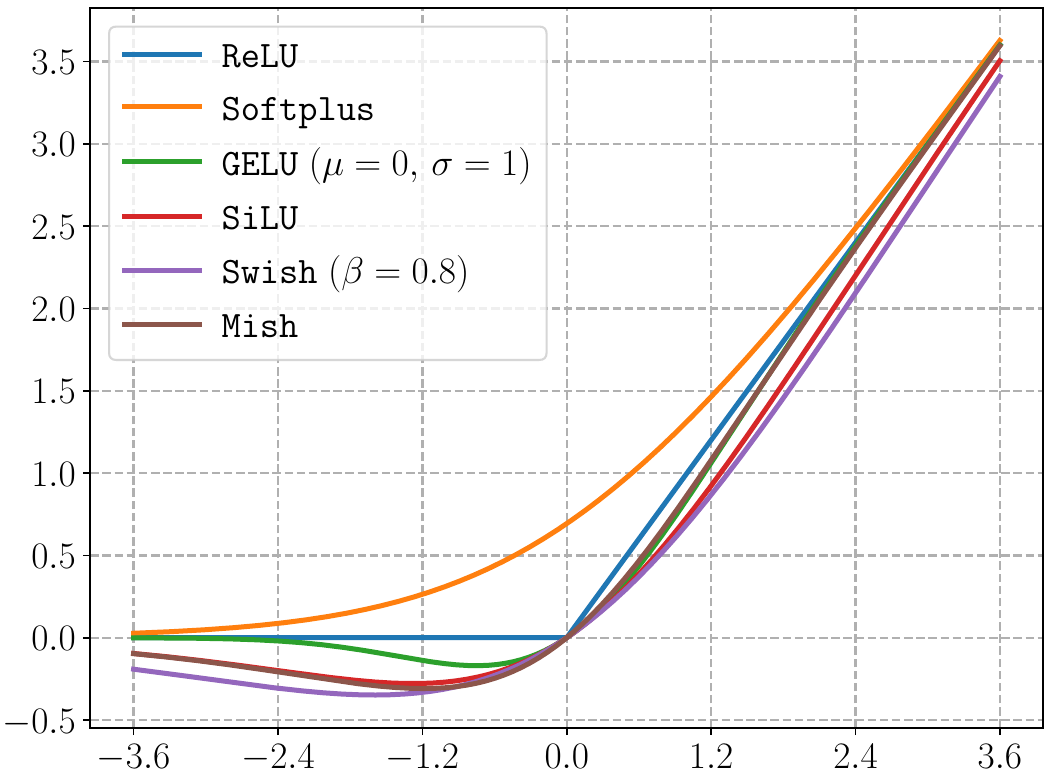}
    \end{subfigure}\hfill
       \begin{subfigure}[c]{0.32455\textwidth}
    \centering            \includegraphics[width=0.998055\textwidth]{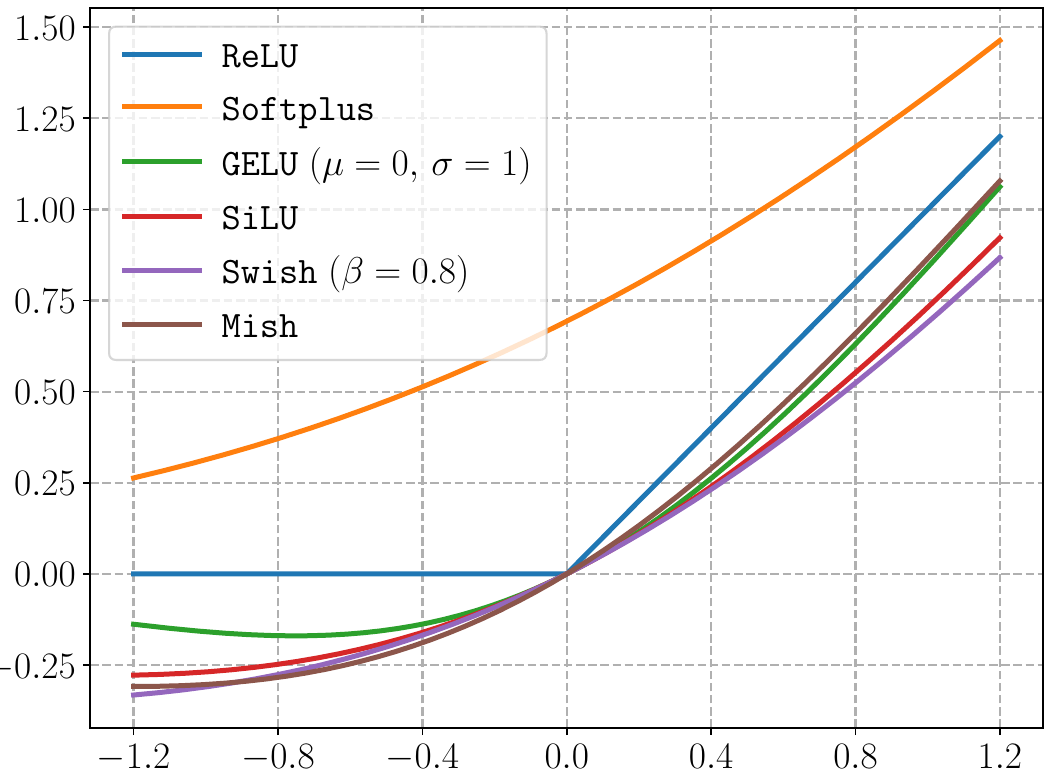}
    \end{subfigure}
    \\   \vspace{11.8pt}
        \begin{subfigure}[c]{0.32455\textwidth}
    \centering            \includegraphics[width=0.998055\textwidth]{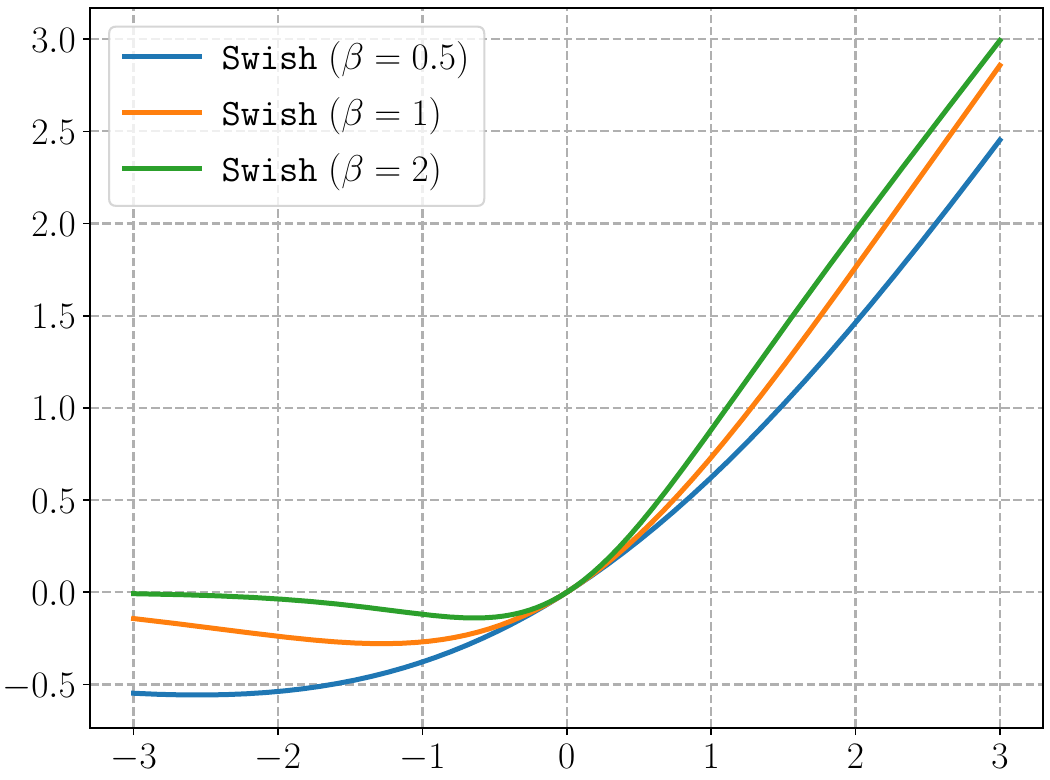}
    \end{subfigure}\hfill
    \begin{subfigure}[c]{0.32455\textwidth}
    \centering            \includegraphics[width=0.998055\textwidth]{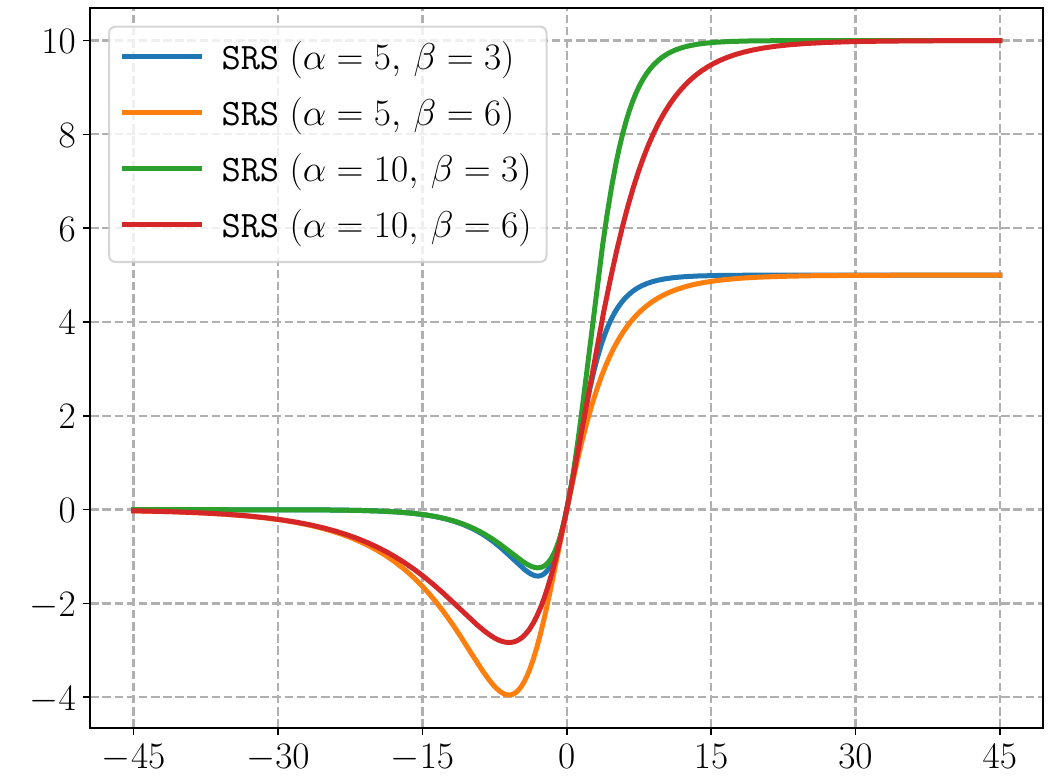}
    \end{subfigure}\hfill
    \begin{subfigure}[c]{0.32455\textwidth}
    \centering            \includegraphics[width=0.998055\textwidth]{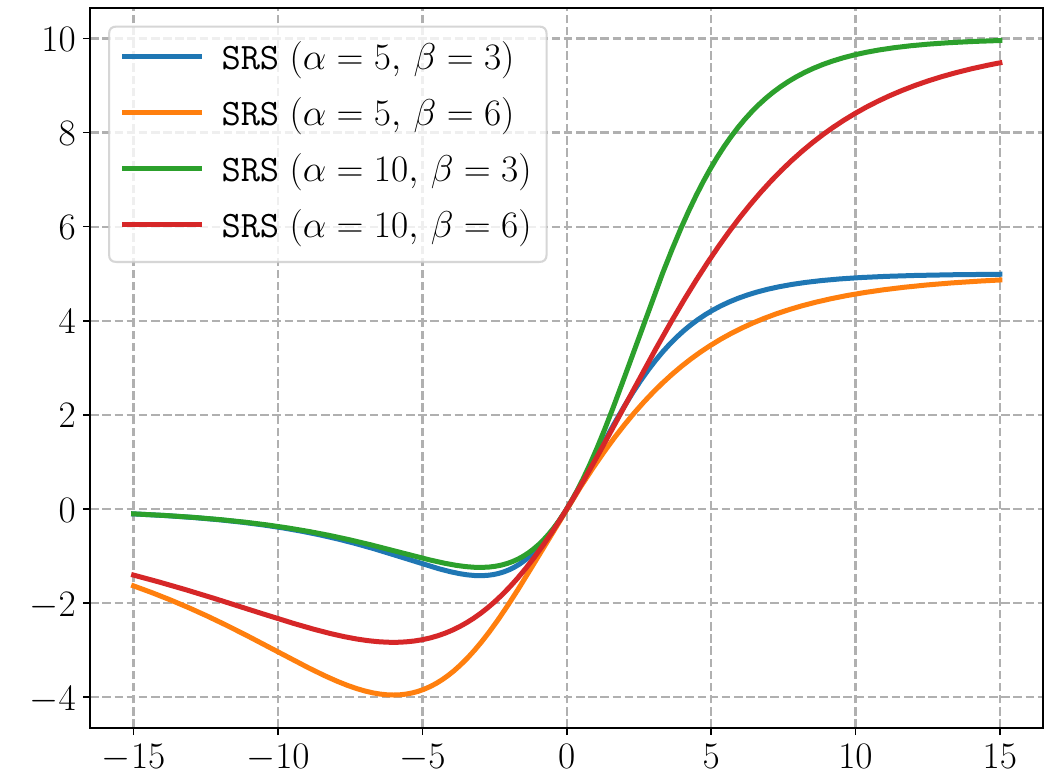}
    \end{subfigure}\\   \vspace{11.8pt}
    \begin{subfigure}[c]{0.32455\textwidth}
    \centering            \includegraphics[width=0.998055\textwidth]{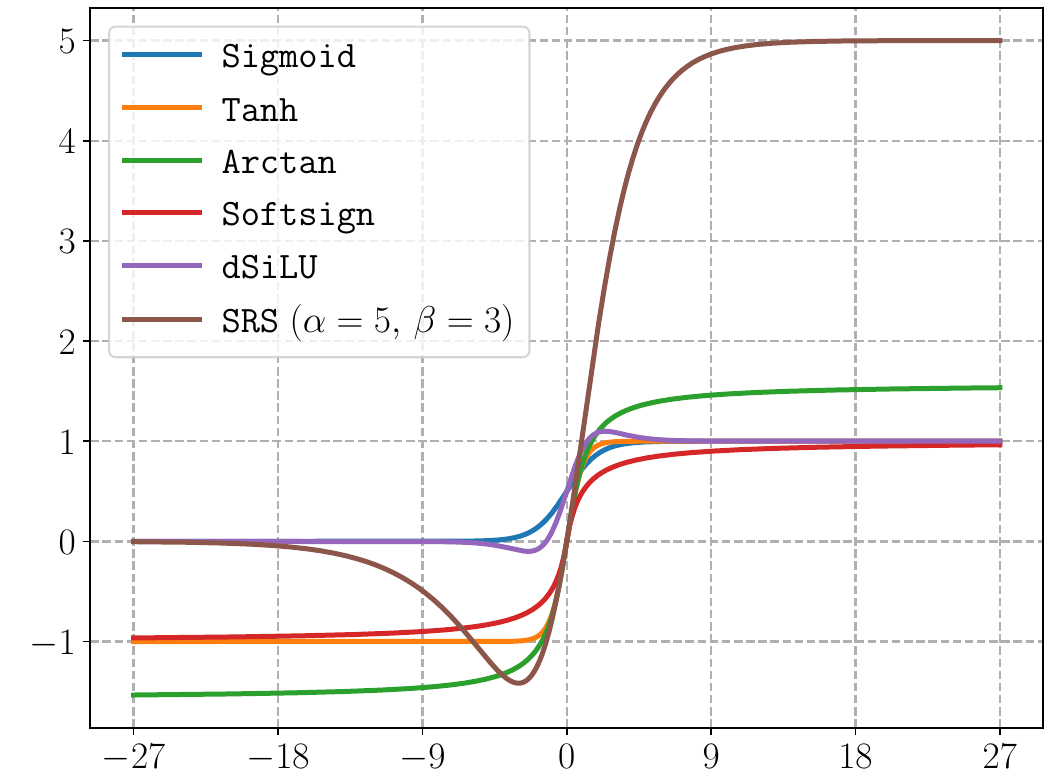}
    \end{subfigure}\hfill
        \begin{subfigure}[c]{0.32455\textwidth}
    \centering     
    \includegraphics[width=0.998055\textwidth]{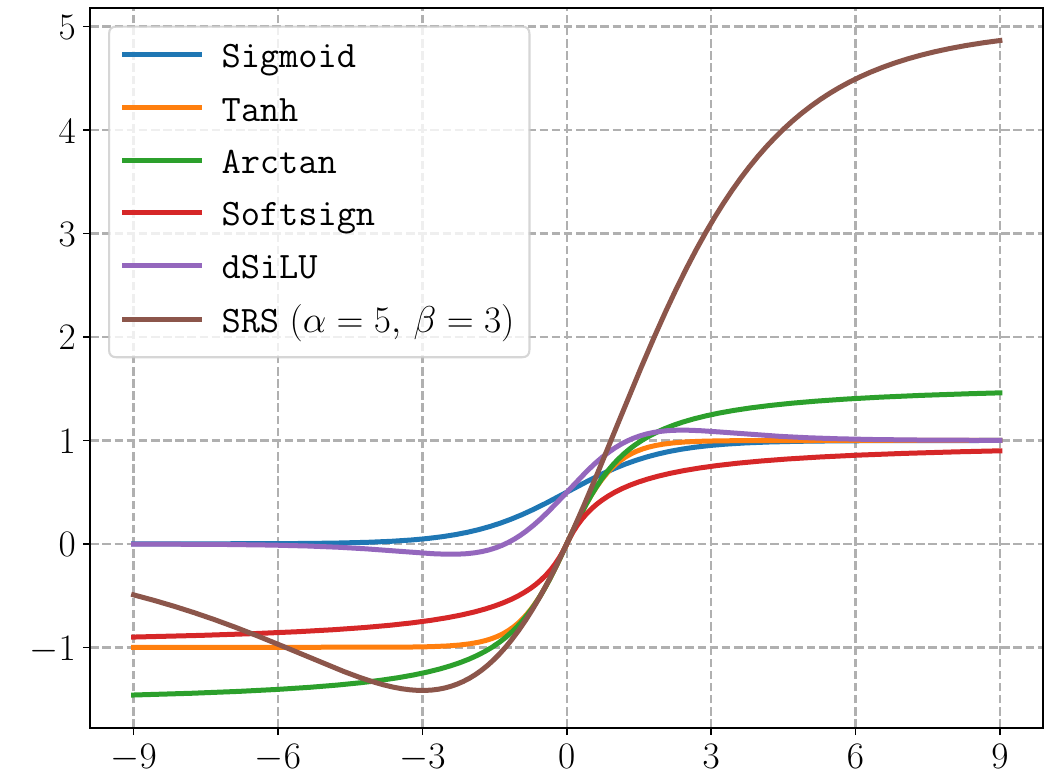}
    \end{subfigure}\hfill
            \begin{subfigure}[c]{0.32455\textwidth}
    \centering     
    \includegraphics[width=0.998055\textwidth]{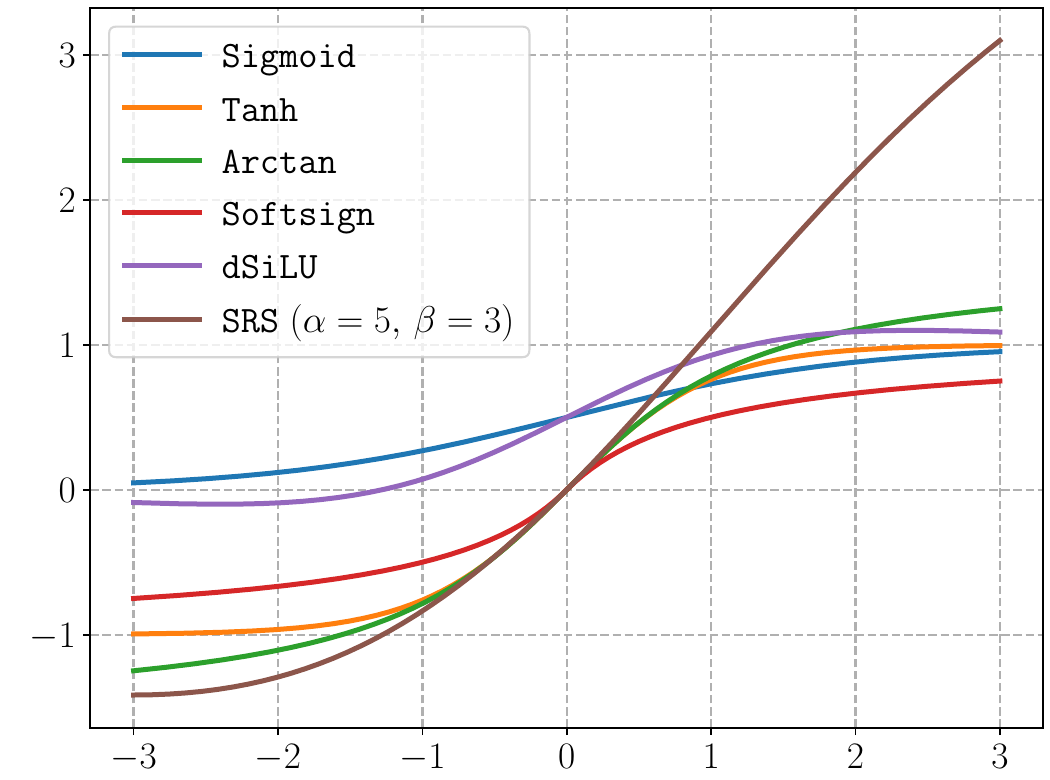}
    \end{subfigure}\\ \vspace{6pt}
    \caption{Illustrations of \ReLU,  \LeakyReLU, $\ReLU^2$,  \ELU, \CELU, \SELU, \Softplus, \GELU, \SiLU, \Swish, \Mish, \Sigmoid, \Tanh, \Arctan, \Softsign, \dSiLU, and \SRS.}
    	\label{fig:scrA:egs:in:intro}
\end{figure}

\section{Proofs of Theorems in Sections~\ref{sec:intro} and \ref{sec:further:discussion}}
\label{sec:proof:thms}

In this section, we will prove the theorems in Sections~\ref{sec:intro} and \ref{sec:further:discussion}, i.e., Theorems~\ref{thm:main}, \ref{thm:main:kth:derivative}, \ref{thm:main:scrA:1k}, \ref{thm:main:scrA:2}, and \ref{thm:main:scrA:2:new}.
To enhance clarity, Section~\ref{sec:notation} offers a concise overview of the notations employed throughout this paper.
Next in Section~\ref{sec:props:proof:thms}, we present the ideas for proving Theorems~\ref{thm:main}, \ref{thm:main:kth:derivative}, \ref{thm:main:scrA:1k}, \ref{thm:main:scrA:2}, and \ref{thm:main:scrA:2:new}.
Moreover, to simplify the proofs, we establish several propositions, which will be proved in later sections.
By assuming the validity of these propositions,
we provide the proof of Theorem~\ref{thm:main}
in Section~\ref{sec:proof:thm:main}
and give the proofs of 
Theorems~\ref{thm:main:kth:derivative}, \ref{thm:main:scrA:1k}, \ref{thm:main:scrA:2}, and \ref{thm:main:scrA:2:new} in Section~\ref{sec:proof:thms:main:others}.



\subsection{Notations}
\label{sec:notation}

The following is an overview of the basic notations used in this paper.
\begin{itemize}
	\item The set difference of two sets $A$ and $B$ is denoted as $A\backslash B\coloneqq\{x:x\in A,\ x\notin B\}$. 
	
	\item 
The symbols $\N$, $\Z$, $\Q$, and $\R$ are used to denote the sets of natural numbers (including $0$), integers, rational numbers, and real numbers, respectively. The set of positive natural numbers is denoted as $\N^+=\N\backslash\{0\}=\{1,2,3,\cdots\}$.

 \item The base of the natural logarithm is denoted as $e$, i.e., $e=\lim_{n\to\infty}(1+\tfrac{1}{n})^n\approx 2.71828$.
\item The indicator (or characteristic) function of a set $A$, denoted by $\one_{A}$, is a function that takes the value $1$ for elements of $A$ and $0$ for elements not in $A$.
	
	
	\item 
 The floor and ceiling functions of a real number $x$ can be represented as
	$\lfloor x\rfloor=\max \{n: n\le x,\ n\in \Z\}$ and $\lceil x\rceil=\min \{n: n\ge x,\ n\in \Z\}$.
	

	
	\item Let $\tbinom{n}{k}$ denote the coefficient of the $x^k$ term in the polynomial expansion of the binomial power $(1+x)^n$ for any $n,k\in\N$ with $n\ge k$, i.e., $\tbinom{n}{k}=\tfrac{n!}{k!(n-k)!}$.

	\item Vectors are denoted by bold lowercase letters, such as $\bma=(a_1,\cdots,a_d)\in\R^d$. On the other hand, matrices are represented by bold uppercase letters. For example, $\bm{A}\in\mathbb{R}^{m\times n}$ refers to a real matrix of size $m\times n$, and $\bm{A}^\ts$ denotes the transpose of matrix $\bm{A}$.
	
	
	
	\item Given any $p\in [1,\infty]$, the $p$-norm (also known as $\ell^p$-norm) of a vector $\bmx=(x_1,\cdots,x_d)\in\R^d$ is defined via
	\begin{equation*}
		\|\bmx\|_p=\|\bmx\|_{\ell^p}\coloneqq \big(|x_1|^p+\cdots+|x_d|^p\big)^{1/p}\quad \tn{if $p\in [1,\infty)$}
	\end{equation*}
	and
	\begin{equation*}		\|\bmx\|_{\infty}=\|\bmx\|_{\ell^\infty}\coloneqq \max\big\{|x_i|: i=1,2,\cdots,d\big\}.
	\end{equation*}
	
		\item Let ``$\rightrightarrows$" denote the uniform convergence. For example, if $\bmf:\R^d\to\R^n$ is a vector-valued function and $\bmf_\delta(\bmx)\rightrightarrows \bmf(\bmx)$ as $\delta\to 0^+$ for any $\bmx\in \Omega\subseteq \R^d$, then
	for any $\eps>0$, there exists $\delta_\eps\in (0,1)$ such that
 \begin{equation*}		
 \|\bmf_\delta-\bmf\|_{\sup(\Omega)}< \eps\quad \tn{for any $\delta\in (0,\delta_\eps)$.}
	\end{equation*}

 \item A network is labeled as ``a network of width $N$ and depth $L$'' when it satisfies the following  two conditions.
\begin{itemize}
\item The count of neurons in each hidden layer of the network does not exceed $N$.
\item The total number of hidden layers in the network is at most $L$.
\end{itemize}
	
	
	\item 
	Suppose $\bmphi:\R^d\to\R^n$ is a vector-valued function realized by a $\varrho$-activated network. Then $\bmphi$ can be expressed as
	\begin{equation*}
		\begin{aligned}
			\bm{x}=\widetilde{\bm{h}}_0 
			\myto{2.0242}^{\bmW_0,\ \bm{b}_0}_{\calbmL_0} \bm{h}_1
			\myto{1.13015}^{\varrho} \widetilde{\bm{h}}_1 \quad \cdots\quad \myto{2.97}^{\bmW_{L-1},\ \bm{b}_{L-1}}_{\calbmL_{L-1}} \bm{h}_L
			\myto{1.13015}^{\varrho} \widetilde{\bm{h}}_L
			\myto{2.342}^{\bmW_{L},\ \bm{b}_{L}}_{\calbmL_L} \bm{h}_{L+1}=\bmphi(\bm{x}),
		\end{aligned}
	\end{equation*}
	where $N_0=d$, $N_1,N_2,\cdots,N_L\in\N^+$,  
$N_{L+1}=n$, and $\bmcalL_i$ is an affine linear map given by $\calbmL_i:\bmx\mapsto\bmW_i \bmx +\bmb_i$ with $\bmW_i\in \R^{N_{i+1}\times N_{i}}$ and $\bm{b}_i\in \R^{N_{i+1}}$ being the weight matrix and the bias vector, respectively, for $i=0,1,\cdots,L$. Here
	\[\bm{h}_{i+1} =\calbmL_i(\tildebmh_{i})=\bmW_i\cdot \tildebmh_{i} + \bm{b}_i \quad \tn{for $i=0,1,\cdots,L$}\]  
	and
	\[	\widetilde{\bm{h}}_i=\varrho(\bm{h}_i)\quad \tn{for $i=1,2,\cdots,L$,}
	\]
 where $\varrho$ is the activation function that can be applied elementwise to a vector input.
 Clearly, $\bmphi\in \nn{\varrho}{N}{L}{d}{n}$, where $N=\max\{N_1,N_2,\cdots,N_L\}$.
Furthermore, $\bmphi$ can be expressed as a composition of functions
	\begin{equation*}
		\bmphi =\calbmL_L\circ\varrho\circ
				\calbmL_{L-1}\circ 
		\ \cdots \  \circ 
				\varrho\circ
		\calbmL_1\circ\varrho\circ\calbmL_0.
	\end{equation*}
	Refer to Figure~\ref{fig:varrhoNetEg} for an illustration.
	
	\begin{figure}[htbp!]     		
		\centering            \includegraphics[width=0.75\textwidth]{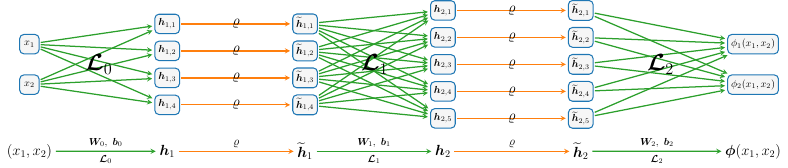}
		\caption{An example of a $\varrho$-activated network of width $5$ and depth $2$. The network realizes a vector-valued function $\bmphi=(\phi_1,\phi_2)$.  Here, $\bmh_{i,j}$ (or $\tildebmh_{i,j}$) represents the $j$-th entry of $\bmh_{i}$ (or $\tildebmh_{i}$) for $(i,j)\in 
  \big\{(1,j):j=1,2,3,4\big\}\cup \big\{(2,j):j=1,2,3,4,5\big\}$. }
		\label{fig:varrhoNetEg}	
   \vspace*{-2.9pt}
	\end{figure}

\end{itemize}


\subsection{Propositions for Proving Theorems in Sections~\ref{sec:intro} and \ref{sec:further:discussion}}
\label{sec:props:proof:thms}

We now present the key ideas for proving theorems introduced in Sections~\ref{sec:intro} and \ref{sec:further:discussion}, i.e., Theorems~\ref{thm:main}, \ref{thm:main:kth:derivative}, \ref{thm:main:scrA:1k}, \ref{thm:main:scrA:2}, and \ref{thm:main:scrA:2:new}.
These five theorems collectively convey a narrative wherein a $\tildevarrho$-activated network can be accurately approximated by a $\varrho$-activated network, provided certain assumptions are met regarding $\varrho$ and $\tildevarrho$. Consequently, it becomes imperative to establish an auxiliary theorem that allows for the substitution of the network's activation function(s) at the cost of a sufficiently small error.

\begin{proposition}\label{prop:activation:replace}
	Given two functions $\varrho,\tildevarrho:\R\to\R$ with $\tildevarrho\in C(\R)$, suppose for any $M>0$, there exists
	$\tildevarrho_\eta\in \nnOneD[\big]{\varrho}{\tildeN}{\tildeL}{\R}{\R}$ for each $\eta\in (0,1)$
	such that
	\begin{equation*}
		\tildevarrho_\eta(x)\rightrightarrows \tildevarrho(x)\quad \tn{as}\  \eta\to 0^+\quad \tn{for any $x\in [-M,M]$.}
	\end{equation*}
	Assuming $\bmphi_{\tildevarrho}\in \nn[\big]{\tildevarrho}{N}{L}{d}{n}$, for any $\eps>0$ and $A>0$, there exists
	$\bmphi_{\varrho}\in \nn[\big]{\varrho}{\tildeN\cdot  N}{\  \tildeL\cdot  L}{d}{n}$
	such that
	\begin{equation*}
		\big\|\bmphi_\varrho-\bmphi_\tildevarrho\big\|_{\sup([-A,A]^d)}<\eps.
	\end{equation*}
\end{proposition}

The proof of Proposition~\ref{prop:activation:replace} can be found in Section~\ref{sec:proof:prop:activation:replace}.
The utilization of Proposition~\ref{prop:activation:replace} simplifies our task of proving Theorems~\ref{thm:main}, \ref{thm:main:kth:derivative}, \ref{thm:main:scrA:1k}, \ref{thm:main:scrA:2}, and \ref{thm:main:scrA:2:new}. Our focus now shifts to constructing $\varrho$-activated networks that can effectively approximate both $\varrho^{(k)}$ (assuming $\varrho\in C^k(\R)$) and \ReLU. To facilitate this construction process, we introduce the following three propositions.


\begin{proposition}
	\label{prop:approx:f:nth:D}
	Given any $n\in \N$ and $a_0<a<b<b_0$, if	
	$f\in C^n\big((a_0,b_0)\big)$, then
	\begin{equation*}
		\frac{\sum_{\ell=0}^{n}(-1)^\ell\binom{n}{\ell} f(x+\ell  t) }{(-t)^n} 
		\rightrightarrows f^{(n)}(x)\quad\tn{as}\  t\to0\quad \tn{for any $x\in [a,b]$.}
	\end{equation*}
\end{proposition}

\begin{proposition}
	\label{prop:approx:ReLU:scrA:1k}
	Given any $M>0$, $k\in \N$, and $\varrho\in \scrA_{1,k}$, there exists
	$\phi_\eps\in \nnOneD{\varrho}{k+2}{1}{\R}{\R}$ for each $\eps\in (0,1)$
	such that
	\begin{equation*}
		\phi_\eps(x)\rightrightarrows \ReLU(x)\quad \tn{as}\   \eps\to 0^+\quad \tn{for any $x\in [-M,M]$.}
	\end{equation*}
\end{proposition}

\begin{proposition}
	\label{prop:approx:ReLU:scrA:2:3}
	Given any $M>0$, for each $\eps\in (0,1)$, there exists
	\begin{equation*}
		\phi_\eps\in 
  \begin{cases}
      \nnOneD{\varrho}{1}{1}{\R}{\R} & \tn{if}\  \varrho\in \tildescrA_{2},\\
      \nnOneD{\varrho}{2}{1}{\R}{\R} & \tn{if}\  \varrho\in \scrA_{2},\\
      \nnOneD{\varrho}{3}{2}{\R}{\R} & \tn{if}\  \varrho\in \scrA_3
  \end{cases}
	\end{equation*}  
	such that
	\begin{equation*}
		\phi_\eps(x)\rightrightarrows \ReLU(x)\quad \tn{as}\   \eps\to 0^+\quad \tn{for any $x\in [-M,M]$.}
	\end{equation*}
\end{proposition}
Propositions~\ref{prop:approx:f:nth:D}, \ref{prop:approx:ReLU:scrA:1k}, and \ref{prop:approx:ReLU:scrA:2:3} will be proved in Sections~\ref{sec:proof:prop:approx:f:nth:D}, \ref{sec:proof:prop:approx:ReLU:scrA:1k}, and \ref{sec:proof:prop:approx:ReLU:scrA:23}, respectively. 
Let us briefly discuss the key ideas for proving these three propositions.

The essence of proving Proposition~\ref{prop:approx:f:nth:D} lies in the application of Cauchy's mean value theorem. Through repeated utilization of such a theorem, we can establish the existence of
$|t_n|\in (0,|t|)$ such that
\begin{equation*}
    \frac{\sum_{\ell=0}^{n}(-1)^\ell\binom{n}{\ell} f(x+\ell  t) }{(-t)^n}
    =    \frac{\sum_{\ell=0}^{n}(-1)^\ell\binom{n}{\ell}\ell^n f^{(n)}(x+\ell  t_n) }{(-1)^n\, n!}.
\end{equation*}
Furthermore, we will demonstrate  $\sum_{\ell=0}^{n}(-1)^\ell\binom{n}{\ell}\ell^n=(-1)^n\, n!$ in Lemma~\ref{lem:combinations:sum} later. With the uniform continuity of $f^{(n)}$ on a closed interval, Proposition~\ref{prop:approx:f:nth:D} follows straightforwardly. 
See more details in Section~\ref{sec:proof:prop:approx:f:nth:D}.

The proof of Proposition~\ref{prop:approx:ReLU:scrA:1k} can be divided into two main steps.
The first step involves demonstrating that
\begin{equation*}
	\frac{\varrho^{(k)}(x_0+\eps x)-\varrho^{(k)}(x_0)}{\eps}\rightrightarrows 
 \tau(x)\coloneqq \begin{cases}
 L_1x & \tn{if}\ x<0,\\
 L_2x& \tn{if}\ x\ge 0\\
 \end{cases}
 \quad \tn{for any $x\in [-A,A]$ and $A>0$,}
\end{equation*}
where $\tau$ can be used to generate \ReLU\  and
\begin{equation*}
    L_1=\lim_{t\to 0^-}\frac{\varrho^{(k)}(x_0+t)-\varrho^{(k)}(x_0)}{t}\neq L_2=\lim_{t\to 0^+}\frac{\varrho^{(k)}(x_0+t)-\varrho^{(k)}(x_0)}{t}.
\end{equation*}
The second step involves employing Proposition~\ref{prop:approx:f:nth:D} to uniformly approximate $\varrho^{(k)}$ using a $\varrho$-activated network. By combining these two steps, we can construct a $\varrho$-activated network that effectively approximates \ReLU.
For further details, refer to Section~\ref{sec:proof:prop:approx:ReLU:scrA:1k}.

The core of proving Proposition~\ref{prop:approx:ReLU:scrA:2:3} is the fact $x\cdot \one_{\{x>0\}}=\ReLU(x)$ for any $x\in\R$. 
This fact simplifies our proof considerably. Our focus then shifts toward constructing $\varrho$-activated networks that can effectively approximate $x$, $\one_{\{x>0\}}$, and $xy$ for any $x,y\in [-A,A]$ and $A>0$.
Additional details can be found in Section~\ref{sec:proof:prop:approx:ReLU:scrA:23}.

\subsection{Proof of Theorem~\ref{thm:main} Based on Propositions}
\label{sec:proof:thm:main}

The proof of Theorem~\ref{thm:main} can be easily demonstrated by using Propositions~\ref{prop:activation:replace}, \ref{prop:approx:ReLU:scrA:1k}, and \ref{prop:approx:ReLU:scrA:2:3}.
\begin{proof}[Proof of Theorem~\ref{thm:main}]
	Since $\scrA=\big(\scrA_{1,0}\cup\scrA_{1,1}\big)
     \cup \scrA_{2}\cup\scrA_3$, 
	we can divide the proof into two cases: $\varrho\in \scrA_{1,0}\cup\scrA_{1,1}$ and $\varrho\in \scrA_{2}\cup\scrA_3$. 
	
	We first consider the case $\varrho\in \scrA_{1,0}\cup\scrA_{1,1}$, i.e., $\varrho\in \scrA_{1,k}$ for $k=0,1$. By Proposition~\ref{prop:approx:ReLU:scrA:1k}, for any $M>0$, there exists
	$\tildevarrho_\eta\in \nnOneD{\varrho}{k+2}{1}{\R}{\R}\subseteq\nnOneD{\varrho}{3}{1}{\R}{\R}$ for each $\eta\in (0,1)$ such that
	\begin{equation*}
		\tildevarrho_\eta(x)\rightrightarrows \ReLU(x)\quad \tn{as}\   \eta\to 0^+\quad \tn{for any $x\in [-M,M]$.}
	\end{equation*}
	Then by Proposition~\ref{prop:activation:replace} with $\tildevarrho$ being \ReLU\ therein, for any $\eps>0$, $A>0$, and $\bmphi_\ReLU\in \nn{\ReLU}{N}{L}{d}{n}$, there exists
	$$\bmphi_{\varrho}\in \nn[\big]{\varrho}{3 N}{L}{d}{n}\subseteq \nn[\big]{\varrho}{3 N}{2L}{d}{n}$$
	such that
	\begin{equation*}
		\big\|\bmphi_\varrho-\bmphi_\ReLU\big\|_{\sup([-A,A]^d)}<\eps.
	\end{equation*}
	
		Next, we consider the case $\varrho\in \scrA_{2}\cup\scrA_3$. By Proposition~\ref{prop:approx:ReLU:scrA:2:3}, for any $M>0$, there exists
	$\tildevarrho_\eta\in \nnOneD{\varrho}{3}{2}{\R}{\R}$ for each $\eta\in (0,1)$ such that
	\begin{equation*}
		\tildevarrho_\eta(x)\rightrightarrows \ReLU(x)\quad \tn{as}\   \eta\to 0^+\quad \tn{for any $x\in [-M,M]$.}
	\end{equation*}
	Then by Proposition~\ref{prop:activation:replace} with $\tildevarrho$ being \ReLU\ therein, for any $\eps>0$, $A>0$, and $\bmphi_\ReLU\in \nn{\ReLU}{N}{L}{d}{n}$, there exists
	$$\bmphi_{\varrho}\in  \nn[\big]{\varrho}{3 N}{2L}{d}{n}$$
	such that
	\begin{equation*}
		\big\|\bmphi_\varrho-\bmphi_\ReLU\big\|_{\sup([-A,A]^d)}<\eps.
	\end{equation*}
	Thus, we finish the proof of Theorem~\ref{thm:main}.
\end{proof}

\subsection{Proofs of Theorems in Section~\ref{sec:additional:theorems} Based on Propositions}
\label{sec:proof:thms:main:others}


The proofs of Theorems~\ref{thm:main:kth:derivative}, \ref{thm:main:scrA:1k}, \ref{thm:main:scrA:2}, and \ref{thm:main:scrA:2:new} can be straightforwardly demonstrated by utilizing Propositions~\ref{prop:activation:replace}, \ref{prop:approx:f:nth:D},  \ref{prop:approx:ReLU:scrA:1k}, and \ref{prop:approx:ReLU:scrA:2:3}.
\begin{proof}[Proof of Theorem~\ref{thm:main:kth:derivative}]
 It follows from $\varrho\in C^k(\R)$ that $\varrho\in C^k\big((-M-1,M+1)\big)$ for any $M>0$.
 By Proposition~\ref{prop:approx:f:nth:D},  we have
 \begin{equation*}
 	\frac{\sum_{\ell=0}^{k}(-1)^\ell\binom{k}{\ell} \varrho(x+\ell  t) }{(-t)^k} 
 	\rightrightarrows \varrho^{(k)}(x)\quad\tn{as}\  t\to0\quad \tn{for any $x\in [-M,M]$.}
 \end{equation*}
 For each $\eta\in (0,1)$, we define
 \begin{equation*}
 	\tildevarrho_\eta(x)\coloneqq \frac{\sum_{\ell=0}^{k}(-1)^\ell\binom{k}{\ell} \varrho(x+\ell \eta) }{(-\eta)^k}\quad \tn{for any $x\in \R$.}
 \end{equation*}
 Clearly, $\tildevarrho_\eta\in \nnOneD{\varrho}{k+1}{1}{\R}{\R}$ for each $\eta\in (0,1)$  and 
	\begin{equation*}
		\tildevarrho_\eta(x)\rightrightarrows \varrho^{(k)}(x)\quad \tn{as}\   \eta\to 0^+\quad \tn{for any $x\in [-M,M]$.}
	\end{equation*}
	Then by Proposition~\ref{prop:activation:replace} with $\tildevarrho$ being $\varrho^{(k)}$ therein, for any $\eps>0$, $A>0$, and $\bmphi_{\varrho^{(k)}}\in \nn{\varrho^{(k)}}{N}{L}{d}{n}$, 
 there exists
	$\bmphi_{\varrho}\in \nn[\big]{\varrho}{(k+1) N}{L}{d}{n}$
	such that 
	\begin{equation*}
		\big\|\bmphi_{\varrho}-\bmphi_{\varrho^{(k)}}\big\|_{\sup([-A,A]^d)}<\eps.
	\end{equation*}
	So we finish the proof of Theorem~\ref{thm:main:kth:derivative}.
\end{proof}

\begin{proof}[Proof of Theorem~\ref{thm:main:scrA:1k}]
By Proposition~\ref{prop:approx:ReLU:scrA:1k}, for any $M>0$, $k\in \N$, and $\varrho\in \scrA_{1,k}$, there exists
$\tildevarrho_\eta\in \nnOneD{\varrho}{k+2}{1}{\R}{\R}$ for each $\eta\in (0,1)$ such that
\begin{equation*}
	\tildevarrho_\eta(x)\rightrightarrows \ReLU(x)\quad \tn{as}\   \eta\to 0^+\quad \tn{for any $x\in [-M,M]$.}
\end{equation*}
Then by Proposition~\ref{prop:activation:replace} with $\tildevarrho$ being \ReLU\ therein, for any $\eps>0$, $A>0$, and $\bmphi_\ReLU\in \nn{\ReLU}{N}{L}{d}{n}$, there exists
$\bmphi_{\varrho}\in \nn[\big]{\varrho}{(k+2) N}{L}{d}{n}$
such that
\begin{equation*}
	\big\|\bmphi_\varrho-\bmphi_\ReLU\big\|_{\sup([-A,A]^d)}<\eps.
\end{equation*}
	So we finish the proof of Theorem~\ref{thm:main:scrA:1k}.
\end{proof}

\begin{proof}[Proof of Theorem~\ref{thm:main:scrA:2}]
By Proposition~\ref{prop:approx:ReLU:scrA:2:3}, for any $M>0$ and $\varrho\in \scrA_{2}$, there exists
$\tildevarrho_\eta\in \nnOneD{\varrho}{2}{1}{\R}{\R}$ for each $\eta\in (0,1)$ such that
\begin{equation*}
	\tildevarrho_\eta(x)\rightrightarrows \ReLU(x)\quad \tn{as}\   \eta\to 0^+\quad \tn{for any $x\in [-M,M]$.}
\end{equation*}
Then by Proposition~\ref{prop:activation:replace} with $\tildevarrho$ being \ReLU\ therein, for any $\eps>0$, $A>0$, and $\bmphi_\ReLU\in \nn{\ReLU}{N}{L}{d}{n}$, there exists
$\bmphi_{\varrho}\in \nn[\big]{\varrho}{2 N}{L}{d}{n}$
such that
\begin{equation*}
	\big\|\bmphi_\varrho-\bmphi_\ReLU\big\|_{\sup([-A,A]^d)}<\eps.
\end{equation*}
	So we finish the proof of Theorem~\ref{thm:main:scrA:2}.
\end{proof}

\begin{proof}[Proof of Theorem~\ref{thm:main:scrA:2:new}]
By Proposition~\ref{prop:approx:ReLU:scrA:2:3}, for any $M>0$ and $\varrho\in \tildescrA_{2}$, there exists
$\tildevarrho_\eta\in \nnOneD{\varrho}{1}{1}{\R}{\R}$ for each $\eta\in (0,1)$ such that
\begin{equation*}
	\tildevarrho_\eta(x)\rightrightarrows \ReLU(x)\quad \tn{as}\   \eta\to 0^+\quad \tn{for any $x\in [-M,M]$.}
\end{equation*}
Then by Proposition~\ref{prop:activation:replace} with $\tildevarrho$ being \ReLU\ therein, for any $\eps>0$, $A>0$, and $\bmphi_\ReLU\in \nn{\ReLU}{N}{L}{d}{n}$, there exists
$\bmphi_{\varrho}\in \nn[\big]{\varrho}{ N}{L}{d}{n}$
such that
\begin{equation*}
	\big\|\bmphi_\varrho-\bmphi_\ReLU\big\|_{\sup([-A,A]^d)}<\eps.
\end{equation*}
	So we finish the proof of Theorem~\ref{thm:main:scrA:2:new}.
\end{proof}

\section{Proof of Proposition~\ref{prop:activation:replace}}
\label{sec:proof:prop:activation:replace}


We will prove Proposition~\ref{prop:activation:replace} in this section. The crucial aspect of the proof is the observation that $\tildevarrho\in C(\R)$ implies  $\tildevarrho$ is uniformly continuous on $[-M,M]$ for any $M>0$. 
Further information and specific details are provided below.


\begin{proof}[Proof of Proposition~\ref{prop:activation:replace}]
For ease of notation, we allow the activation function to be applied elementwise to a vector input.
Since $\bmphi_{\tildevarrho}\in \nn[\big]{\tildevarrho}{N}{L}{d}{n}$, 	$\bmphi_\tildevarrho$ is realized by a $\tildevarrho$-activated network with $\hatL$ hidden layers, where $L\ge \hatL\in\N^+$. We may assume $\hatL=L$ since the proof remains similar if we replace $L$ with $\hatL$ when $\hatL<L$.
 Then $\bmphi_\tildevarrho$ can be represented in a form of function compositions
	\begin{equation*}
		\bmphi_\tildevarrho(\bmx) =\calbmL_L\circ\tildevarrho\circ\calbmL_{L-1}\circ  \ \cdots \  \circ \tildevarrho\circ\calbmL_1\circ\tildevarrho\circ\calbmL_0(\bmx)\quad \tn{for any $\bmx\in\R^d$},
	\end{equation*}
	where $N_0=d$, $N_1,N_2,\cdots,N_L\in\N^+$ with $\max\{N_1,N_2,\cdots,N_L\}\le N$,  $N_{L+1}=n$, $\bmW_\ell\in \R^{N_{\ell+1}\times N_{\ell}}$ and $\bm{b}_\ell\in \R^{N_{\ell+1}}$ are the weight matrix and the bias vector in the $\ell$-th affine linear transform $\calbmL_\ell:\bmy \mapsto \bmW_\ell\cdot\bmy+\bmb_\ell$ for each $\ell\in \{0,1,\cdots,L\}$.
	
	Recall that there exists
	\begin{equation*}
		\tildevarrho_\eta\in \nnOneD[\big]{\varrho}{\tildeN}{\tildeL}{\R}{\R}\quad \tn{for each $\eta\in (0,1)$}
	\end{equation*}  
	such that
	\begin{equation*}
		\tildevarrho_\eta(t)\rightrightarrows \tildevarrho(t)\quad \tn{as}\  \eta\to 0^+\quad \tn{for any $t\in [-M,M]$,}
	\end{equation*}
where $M>0$ is a large number determined later.
	For each $\eta\in (0,1)$, we define 
	\begin{equation*}
		\bmphi_{\tildevarrho_\eta}(\bmx) \coloneqq\calbmL_L\circ{\tildevarrho_\eta}\circ\calbmL_{L-1}\circ  \ \cdots \  \circ {\tildevarrho_\eta}\circ\calbmL_1\circ\tildevarrho_\eta\circ\calbmL_0(\bmx)\quad \tn{for any $\bmx\in\R^d$}.
	\end{equation*}
	It is easy to verify that
	\begin{equation*}
		\bmphi_{\tildevarrho_\eta}\in \nn[\big]{\varrho}{\tildeN\cdot  N}{\  \tildeL\cdot  L}{d}{n}.
	\end{equation*}
	Moreover, we will prove 
	\begin{equation*}
		\bmphi_{\tildevarrho_\eta}(\bmx) 
		\rightrightarrows
		\bmphi_{\tildevarrho}(\bmx)
		\quad \tn{as}\   \eta\to 0^+
		\quad \tn{for any $\bmx\in[-A,A]^d$}.
	\end{equation*}
	
	For each $\eta\in (0,1)$ and $\ell=1,2,\cdots,L+1$, we define
	\begin{equation*}
		\bmh_\ell(\bmx)
		\coloneqq \calbmL_{\ell-1}\circ\tildevarrho\circ\calbmL_{\ell-2}\circ  \ \cdots \  \circ \tildevarrho\circ\calbmL_1\circ\tildevarrho\circ\calbmL_0(\bmx)\quad \tn{for any $\bmx\in\R^d$}
	\end{equation*}
	and 
	\begin{equation*}
		\bmh_{\ell,\eta}(\bmx)
		\coloneqq \calbmL_{\ell-1}\circ\tildevarrho_\eta\circ\calbmL_{\ell-2}\circ  \ \cdots \  \circ \tildevarrho_\eta\circ\calbmL_1\circ\tildevarrho_\eta\circ\calbmL_0(\bmx)\quad \tn{for any $\bmx\in\R^d$}.
	\end{equation*}
	Note that $\bmh_{\ell}$ and $\bmh_{\ell,\eta}$ are mappings from $\R^d$ to $\R^{N_\ell}$ for each $\eta\in (0,1)$ and $\ell=1,2,\cdots,L+1$.
	
	For  $\ell=1,2,\cdots,L+1$, we will prove by induction that
	\begin{equation}\label{eq:induction:h:ell}
		\bmh_{\ell,\eta}(\bmx)\rightrightarrows \bmh_{\ell}(\bmx) \quad \tn{as}\    \eta\to 0^+\quad \tn{for any $\bmx\in [-A,A]^d$.}
	\end{equation}

	First, we consider the case $\ell=1$. Clearly,
	\begin{equation*}
		\bmh_{1,\eta}(\bmx)=\calbmL_0(\bmx)= \bmh_{1}(\bmx)\rightrightarrows \bmh_{1}(\bmx) \quad \tn{as}\  \eta\to 0^+\quad \tn{for any $\bmx\in[-A,A]^d$.}
	\end{equation*}
	This means Equation~\eqref{eq:induction:h:ell} holds for $\ell=1$.
	
	Next, supposing Equation~\eqref{eq:induction:h:ell} holds for $\ell=i\in \{1,2,\cdots,L\}$, our goal is to prove that it also holds for $\ell=i+1$. Determine $M>0$ via
	\begin{equation*}
		M=  \sup \Big\{\|\bmh_j(\bmx)\|_{\ell^\infty}+1:
		\bmx\in [-A,A]^d,\quad j=1,2,\cdots,L+1\Big\},
	\end{equation*}
	where the continuity of $\tildevarrho$ guarantees the above supremum is finite, i.e., $M\in [1,\infty)$.
	By the induction hypothesis, we have
	\begin{equation*}
		\bmh_{i,\eta}(\bmx)\rightrightarrows \bmh_{i}(\bmx) \quad \tn{as}\    \eta\to 0^+\quad \tn{for any $\bmx\in [-A,A]^d$.}
	\end{equation*}
	Clearly, for any $\bmx\in [-A,A]^d$, we have $\|\bmh_{i}(\bmx)\|_{\ell^\infty}\le M$ and  \begin{equation*}
	    \|\bmh_{i,\eta}(\bmx)\|_{\ell^\infty} \le \|\bmh_{i}(\bmx)\|_{\ell^\infty}+1\le  M\quad \tn{ for small $\eta>0$.}
	\end{equation*}
	
	Recall that $\tildevarrho_\eta(t)\rightrightarrows \tildevarrho(t)$ as $\eta\to 0^+$ for any $t\in [-M,M]$. Then, we have 
	\begin{equation*}
		\tildevarrho_\eta\circ \bmh_{i,\eta}(\bmx)
		-\tildevarrho\circ \bmh_{i,\eta}(\bmx)
		\rightrightarrows \bmzero\quad \tn{as}\   \eta\to 0^+\quad \tn{for any $\bmx\in [-A,A]^d$.}
	\end{equation*}
	The continuity of $\tildevarrho$ implies
	the uniform continuity of $\tildevarrho$ on $[-M,M]$, from which we deduce
	\begin{equation*}
		\tildevarrho\circ \bmh_{i,\eta}(\bmx)
		-
		\tildevarrho\circ\bmh_{i}(\bmx) 
		\rightrightarrows \bmzero
		\quad \tn{as}\    \eta\to 0^+\quad \tn{for any $\bmx\in [-A,A]^d$.}
	\end{equation*}

	Therefore, for any $\bmx\in [-A,A]^d$, as $\eta\to 0^+$, we have
	\begin{equation*}
		\tildevarrho_\eta\circ \bmh_{i,\eta}(\bmx)
		-
		\tildevarrho\circ\bmh_{i}(\bmx) 
		=
		\underbrace{
			\tildevarrho_\eta\circ \bmh_{i,\eta}(\bmx)
			-\tildevarrho\circ \bmh_{i,\eta}(\bmx)
		}_{\rightrightarrows \bmzero}
		+
		\underbrace{\tildevarrho\circ \bmh_{i,\eta}(\bmx)
			-
			\tildevarrho\circ\bmh_{i}(\bmx) 
		}_{\rightrightarrows \bmzero}
		\rightrightarrows \bmzero,
	\end{equation*}
	implying
	\begin{equation*}
		\bmh_{i+1,\eta}(\bmx)=\calbmL_i\circ\tildevarrho_\eta\circ \bmh_{i,\eta}(\bmx)
		\rightrightarrows
		\calbmL_i\circ\tildevarrho\circ \bmh_{i}(\bmx)=\bmh_{i+1}(\bmx).
	\end{equation*}
	This means Equation~\eqref{eq:induction:h:ell} holds for $\ell=i+1$. So we complete the inductive step.

	By the principle of induction,  we have
	\begin{equation*}
		\bmphi_{\tildevarrho_{\eta}}(\bmx) =\bmh_{L+1,\eta}(\bmx)
		\rightrightarrows
		\bmh_{L+1}(\bmx)=
		\bmphi_{\tildevarrho}(\bmx)
		\quad \tn{as}\   \eta\to 0^+
		\quad \tn{for any $\bmx\in[-A,A]^d$}.
	\end{equation*}
	Then for any $\eps>0$, there exists a small $\eta_0>0$ such that 
	\begin{equation*}
		\big\|\bmphi_{\tildevarrho_{\eta_0}} -
		\bmphi_{\tildevarrho}\big\|_{\sup([-A,A]^d)}<\varepsilon.
	\end{equation*}
	By defining $\bmphi_\varrho\coloneqq\bmphi_{\tildevarrho_{\eta_0}}$, 
	we have
	\begin{equation*}		\bmphi_\varrho=\bmphi_{\tildevarrho_{\eta_0}}\in \nn[\big]{\varrho}{\tildeN\cdot  N}{\  \tildeL\cdot  L}{d}{n}
	\end{equation*}
	and 
	\begin{equation*}
		\big\|\bmphi_{\varrho} -
		\bmphi_{\tildevarrho}\big\|_{\sup([-A,A]^d)}=\big\|\bmphi_{\tildevarrho_{\eta_0}} -
		\bmphi_{\tildevarrho}\big\|_{\sup([-A,A]^d)}<\varepsilon.
	\end{equation*}
	So we finish the proof of Proposition~\ref{prop:activation:replace}.
\end{proof}

\section{Proof of Proposition~\ref{prop:approx:f:nth:D}}
\label{sec:proof:prop:approx:f:nth:D}

In this section, our goal is to prove Proposition~\ref{prop:approx:f:nth:D}. To facilitate the proof, we first introduce a lemma in Section~\ref{sec:lemma:proof:prop:approx:f:nth:D} that simplifies the process. Subsequently, we provide the detailed proof in Section~\ref{sec:proof:prop:approx:f:nth:D:with:lemma}.

\subsection{A Lemma for Proving Proposition~\ref{prop:approx:f:nth:D}}
\label{sec:lemma:proof:prop:approx:f:nth:D}

\begin{lemma}
	\label{lem:combinations:sum}
	Given any $n\in\N$, it holds that
	\begin{equation*}
		\sum_{\ell =0}^{n} (-1)^\ell \binom{n}{\ell }\ell^i  =
		\begin{cases}
			0 & \tn{if}\ i\in \{0,1,\cdots,n-1\},\\
			(-1)^n\,n! & \tn{if}\ i=n.
		\end{cases}
	\end{equation*}
\end{lemma}
\begin{proof}
	To simplify the proof, we claim that there exists a polynomial $p_i$  for each $i\in\{0,1,\cdots,n\}$ such that
	\begin{equation*}
		\sum_{\ell =0}^{n} t^\ell \binom{n}{\ell }\ell^i  =
		(1+t)^{n-i}\bigg(\frac{n!}{(n-i)!}t^i+(1+t)p_i(t)\bigg)\quad\tn{for any $t\in (-1,0)$.}
	\end{equation*}
 By assuming the validity of the claim, we have
	\begin{equation*}
		\begin{split}
			\sum_{\ell =0}^{n} (-1)^\ell \binom{n}{\ell }\ell^i  =\lim_{t\to -1^+} \sum_{\ell =0}^{n} t^\ell \binom{n}{\ell }\ell^i
			&=\lim_{t\to -1^+}
			(1+t)^{n-i}\bigg(\frac{n!}{(n-i)!}t^i+(1+t)p_i(t)\bigg)\\
			&=
			\begin{cases}
				0 & \tn{if}\ i\in \{0,1,\cdots,n-1\},\\
				(-1)^n\,n! & \tn{if}\ i=n.
			\end{cases}
		\end{split}
	\end{equation*}
	It remains to prove the claim and we will establish its validity by induction.
	
	First, we consider the case $i=0$. Clearly,
	\begin{equation*}
		\sum_{\ell =0}^{n} t^\ell \binom{n}{\ell }\ell^0  =
		\sum_{\ell =0}^{n} t^\ell \binom{n}{\ell }=(1+t)^n=
		(1+t)^{n-0}\bigg(\frac{n!}{(n-0)!}t^0+(1+t)p_0(t)\bigg)
	\end{equation*}
	for any $t\in (-1,0)$, where $p_0(t)=0$. That means the claim holds for $i=0$.
	
	Next, assuming the claim holds for $i=j\in \{0,1,\cdots,n-1\}$, we will show it also holds for $i=j+1$. By the induction hypothesis,
	 we have
	 \begin{equation*}
	 	\sum_{\ell =0}^{n} t^\ell \binom{n}{\ell }\ell^j  =
	 	(1+t)^{n-j}\bigg(\underbrace{\frac{n!}{(n-j)!}t^j+(1+t)p_j(t)}_{\tildep_j(t)}\bigg)=(1+t)^{n-j}\tildep_j(t)
	 \end{equation*} 
	 for any $t\in (-1,0)$, where $\tildep_j(t)=\frac{n!}{(n-j)!}t^j+(1+t)p_j(t)$ is a polynomial. 
	 By differentiating both sides of the equation above, we obtain
	 \begin{equation*}
	 	\begin{split}
	 		\sum_{\ell =0}^{n} \ell t^{\ell -1}\binom{n}{\ell }\ell^j  
	 		&=(n-j)(1+t)^{n-j-1}\tildep_j(t)+(1+t)^{n-j}\tfrac{d}{dt}\tildep_j(t)\\
	 		&= (1+t)^{n-j-1}\Big((n-j)\tildep_j(t)+(1+t)\tfrac{d}{dt}\tildep_j(t)\Big)\\
	 	\end{split}
	 \end{equation*}
	 for any $t\in (-1,0)$, implying
	\begin{equation*}
		\begin{split}
			\sum_{\ell =0}^{n} t^{\ell }\binom{n}{\ell }\ell^{j+1 } 
			&=
			t\sum_{\ell =0}^{n} \ell t^{\ell -1}\binom{n}{\ell }\ell^j =
			 t(1+t)^{n-j-1}\Big((n-j)\tildep_j(t)+(1+t)\tfrac{d}{dt}\tildep_j(t)\Big)\\
			 &= 
			 (1+t)^{n-j-1}\Big(t(n-j)\tildep_j(t)+t(1+t)\tfrac{d}{dt}\tildep_j(t)\Big)\\
			 &= 
			 (1+t)^{n-(j+1)}\bigg(t(n-j)\Big(\underbrace{\tfrac{n!}{(n-j)!}t^j+(1+t)p_j(t)}_{\tildep_j(t)}\Big)+t(1+t)\tfrac{d}{dt}\tildep_j(t)\bigg)\\
			 &= 
			 (1+t)^{n-(j+1)}\bigg(\tfrac{n!(n-j)}{(n-j)!}t^{j+1}+t(n-j)(1+t)p_j(t)+t(1+t)\tfrac{d}{dt}\tildep_j(t)\bigg)		\\
			 &= 
			 (1+t)^{n-(j+1)}\bigg(\tfrac{n!}{(n-(j+1))!}t^{j+1}+(1+t)\Big(\underbrace{t(n-j)p_j(t)+t\tfrac{d}{dt}\tildep_j(t)}_{p_{j+1}(t)}\Big)\bigg)	\\
			 &= 
			 (1+t)^{n-(j+1)}\bigg(\tfrac{n!}{(n-(j+1))!}t^{j+1}+(1+t)p_{j+1}(t)\bigg),		 
		\end{split}
	\end{equation*}
	for any $t\in (-1,0)$, where $p_{j+1}(t)=t(n-j)p_j(t)+t\tfrac{d}{dt}\tildep_j(t)$ is a polynomial.	
	With the completion of the induction step, we have successfully demonstrated the validity of the claim.
	Thus, we complete the proof of Lemma~\ref{lem:combinations:sum}.
\end{proof}

\subsection{Proof of Proposition~\ref{prop:approx:f:nth:D} Based on Lemma~\ref{lem:combinations:sum}}
\label{sec:proof:prop:approx:f:nth:D:with:lemma}

Equipped with Lemma~\ref{lem:combinations:sum}, we are prepared to demonstrate the proof of Proposition~\ref{prop:approx:f:nth:D}.
\begin{proof}[Proof of Proposition~\ref{prop:approx:f:nth:D}]
We may assume $n\in \N^+$ since the case $n=0$ is trivial.
	For each $x\in [a,b]$, we define
	\begin{equation*}
		g_x(t)\coloneqq \sum_{\ell =0}^{n}(-1)^\ell \binom{n}{\ell } f(x+\ell t) \quad \tn{for any $t\in (-c_0,c_0)$},
	\end{equation*}
	where $c_0>0$ is a small number ensuring that $x+\ell t\in (a_0,b_0)$ for $\ell =0,1,\cdots,n$. For example, we can set
	\begin{equation*}
		c_0=\min\Big\{\frac{a-a_0}{n+1},\, \frac{b_0-b}{n+1}
		\Big\}.
	\end{equation*}

It follows from $f\in C^n\big((a_0,b_0)\big)$ that $f^{(n)}$ is continuous on $(a_0,b_0)$, implying
 $f^{(n)}$ is 
 uniformly continuous on $[a-nc_0,b+nc_0]\subseteq (a_0,b_0)$.
	For any $\eps>0$, there exists $\delta_0\in (0,c_0)$ such that
	\begin{equation}
		\label{eq:f:nth:derivative:eps}
		\big|f^{(n)}(x_1)-f^{(n)}(x_2)\big|<\tfrac{\eps}{C_n}\quad \tn{if $|x_1-x_2|<n\delta_0$\quad for any $x_1,x_2\in [a-nc_0,b+nc_0]$,}
	\end{equation}
	where 
	$C_n=\sum_{j=0}^n j^n \binom{n}{j}$.
	
%
		
	For each $x\in [a,b]$, we have
	\begin{equation*}
		g_x^{(i)}(t)=\sum_{\ell =0}^{n} (-1)^\ell \binom{n}{\ell }\ell^i  f^{(i)}(x+\ell t)\quad \tn{for any $t\in (-c_0,c_0)$ and $i=0,1,\cdots,n$,}
	\end{equation*}
	implying
	\begin{equation*}
		g_x^{(i)}(0)=\sum_{\ell =0}^{n} (-1)^\ell \binom{n}{\ell }\ell^i  f^{(i)}(x)
		=0\quad \tn{for $i=0,1,\cdots,n-1$,}
	\end{equation*}		
	where the last equality comes from Lemma~\ref{lem:combinations:sum}.

	Then for any $t\in (-\delta_0,0)\cup(0,\delta_0)$ and each $x\in [a,b]$,
	by Cauchy's mean value theorem, there exist 
	$0<|t_{x,n}|<\cdots<|t_{x,1}|<|t|<\delta_0$ such that
	\begin{equation*}
		\begin{split}
			\frac{g_x(t)}{t^n}
			&=\frac{g_x^{(0)}(t)-g_x^{(0)}(0)}{ t^n-0}
			=\frac{g_x^{(1)}(t_{x,1})}{nt_{x,1}^{n-1}}
			=\frac{g_x^{(1)}(t_{x,1})-g_x^{(1)}(0)}{nt_{x,1}^{n-1}-0}\\
			&=\frac{g_x^{(2)}(t_{x,2})}{ n(n-1)t_{x,2}^{n-2}}=\frac{g_x^{(2)}(t_{x,2})-g_x^{(2)}(0)}{ n(n-1)t_{x,2}^{n-2}-0}
			=\frac{g_x^{(3)}(t_{x,3})}{ n(n-1)(n-2)t_{x,3}^{n-3}}=\cdots 
			=  \frac{g_x^{(n)}(t_{x,n})}{ n!}.
		\end{split}
	\end{equation*}

	Moreover, for any $t\in (-\delta_0,0)\cup(0,\delta_0)$ and each $x\in [a,b]\subseteq[a-nc_0,b+nc_0]$, we have
	\begin{equation*}
		|(x+\ell t_{x,n})-x|=|\ell t_{x,n}|\le |n t_{x,n}|< n\delta_0<nc_0\quad \tn{and}\quad   x+\ell t_{x,n}\in [a-nc_0,b+nc_0], 
	\end{equation*}
	for $\ell =0,1,\cdots,n$, from which we deduce 
	\begin{equation*}
		\big| f^{(n)}(x+\ell t_{x,n})- f^{(n)}(x)\big|<\frac{\eps}{C_n}=\frac{\eps}{ \sum_{j=0}^n j^n \binom{n}{j}},
	\end{equation*}
	where the strict inequality comes from Equation~\eqref{eq:f:nth:derivative:eps}.

		Set $\lambda_\ell =\frac{ (-1)^\ell \binom{n}{\ell }\ell^n}{(-1)^n\,n!}$
	for $\ell =0,1,\cdots,n$. By Lemma~\ref{lem:combinations:sum}, we have
	\begin{equation*}
		\sum_{\ell =0}^n \lambda_\ell =\sum_{\ell =0}^n \frac{ (-1)^\ell \binom{n}{\ell }\ell^n}{(-1)^n\,n!}
		= \frac{ \sum_{\ell =0}^n(-1)^\ell \binom{n}{\ell }\ell^n}{(-1)^n\,n!}
		=\frac{{(-1)^n\,n!}}{{(-1)^n\,n!}}=1.
	\end{equation*}
	
	Therefore, for any $t\in (-\delta_0,0)\cup(0,\delta_0)$ and each $x\in [a,b]$, we have
	\begin{equation*}
		\begin{split}
			&\bigg|\frac{\sum_{\ell =0}^{n}(-1)^\ell \binom{n}{\ell } f(x+\ell t)}{(-t)^n}-f^{(n)}(x)\bigg|
			=\bigg|\frac{g_x(t)}{(-1)^nt^n}-f^{(n)}(x)\bigg|
			=\bigg|\frac{g_x^{(n)}(t_{x,n})}{(-1)^n\,n!}-f^{(n)}(x)\bigg|\\
			=\;& 
			\bigg|\frac{\sum_{\ell =0}^{n} (-1)^\ell \binom{n}{\ell }\ell^nf^{(n)}(x+\ell t_{x,n})}{(-1)^n\,n!}-f^{(n)}(x)\bigg|
			=\bigg|\sum_{\ell =0}^{n}\lambda_\ell  f^{(n)}(x+\ell t_{x,n})-f^{(n)}(x)\bigg|\\
			=\;& \bigg|\sum_{\ell =0}^{n}\lambda_\ell  f^{(n)}(x+\ell t_{x,n})-\sum_{\ell =0}^{n}\lambda_\ell  f^{(n)}(x)\bigg|
           \le \sum_{\ell =0}^{n}|\lambda_\ell |	\cdot \big| f^{(n)}(x+\ell t_{x,n})- f^{(n)}(x)\big|	\\
			<\;& \sum_{\ell =0}^{n}|\lambda_\ell |\cdot\frac{\eps}{C_n}
			= \sum_{\ell =0}^n \frac{\ell^n \binom{n}{\ell }}{n!}\cdot\frac{\eps}{ \sum_{j=0}^n j^n \binom{n}{j}}\le \sum_{\ell =0}^n \ell^n \binom{n}{\ell }\cdot\frac{\eps}{ \sum_{j=0}^n j^n \binom{n}{j}}=\eps.
		\end{split}
	\end{equation*}
	Since $\eps>0$ and $x\in [a,b]$ are arbitrary, we can conclude that
		\begin{equation*}
		\frac{\sum_{\ell =0}^{n}(-1)^\ell \binom{n}{\ell } f(x+\ell t) }{(-t)^n} 
		\rightrightarrows f^{(n)}(x)\quad\tn{as}\  t\to0\quad \tn{for any $x\in [a,b]$.}
	\end{equation*}
	So we finish the proof of Proposition~\ref{prop:approx:f:nth:D}.
\end{proof}

\section{Proof of Proposition~\ref{prop:approx:ReLU:scrA:1k}}
\label{sec:proof:prop:approx:ReLU:scrA:1k}

The objective of this section is to provide the proof of Proposition~\ref{prop:approx:ReLU:scrA:1k}. To streamline the proof process, we first introduce a lemma in Section~\ref{sec:lemma:proof:prop:approx:ReLU:scrA:1k}. Subsequently, we present the comprehensive proof in Section~\ref{sec:proof:prop:approx:ReLU:scrA:1k:with:lemma}.

\subsection{A Lemma for Proving Proposition~\ref{prop:approx:ReLU:scrA:1k}}
\label{sec:lemma:proof:prop:approx:ReLU:scrA:1k}

\begin{lemma}
    \label{lem:approx:x}
    Suppose $f:\R\to\R$ is a function with $f^\prime(x_0)\neq 0$ for some $x_0\in\R$.
 Then for any $M>0$, it holds that
 \begin{equation*}
     \frac{f(x_0+\eps x)-f(x_0)}{\eps  f^\prime(x_0)}\rightrightarrows
     x\quad \tn{as}\   \eps\to 0^+\quad \tn{for any $x\in [-M,M]$.}
 \end{equation*}
\end{lemma}
\begin{proof}
    By Taylor's theorem with Peano's form of remainder, there exists $h:\R \to \R$ such that $\lim_{\eta\to 0}h(x_0+\eta)=0$ and 
\begin{equation*}
\begin{split}
f(z)=f(x_0)+f^\prime(x_0)(z-x_0)+h(z)(z-x_0)\quad \tn{for any $z\in\R$.}
\end{split}
\end{equation*}
By substituting $z$ with $x_0+\eps x$ in the above equation, for any $x\in[-M,M]$ and $\eps>0$, we obtain
\begin{equation*}
	\begin{split}
	    \frac{f(x_0+\eps x)-f(x_0)}{\eps 
 f^\prime(x_0)}=
 \frac{f(x_0)+f^\prime(x_0)(\eps  x)+h(x_0+\eps  x)(\eps  x)-f(x_0)}{\eps 
 f^\prime(x_0)}=x+\frac{h(x_0+\eps  x)  x}{
 f^\prime(x_0)}.
	\end{split}
\end{equation*}
It follows from $\lim_{\eta\to 0}h(x_0+\eta)=0$ that
\begin{equation*}
	\frac{h(x_0+\eps  x)  x}{
 f^\prime(x_0)}\rightrightarrows 0\quad \tn{as}\   \eps\to 0^+\quad \tn{for any $x\in [-M,M]$,}
\end{equation*}
from which we deduce
\begin{equation*}
	\frac{f(x_0+\eps x)-f(x_0)}{\eps 
 f^\prime(x_0)}\rightrightarrows  x\quad \tn{as}\   \eps\to 0^+\quad \tn{for any $x\in [-M,M]$.}
\end{equation*}
So we finish the proof of 
Lemma~\ref{lem:approx:x}.
\end{proof}

\subsection{Proof of Proposition~\ref{prop:approx:ReLU:scrA:1k} Based on Lemma~\ref{lem:approx:x}}
\label{sec:proof:prop:approx:ReLU:scrA:1k:with:lemma}

With Lemma~\ref{lem:approx:x} in hand, we are ready to present the proof of Proposition~\ref{prop:approx:ReLU:scrA:1k}.
\begin{proof}[Proof of Proposition~\ref{prop:approx:ReLU:scrA:1k}]
	Given any $\eps\in (0,1)$, our goal is to construct $\phi_\eps\in \nnOneD{\varrho}{(k+2)}{1}{\R}{\R}$ with $\varrho\in \scrA_{1,k}$ to approximate \ReLU\  well on $[-M,M]$.
	
	Clearly, there exist  $a_0<b_0$ and $x_0\in (a_0,b_0)$ such that $\varrho\in C^k\big((a_0,b_0)\big)$ and 
	\begin{equation*}
		L_1=\lim_{t\to 0^-}\tfrac{\varrho^{(k)}(x_0+t)-\varrho^{(k)}(x_0)}{t}\neq L_2=\lim_{t\to 0^+}\tfrac{\varrho^{(k)}(x_0+t)-\varrho^{(k)}(x_0)}{t}.
	\end{equation*}
Set
\begin{equation*}
	c_0=\min\big\{\tfrac{b_0-x_0}{2},\, \tfrac{x_0-a_0}{2}\big\}\quad 
	\tn{and}\quad 
 	K=\max\Big\{1,\; \big|\tfrac{1}{L_2-L_1}\big|,\; \big|\tfrac{L_1}{L_2-L_1}\big|\Big\}.
\end{equation*}
There exists a small $\delta_\eps\in (0,c_0)$ such that
\begin{equation*}
	\Big|\tfrac{\varrho^{(k)}(x_0+t  )-\varrho^{(k)}(x_0)}{t}-\big(L_1\cdot\one_{\{t<0\}}+L_2\cdot\one_{\{t>0\}}\big)\Big|
	<{\eps}/{(4KM)}
\end{equation*}
for any $t\in (-\delta_\eps,0)\cup(0,\delta_\eps)$.
Define $r_\eps\coloneqq
    \min\big\{\eps,\tfrac{\delta_\eps}{3M}\big\}$ and 
\begin{equation*}
    \psi_\eps(x)\coloneqq
    \frac{\varrho^{(k)}(x_0+r_\eps x)-\varrho^{(k)}(x_0)}{r_\eps}
    \quad \tn{for any $x\in\R$.}
\end{equation*}
Clearly, $\psi_\eps(0)=0$. Moreover, since
\(0<r_\eps\le \frac{\delta_\eps}{3M}<\frac{\delta_\eps}{2M}\),
we have
$r_\eps x\in(-\delta_\eps,0)\cup(0,\delta_\eps)$ for any
$x\in[-2M,0)\cup(0,2M]$, implying
\begin{equation*}
	\begin{split}
		&\Big|\psi_\eps(x)-\big(L_1\cdot\one_{\{x<0\}}+L_2\cdot\one_{\{x>0\}}\big)x\Big|
		\le |x|\cdot\Big|\psi_\eps(x)/x-\big(L_1\cdot\one_{\{x<0\}}+L_2\cdot\one_{\{x>0\}}\big)\Big|\\
		=\; &|x|\cdot\Big|\tfrac{\varrho^{(k)}(x_0+r_\eps x)-\varrho^{(k)}(x_0)}{r_\eps  x}-\big(L_1\cdot\one_{\{r_\eps  x<0\}}+L_2\cdot\one_{\{r_\eps  x>0\}}\big)\Big|< 2M\cdot \tfrac{\eps}{4K M}=\eps/(2K).
	\end{split}
\end{equation*}
Thus, for each $\eps\in  (0,1)$, we have
\begin{equation*}
	\Big|\psi_\eps(x)-\big(L_1\cdot\one_{\{x<0\}}+L_2\cdot\one_{\{x>0\}}\big)x\Big|<\eps/(2K)\quad  \tn{for any $x\in [-2M,2M]$,}
\end{equation*}
implying
\begin{equation}\label{eq:psieps:minus:psi}
	\big|\psi_\eps(x)-\psi(x)\big|<\eps/(2K) \quad \tn{for any $x\in [-2M,2M]$,}
\end{equation}
where
\begin{equation*}
	\tn{$\psi(x)\coloneqq\big(L_1\cdot\one_{\{x<0\}}+L_2\cdot\one_{\{x>0\}}\big)x$\quad for any $x\in\R$.}
\end{equation*}
Moreover, for any $x\in\R$, we have
\begin{equation*}
	\begin{split}
		\psi(x)-L_1 x
		&=\big(L_1\cdot\one_{\{x<0\}}+L_2\cdot\one_{\{x>0\}}\big)x-
		L_1 x\big(\one_{\{x<0\}}+\one_{\{x>0\}}\big)\\
		&=
		(L_2-L_1)\cdot\one_{\{x>0\}}\cdot  x=(L_2-L_1)\cdot\ReLU(x),
	\end{split}
\end{equation*}
from which we deduce
\begin{equation*}
    \tfrac{1}{L_2-L_1}\psi(x)-\tfrac{L_1}{L_2-L_1}x=\ReLU(x).
\end{equation*}
To construct a $\varrho$-activated network to  approximate \ReLU\ well, we only need to construct $\varrho$-activated networks to effectively approximate $\psi(x)$ and $x$ for any $x\in [-M,M]$. 
When $k\ge 1$ and $\varrho^\prime(x_1)\neq  0$, we have the option to employ $\tfrac{\varrho(x_1+\eta x)-\varrho(x_1)}{\eta \varrho^\prime(x_1)}$ for a sufficiently accurate approximation of $x$ when $\eta$ is small. However, in the scenario where $k=0$, this approach is not applicable. As a result, we will split the remainder of the proof into two cases: one where $k=0$ and the other where $k\ge 1$.

\mycase{1}{$k=0$.}
First, let us consider the case of $k=0$. 
In this case,  $\varrho^{(k)}=\varrho$.
 For each $\eps\in  (0,1)$ and any $x\in [-M,M]$, we have $x-M\in [-2M,0]\subseteq [-2M,2M]$, 
 and by combining this with Equation~\eqref{eq:psieps:minus:psi}, we deduce
\begin{equation}\label{eq:approx:L1x}
\begin{split}
    	\eps/(2K) & > 
     \Big|\psi_\eps(x-M)-\psi(x-M)\Big|\\
     & =\Big|\psi_\eps(x-M)-\big(L_1\cdot\one_{\{x-M<0\}}+L_2\cdot\one_{\{x-M>0\}}\big)(x-M)\Big|\\
     & =\Big|\psi_\eps(x-M)-L_1(x-M)\Big|=
     \Big|\psi_\eps(x-M)+L_1 M -L_1 x\Big|.
\end{split}
\end{equation}

Define
\begin{equation*}
	\begin{split}
	    \phi_\eps(x)\coloneqq \;&\tfrac{1}{L_2-L_1}\psi_\eps(x)-\tfrac{1}{L_2-L_1}\Big(
 \psi_\eps(x-M)+L_1 M\Big)\\
 =\; & \tfrac{1}{L_2-L_1}\tfrac{\varrho(x_0+r_\eps x)-\varrho(x_0)}{r_\eps}-\tfrac{1}{L_2-L_1}\Big(\tfrac{\varrho(x_0+r_\eps (x-M))-\varrho(x_0)}{r_\eps}+L_1 M\Big)
	\end{split}
\end{equation*}
for any $x\in\R$.
It is easy to verify that $\phi_\eps\in \nnOneD{\varrho}{2}{1}{\R}{\R}=\nnOneD{\varrho}{k+2}{1}{\R}{\R}$. 
Moreover, for each $\eps\in  (0,1)$ and any $x\in [-M,M]$, we have
\begin{equation*}
	\begin{split}
		|\phi_\eps(x)-\ReLU(x)|
		&=  \bigg|\underbrace{\tfrac{1}{L_2-L_1}\psi_\eps(x)-\tfrac{1}{L_2-L_1}\Big(
 \psi_\eps(x-M)+L_1 M\Big)}_{\phi_\eps}-\Big(\underbrace{\tfrac{1}{L_2-L_1}\psi(x)-\tfrac{L_1}{L_2-L_1}x}_{\ReLU}\Big)\bigg|\\
		& \le  \big|\tfrac{1}{L_2-L_1}\big|\cdot\big|\psi_\eps(x)-\psi(x)\big|+\big|\tfrac{1}{L_2-L_1}\big|\cdot\Big|\Big(
 \psi_\eps(x-M)+L_1 M\Big)-L_1x\Big|\\
		&<   K\cdot \tfrac{\eps}{2K}+K\cdot \tfrac{\eps}{2K}=\eps,
	\end{split}
\end{equation*}
where the strict inequality comes from Equations~\eqref{eq:psieps:minus:psi} and \eqref{eq:approx:L1x}.
Therefore, we can conclude that
	\begin{equation*}
	\phi_\eps(x)\rightrightarrows \ReLU(x)\quad \tn{as}\   \eps\to 0^+\quad \tn{for any $x\in [-M,M]$.}
\end{equation*}
That means we finish the proof for the case of $k=0$.

\mycase{2}{$k\ge 1$.}
Next, let us consider the case of $k\ge 1$. 
Define
\begin{equation*}
	\tildephi_{\eps}(x)\coloneqq \tfrac{1}{L_2-L_1}\psi_\eps(x)-\tfrac{L_1}{L_2-L_1}x\quad \tn{for any $x\in\R$.}
\end{equation*}
Then by Equation~\eqref{eq:psieps:minus:psi}, for each $\eps\in  (0,1)$ and any $x\in [-M,M]\subseteq [-2M,2M]$, we have
\begin{equation}\label{eq:tildephieps:minus:ReLU}
	\begin{split}
		&\big|\tildephi_\eps(x)-\ReLU(x)\big|
		=  \Big|\Big(\tfrac{1}{L_2-L_1}\psi_\eps(x)-\tfrac{L_1}{L_2-L_1}x\Big)-\Big(\tfrac{1}{L_2-L_1}\psi(x)-\tfrac{L_1}{L_2-L_1}x\Big)\Big|\\
		= \;& \Big|\tfrac{1}{L_2-L_1}\psi_\eps(x)-\tfrac{1}{L_2-L_1}\psi(x)\Big|
		 \le  \big|\tfrac{1}{L_2-L_1}\big|\cdot\big|\psi_\eps(x)-\psi(x)\big|<   K\cdot \tfrac{\eps}{2K}=\eps/2.
	\end{split}
\end{equation}
Our goal is to use a $\varrho$-activated network to effectively approximate
\begin{equation*}
	\tildephi_{\eps}(x)= \tfrac{1}{L_2-L_1}\psi_\eps(x)-\tfrac{L_1}{L_2-L_1}x
	=\tfrac{1}{L_2-L_1}\tfrac{\varrho^{(k)}(x_0+r_\eps x)-\varrho^{(k)}(x_0)}{r_\eps}-\tfrac{L_1}{L_2-L_1}x
\end{equation*}
for any $x\in [-M,M]$ and $\eps\in (0,1)$. To this end,
we need to construct $\varrho$-activated networks to  effectively approximate 
$\varrho^{(k)}(x_0+r_\eps x)$ and $x$ for any $x\in [-M,M]$ and $\eps\in (0,1)$.

Recall that  $\varrho\in C^k\big((a_0,b_0)\big)\backslash C^{k+1}\big((a_0,b_0)\big)$ with $k\ge 1$.
Then there exists $x_1\in (a_0,b_0)$ such that $\varrho^\prime(x_1)\neq 0$.
For each $\eta\in(0,1)$, we define
\begin{equation*}
	g_{\eta}(x)\coloneqq\frac{\varrho(x_1+\eta x)-\varrho(x_1)}{\eta\varrho^\prime(x_1)}\quad \tn{for any $x\in \R$.}
\end{equation*}
By Lemma~\ref{lem:approx:x}, 
\begin{equation*}
	g_{\eta}(x)=\frac{\varrho(x_1+\eta x)-\varrho(x_1)}{\eta\varrho^\prime(x_1)}\rightrightarrows  x\quad \tn{as}\   \eta\to 0^+\quad \tn{for any $x\in [-M,M]$.}
\end{equation*}

For each $\eta\in (0,1)$, we define
\begin{equation*}
	h_\eta(z)\coloneqq\frac{\sum_{i=0}^{k}(-1)^i\binom{k}{i} \varrho(z+i\eta) }{(-\eta)^k}\quad \tn{for any $z\in\R$.}
\end{equation*}
Recall that 
$c_0=\min\big\{\tfrac{b_0-x_0}{2},\, \tfrac{x_0-a_0}{2}\big\}$ and $\varrho\in C^k\big((a_0,b_0)\big)$. By Proposition~\ref{prop:approx:f:nth:D},
\begin{equation*}
	h_\eta(z)=\frac{\sum_{i=0}^{k}(-1)^i\binom{k}{i} \varrho(z+i\eta) }{(-\eta)^k} 
	\rightrightarrows \varrho^{(k)}(z)\quad\tn{as}\  \eta\to0\quad \tn{for any $z\in [x_0-c_0,x_0+c_0]$.}
\end{equation*}

Then there exists $\eta_\eps\in (0,1)$ such that
\begin{equation*}
	\big|g_{\eta_\eps}(x)-x\big|<r_\eps/(4K)\quad \tn{for any $x\in [-M,M]$}
\end{equation*}
and 
\begin{equation*}
	\big|	h_{\eta_\eps}(z) - \varrho^{(k)}(z)\big|<r_\eps^2/(4K)\quad \tn{for any $z\in [x_0-c_0,x_0+c_0]$.}
\end{equation*}				
	
Next, we can define the desired $\phi_\eps$ via
\begin{equation*}
	\begin{split}
		\phi_{\eps}(x)\coloneqq\;&
		\frac{1}{L_2-L_1}\frac{h_{\eta_\eps}(x_0+r_\eps x)-\varrho^{(k)}(x_0)}{r_\eps}-\frac{L_1}{L_2-L_1}g_{\eta_\eps}(x)\\
		=\;&\frac{\sum_{i=0}^{k}(-1)^i\binom{k}{i} \varrho(x_0+r_\eps x+i\eta_\eps) -(-\eta_\eps)^k\varrho^{(k)}(x_0)}{(-\eta_\eps)^k(L_2-L_1)r_\eps}-\frac{L_1\varrho(x_1+\eta_\eps x)-L_1\varrho(x_1)}{(L_2-L_1)\eta_\eps\varrho^\prime(x_1)}
	\end{split}
\end{equation*}	
for any $x\in \R$. It is easy to verify that $\phi_\eps\in \nnOneD{\varrho}{k+2}{1}{\R}{\R}$. Moreover, since $0< r_\eps\le   \tfrac{\delta_\eps}{3M} \le \tfrac{c_0}{3M}$, we have $x_0+r_\eps  x \in [x_0-c_0,x_0+c_0]$  for any $x\in [-M,M]$, implying
\begin{equation*}
	\begin{split}
		&\big|\phi_{\eps}(x)-\tildephi_{\eps}(x)\big|\\
		=\;& \Big|\Big(	\tfrac{1}{L_2-L_1}\tfrac{h_{\eta_\eps}(x_0+r_\eps x)-\varrho^{(k)}(x_0)}{r_\eps}-\tfrac{L_1}{L_2-L_1}g_{\eta_\eps}(x)\Big)-\Big(\tfrac{1}{L_2-L_1}\tfrac{\varrho^{(k)}(x_0+r_\eps x)-\varrho^{(k)}(x_0)}{r_\eps}-\tfrac{L_1}{L_2-L_1}x\Big)\Big|\\
		\le \;& \big|\tfrac{1}{L_2-L_1}\big|\cdot\Big|	\tfrac{h_{\eta_\eps}(x_0+r_\eps x)-\varrho^{(k)}(x_0)}{r_\eps}-\tfrac{\varrho^{(k)}(x_0+r_\eps x)-\varrho^{(k)}(x_0)}{r_\eps}\Big|
		+\big|\tfrac{L_1}{L_2-L_1}\big|\cdot\big|g_{\eta_\eps}(x)-x\big|\\
		\le \;& \tfrac{1}{r_\eps}\big|\tfrac{1}{L_2-L_1}\big|\cdot\Big|	h_{\eta_\eps}(x_0+r_\eps x)-\varrho^{(k)}(x_0+r_\eps x)\Big|
		+K\cdot \tfrac{r_\eps}{4K}
		\\  \le    \;&   \tfrac{1}{r_\eps}K\cdot\tfrac{r_\eps^2}{4K}
		+K\cdot \tfrac{r_\eps}{4K}=r_\eps/2\le \eps/2.
	\end{split}
\end{equation*}
Combining this with Equation~\eqref{eq:tildephieps:minus:ReLU}, we can conclude that
\begin{equation*}
	\big|\phi_\eps(x)-\ReLU(x)\big|\le \big|\phi_\eps(x)-\tildephi_\eps(x)\big|+\big|\tildephi_\eps(x)-\ReLU(x)\big|<\eps/2+\eps/2=\eps,
\end{equation*}
for each $\eps\in  (0,1)$ and any $x\in [-M,M]$. That means
	\begin{equation*}
	\phi_\eps(x)\rightrightarrows \ReLU(x)\quad \tn{as}\   \eps\to 0^+\quad \tn{for any $x\in [-M,M]$.}
\end{equation*}
	So we finish the proof of Proposition~\ref{prop:approx:ReLU:scrA:1k}.
\end{proof}

\section{Proof of Proposition~\ref{prop:approx:ReLU:scrA:2:3}}
\label{sec:proof:prop:approx:ReLU:scrA:23}

We will prove Proposition~\ref{prop:approx:ReLU:scrA:2:3} in this section. To this end, we first establish two lemmas in Section~\ref{sec:lemmas:proof:prop:approx:ReLU:scrA:23}, which play important roles in proving Proposition~\ref{prop:approx:ReLU:scrA:2:3}. Next,
we give the detailed proof of Proposition~\ref{prop:approx:ReLU:scrA:2:3}  in Section~\ref{sec:proof:thm:approx:ReLU:with:lemmas} based on these two lemmas.

\subsection{Lemmas for Proving Proposition~\ref{prop:approx:ReLU:scrA:2:3}}
\label{sec:lemmas:proof:prop:approx:ReLU:scrA:23}

\begin{lemma}\label{lem:xy}
	Given any $A>0$ and a function $\varrho:\R\to\R$, suppose there exists $x_0\in\R$ satisfying $\varrho^\dprime(x_0)\neq 0$. Then there exists
	\begin{equation*}
		\phi_\eps\in \nnOneD{\varrho}{3}{1}{\R^2}{\R}\quad \tn{for each $\eps\in (0,1)$}
	\end{equation*}  
	such that
	\begin{equation*}
		\phi_\eps(x,y)\rightrightarrows  xy\quad \tn{as $\eps\to0^+$}\quad \tn{for any $x,y\in[-A,A]$.}
	\end{equation*}
\end{lemma}
\begin{proof}
By Taylor's theorem with Peano's form of remainder, there exists $h:\R \to \R$ such that $\lim_{\eta\to 0}h(x_0+\eta)=0$ and 
\begin{equation}
\label{eq:2nd:taylor:expansion}
\begin{split}
\varrho(z)=\varrho(x_0)+\varrho^\prime(x_0)(z-x_0)+\tfrac{\varrho^\dprime(x_0)}{2}(z-x_0)^2+ h(z)(z-x_0)^2\quad \tn{for any $z\in\R$.}
\end{split}
\end{equation}
 
For each $\eps\in (0,1)$, we define
\begin{equation*}
\phi_\eps(x,y)\coloneqq\frac{\varrho(x_0+\eps  x+\eps y)-\varrho(x_0+\eps y)-\varrho(x_0+\eps x)+\varrho(x_0)}{\eps^2\varrho^\dprime(x_0)}
	\quad\tn{for any $x,y\in \R$.}
\end{equation*}
Clearly, $\phi_\eps\in \nnOneD{\varrho}{3}{1}{\R^2}{\R}$.
Moreover, for any $x,y\in [-A,A]$, 
applying Equation~\eqref{eq:2nd:taylor:expansion} with 
$z$ taken as
$x_0+\eps x+\eps  y$, $x_0+\eps x$, and $x_0+\eps y$ therein, we obtain
\begin{equation*}
    \begin{split}       &\phantom{=}\;\;\eps^2\varrho^\dprime(x_0)\phi_\eps(x,y)
    =\varrho(x_0+\eps  x+\eps y)-\varrho(x_0+\eps x)-\varrho(x_0+\eps y)+\varrho(x_0)
    \\
    &=\varrho(x_0)+\varrho^\prime(x_0)(\eps  x+  \eps 
 y)+\tfrac{\varrho^\dprime(x_0)}{2}(\eps  x+  \eps 
 y)^2+ h(x_0+\eps  x+\eps  y)(\eps  x+  \eps 
 y)^2
 \\
 &\phantom{=}\;\; -\Big(\varrho(x_0)+\varrho^\prime(x_0)(\eps  x)+\tfrac{\varrho^\dprime(x_0)}{2}(\eps  x)^2+ h(x_0+\eps  x)(\eps  x)^2\Big)
  \\
 &\phantom{=}\;\;-\Big(\varrho(x_0)+\varrho^\prime(x_0)( \eps 
 y)+\tfrac{\varrho^\dprime(x_0)}{2}( \eps 
 y)^2+ h(x_0+ \eps  y)(\eps 
 y)^2\Big) +\varrho(x_0)
 \\
 &=\varrho^\dprime(x_0)\eps ^2 xy+\Big(h(x_0+\eps  x+\eps  y)(\eps  x+  \eps 
 y)^2-
 h(x_0+\eps  x)(\eps  x)^2-h(x_0+\eps  y)(\eps 
 y)^2\Big),
    \end{split}
\end{equation*}
from which we deduce
\begin{equation*}
    \begin{split}       
    \phi_\eps(x,y)
 &=xy +\frac{h(x_0+\eps  x+\eps  y)(\eps  x+  \eps 
 y)^2-
 h(x_0+\eps  x)(\eps  x)^2-h(x_0+\eps  y)(\eps 
 y)^2}{\eps^2\varrho^\dprime(x_0)}
 \\
 &=xy +\frac{h(x_0+\eps  x+\eps  y)(x+y)^2-
 h(x_0+\eps  x)x^2-h(x_0+\eps  y)y^2}{\varrho^\dprime(x_0)}.
    \end{split}
\end{equation*}
It follows from $\lim_{\eta\to 0}h(x_0+\eta)=0$ that
\begin{equation*}
    h(x_0+\eps  x+\eps  y)(x+y)^2\rightrightarrows 0,\quad 
 h(x_0+\eps  x)x^2\rightrightarrows 0,\quad h(x_0+\eps  y)y^2\rightrightarrows 0
\end{equation*}
as $\eps\to 0^+$ for any $x,y\in[-A,A]$. Consequently, we get
	\begin{equation*}
		\phi_\eps(x,y)\rightrightarrows  xy\quad \tn{as $\eps\to0^+$}\quad \tn{for any $x,y\in[-A,A]$,}
	\end{equation*}
 which means we complete the proof of Lemma~\ref{lem:xy}.
\end{proof}

\begin{lemma}
	\label{lem:approx:ReLU}
	Given any $M>0$ and two functions $g_1,g_{2,\delta}:\R\to\R$ for each $\delta\in (0,1)$, suppose
	\begin{equation*}
		\sup_{x\in\R}|g_1(x)|<\infty,\quad \lim_{x\to-\infty}g_1(x)=0, \quad \lim_{x\to\infty}g_1(x)=1,
	\end{equation*}
	and 
		\begin{equation*}
 g_{2,\delta}(x)\rightrightarrows x\quad  \tn{as}\ \delta\to0^+ \quad \tn{for any $x\in[-M,M]$. }
	\end{equation*}
	Then for any $\eps>0$, there exist $K_\eps>0$ and $\delta_\eps\in (0,1)$ such that
	\begin{equation*}
		\big|g_1(K_\eps x)\cdot g_{2,\delta_\eps}(x)-\ReLU(x)\big|<\eps\quad \tn{for any $x\in [-M,M]$}.
	\end{equation*}
\end{lemma}
\begin{proof}
	Since $\sup_{x\in\R}|g_1(x)|<\infty$, $\lim_{x\to-\infty}g_1(x)=0$, and $
	\lim_{x\to\infty}g_1(x)=1$, 
 we have
 	\begin{equation*}
		K_0=\sup_{x\in\R}|g_1(x)|\in [1,\infty)
	\end{equation*}
and there exists $K_1>0$ such that 
	\begin{equation*}
		\big|g_1(y)\big|<\eps_1 \  \tn{for any $y\le -K_1/4$}\quad \tn{and}\quad \big|g_1(y)-1\big|<\eps_1  \  \tn{for any $y\ge K_1/4$},
	\end{equation*}
	where $\eps_1=\eps/(2M)$. It follows that
	\begin{equation}\label{eq:g1:minus:h}
		\big|g_1(K_0K_1x/\eps)-\one_{\{x>0\}}\big|<\eps_1=\eps/(2M) \quad \tn{if $|x|\ge \eps/(4K_0)$},
	\end{equation}
	Recall that $g_{2,\delta}(x )\rightrightarrows x$ as $\delta\to0^+$ for any $x\in [-M,M]$.
	There exists $\delta_\eps\in (0,1)$ such that
	\begin{equation}\label{eq:g2delta:minus:x}
		\big| g_{2,\delta_\eps}(x) -  x\big|<\eps_2=\eps/(3K_0)\quad \tn{for any $x\in [-M, M]$.}
	\end{equation}
	We observe that $\ReLU(x)=x\cdot \one_{\{x>0\}}$ for any $x\in\R$. Setting $K_\eps=K_0K_1/\eps$ and by Equation~\eqref{eq:g2delta:minus:x}, 
	for any $x\in [-M,M]$, 
	we have
	\begin{equation*}
		\begin{split}
			\big|g_1(K_\eps x)g_{2,\delta_\eps}(x)-\ReLU(x)\big|
			&= \big|g_1(K_\eps x)g_{2,\delta_\eps}(x)-x\cdot \one_{\{x>0\}}\big|\\
			&\le \big|g_1(K_\eps x)g_{2,\delta_\eps}(x)-x g_1(K_\eps  x)\big|
			+ \big|x  g_1(K_\eps x)-x\cdot \one_{\{x>0\}}\big|\\
			&\le \big|g_1(K_\eps  x)\big|\cdot\big|g_{2,\delta_\eps}(x)-x \big|
			+ |x|\cdot\big|g_1(K_\eps x)-\one_{\{x>0\}}\big|\\
			& \le K_0\cdot \eps_2 +|x|\cdot\big|g_1(K_0K_1 x/\eps)-\one_{\{x>0\}}\big|
   \\& = \eps/3 +|x|\cdot\big|g_1(K_0K_1 x/\eps)-\one_{\{x>0\}}\big|.
		\end{split}
	\end{equation*}
	In the case of $|x|<\eps/(4K_0)$, we have
		\begin{equation*}
		\begin{split}
			\big|g_1(K_\eps x)g_{2,\delta_\eps}(x)-\ReLU(x)\big|
			& \le  \eps/3+|x|\cdot\big|g_1(K_0K_1 x/\eps)-\one_{\{x>0\}}\big|\\
			&\le \eps/3+ \tfrac{\eps}{4K_0}\cdot (K_0+1)\le \eps/3+\eps/2<\eps.
		\end{split}
	\end{equation*}
	We may assume $\eps/(4K_0)\le M$ since the proof is complete if $\eps/(4K_0)>M$.
	In the case of $|x|\in [\eps/(4K_0), M]$, by Equation~\eqref{eq:g1:minus:h}, we have
			\begin{equation*}
		\begin{split}
			\big|g_1(K_\eps x)g_{2,\delta_\eps}(x)-\ReLU(x)\big|
			& \le \eps/3 +|x|\cdot\big|g_1(K_0K_1 x/\eps)- \one_{\{x>0\}}\big|\\
			&\le \eps/3 + M\cdot 
		\eps_1 = \eps/3+ M\cdot 
		\tfrac{\eps}{2M}= \eps/3+\eps/2<\eps
		\end{split}
	\end{equation*}
	Therefore, for any $x\in [-M,M]$, we have
			\begin{equation*}
		\begin{split}
			\big|g_1(K_\eps x)g_{2,\delta_\eps}(x)-\ReLU(x)\big|
			<\eps,
		\end{split}
	\end{equation*}
	which means we finish the proof of Lemma~\ref{lem:approx:ReLU}.
\end{proof}

\subsection{Proof of Proposition~\ref{prop:approx:ReLU:scrA:2:3} Based on Lemmas~\ref{lem:xy} and \ref{lem:approx:ReLU}}
\label{sec:proof:thm:approx:ReLU:with:lemmas}



Having established Lemmas~\ref{lem:xy} and \ref{lem:approx:ReLU} in Section~\ref{sec:lemmas:proof:prop:approx:ReLU:scrA:23}, we are now prepared to prove Proposition~\ref{prop:approx:ReLU:scrA:2:3}.

\begin{proof}[Proof of Proposition~\ref{prop:approx:ReLU:scrA:2:3}]
	For any $\eps\in (0,1)$, our goal is to construct 
 	\begin{equation*}
		\phi_\eps\in 
  \begin{cases}
      \nnOneD{\varrho}{1}{1}{\R}{\R} & \tn{if}\  \varrho\in \tildescrA_{2},\\
      \nnOneD{\varrho}{2}{1}{\R}{\R} & \tn{if}\  \varrho\in \scrA_{2},\\
      \nnOneD{\varrho}{3}{2}{\R}{\R} & \tn{if}\  \varrho\in \scrA_3
  \end{cases}
  \end{equation*}
  to approximate \ReLU\  well on $[-M,M]$.
We divide the proof into three cases: $\varrho\in\tildescrA_{2}$,
$\varrho\in\scrA_{2}$, 
and $\varrho\in \scrA_3$.

\mycase{1}{$\varrho\in \tildescrA_{2}$.}
Let us first consider the case of $\varrho\in \tildescrA_{2}$.
The fact $\varrho\in \tildescrA_{2}$ implies that there exist $\tildeb_0,\tildeb_1\in\R$ and $\tildeh:\R\to\R$ such that 
\begin{equation*}
    \sup_{x\in \R}|\tildeh(x)|<\infty,\quad  L_1=\lim_{x\to -\infty} \tildeh(x)\neq L_2=\lim_{x\to \infty} \tildeh(x),\quad 
L_1\cdot  L_2=0,
\end{equation*}
and
\begin{equation*}
    \varrho(x)=(x+\tildeb_0)\cdot \tildeh(x)+\tildeb_1\quad \tn{for any $x\in\R$.}
\end{equation*}
The equality \( L_1 \cdot L_2 = 0 \) indicates that at least one of the values, either \( L_1 \) or \( L_2 \), must be zero.
Set $w_0=L_1+L_2$, $b_0=(L_1+L_2)\tildeb_0$, $b_1=\tildeb_1$, and 
\begin{equation*}
    w_1=\one_{\{L_1=0\}}-\one_{\{L_1\neq 0\}}=\begin{cases}
        1 & \tn{if $L_1=0$},\\
        -1 & \tn{if $L_1\neq 0$}
    \end{cases}
    =\begin{cases}
        1 & \tn{if $L_1=0$},\\
        -1 & \tn{if $L_2 = 0$}.
    \end{cases}
\end{equation*}
Then for any $x\in\R$, we have 
\begin{equation*}
   \begin{split}
       \varrho(x)
    =(x+\tildeb_0)\cdot \tildeh(x)+\tildeb_1
    &=\Big((L_1+L_2)x+(L_1+L_2)\tildeb_0\Big)\cdot \frac{\tildeh(w_1^2x)}{L_1+L_2}+\tildeb_1\\
    &=(w_0x+b_0)\cdot h(w_1x)+b_1,
   \end{split}
\end{equation*}
where $h:\R\to\R$ is defined via 
\begin{equation*}
    h(x)\coloneqq\frac{\tildeh(w_1 x)}{L_1+L_2}\quad \tn{for any $x\in\R$.}
\end{equation*}
It is easy to verify that
$\sup_{x\in \R}|h(x)|<\infty$,
\begin{equation*}
    \lim_{x\to -\infty} h(x)=
   \lim_{x\to -\infty} \frac{\tildeh(w_1 x)}{L_1+L_2}
   =\begin{cases}
       \lim\limits_{x\to -\infty} \frac{\tildeh(x)}{L_1+L_2}=\frac{L_1}{L_1+L_2}=0 & \tn{if $L_1=0$}, \\
      \lim\limits_{x\to -\infty} \frac{\tildeh(-x)}{L_1+L_2}=\frac{L_2}{L_1+L_2}=0   & \tn{if $L_2 = 0$},
   \end{cases}
\end{equation*}
and 
\begin{equation*}
    \lim_{x\to \infty} h(x)=
   \lim_{x\to \infty} \frac{\tildeh(w_1 x)}{L_1+L_2}
   =\begin{cases}
       \lim\limits_{x\to \infty} \frac{\tildeh(x)}{L_1+L_2}=\frac{L_2}{L_1+L_2}=1 & \tn{if $L_1=0$}, \\
      \lim\limits_{x\to \infty} \frac{\tildeh(-x)}{L_1+L_2}=\frac{L_1}{L_1+L_2}=1   & \tn{if $L_2 = 0$}.
   \end{cases}
\end{equation*}
By defining an affine linear map $\calL(x)\coloneqq \tfrac{w_1}{|w_0w_1|}x-\tfrac{b_0}{w_0}$ for any $x\in\R$, we have
\begin{equation}\label{eq:varrho:calL:eq:tildescrA20}
    \varrho\circ\calL(x)=(w_0\calL(x)+b_0)\cdot h\big(w_1\calL(x)\big)+b_1=
    \tfrac{w_0w_1}{|w_0w_1|}x\cdot h\big(w_1\calL(x)\big)+b_1
\end{equation}
for any $x\in\R$.
To make use of Lemma~\ref{lem:approx:ReLU}, we define
\begin{equation*}
    g_1(x)\coloneqq h\big(w_1\calL(x)\big)= h\Big(\tfrac{w_1^2}{|w_0w_1|}x-\tfrac{w_1b_0}{w_0}\Big)\quad \tn{for any $x\in\R$}
\end{equation*}
and 
\begin{equation*}
    g_{2,\delta}(x)\coloneqq  x\quad \tn{ for any $x\in\R$ and each $\delta\in (0,1)$}. 
\end{equation*}
It is worth noting that $\tfrac{w_1^2}{|w_0w_1|}>0$. Consequently, we can deduce that
\begin{equation*}
	\sup_{x\in\R}|g_1(x)|<\infty,\quad \lim_{x\to-\infty}g_1(x)=0,\quad \tn{and}\quad
	\lim_{x\to\infty}g_1(x)=1.
\end{equation*}
According to  Lemma~\ref{lem:approx:ReLU},
there exist $K_\eps>0$ and $\delta_\eps\in (0,1)$ such that
	\begin{equation*}
		\big|g_1(K_\eps x)\cdot g_{2,\delta_\eps}(x)-\ReLU(x)\big|<\eps\quad \tn{for any $x\in [-M, M]$}.
	\end{equation*}
This means
	\begin{equation*}
		\Big|h\big(w_1\calL(K_\eps x)\big)\cdot x -\ReLU(x)\Big|<\eps\quad \tn{for any $x\in [-M,M]$}.
	\end{equation*}

Define 
\begin{equation*}
    \phi_\eps(x)\coloneqq \tfrac{|w_0w_1|}{w_0w_1}\tfrac{1}{K_\eps}\Big(\varrho\circ \calL(K_\eps  x)-b_1\Big)
    =\tfrac{|w_0w_1|}{w_0w_1}\tfrac{1}{K_\eps} \varrho\Big(\tfrac{w_1K_\eps}{|w_0w_1|}  x-\tfrac{b_0}{w_0}\Big)-\tfrac{|w_0w_1|}{w_0w_1}\tfrac{1}{K_\eps}b_1
\end{equation*}
\tn{for any $x\in \R.$}
Clearly, $\phi_\eps\in \nnOneD[]{\varrho}{1}{1}{\R}{\R}$.  
 Furthermore, based on Equation~\eqref{eq:varrho:calL:eq:tildescrA20}, for any $x\in [-M,M]$, we have
\begin{equation*}
    \begin{split}
        \big|\phi_\eps(x)-\ReLU(x)\big|
        &=\Big|\tfrac{|w_0w_1|}{w_0w_1}\tfrac{1}{K_\eps}\Big(\varrho\circ \calL(K_\eps  x)-b_1\Big)-\ReLU(x)\Big|
        \\
        &=\Big|\tfrac{|w_0w_1|}{w_0w_1}\tfrac{1}{K_\eps}
        \overbrace{
        \tfrac{w_0w_1}{|w_0w_1|}(K_\eps  x)\cdot h\big(w_1\calL(K_\eps x)\big)
        }^{= \varrho\circ\calL(K_\eps x)-b_1 \tn{ by \eqref{eq:varrho:calL:eq:tildescrA20}}
        }
        -\ReLU(x)\Big|
        \\
        &=\Big|x\cdot h\big(w_1\calL(K_\eps x)\big)-\ReLU(x)\Big|<\eps.
    \end{split}
\end{equation*}

\mycase{2}{$\varrho\in \scrA_{2}$.}
Next, let us consider the case of $\varrho\in \scrA_{2}$.
The fact $\varrho\in \scrA_{2}$ implies that there exist $b_0,b_1\in\R$ and $h:\R\to\R$ such that  
\begin{equation*}
    \sup_{x\in \R}|h(x)|<\infty,\quad  L_1=\lim_{x\to -\infty} h(x)\neq L_2=\lim_{x\to \infty} h(x),
\end{equation*}
and
\begin{equation*}
    \varrho(x)=(x+b_0)\cdot h(x)+b_1\quad \tn{for any $x\in\R$.}
\end{equation*}
By defining an affine linear map $\calL(x)\coloneqq x-b_0$ for any $x\in\R$, we have
\begin{equation}\label{eq:varrho:calL:eq:scrA20}
    \varrho\circ\calL(x)=(\calL(x)+b_0)\cdot h\big(\calL(x)\big)+b_1=
    x\cdot h\big(\calL(x)\big)+b_1
\end{equation}
for any $x\in\R$.
To make use of Lemma~\ref{lem:approx:ReLU}, we define 
\begin{equation*}
    g_1(x)\coloneqq \frac{h\big(\calL(x)\big)-L_1}{L_2-L_1}= \frac{h(x-b_0)-L_1}{L_2-L_1} \quad \tn{for any $x\in\R$}
\end{equation*}
and 
\begin{equation*}
    g_{2,\delta}(x)\coloneqq  x\quad \tn{ for any $x\in\R$ and each $\delta\in (0,1)$}. 
\end{equation*}
Consequently, we can deduce that
\begin{equation*}
	\sup_{x\in\R}|g_1(x)|<\infty,\quad \lim_{x\to-\infty}g_1(x)=\tfrac{L_1-L_1}{L_2-L_1}=0,\quad \tn{and}\quad
	\lim_{x\to\infty}g_1(x)=\tfrac{L_2-L_1}{L_2-L_1}=1.
\end{equation*}
By  Lemma~\ref{lem:approx:ReLU} and setting $\tildeM=2M>0$,
there exist $K_\eps>0$ and $\delta_\eps\in (0,1)$ such that
	\begin{equation*}
		\big|g_1(K_\eps x)\cdot g_{2,\delta_\eps}(x)-\ReLU(x)\big|<\eps/2\quad \tn{for any $x\in [-\tildeM, \tildeM]=[-2M, 2M]$}.
	\end{equation*}
This means
	\begin{equation}
 \label{eq:g-L1:over:L2-L1-ReLU:up}
		\bigg|\frac{h\big(\calL(K_\eps 
 x)\big)-L_1}{L_2-L_1}\cdot x -\ReLU(x)\bigg|<\eps/2\quad \tn{for any $x\in [-2M,2M]$}.
	\end{equation}

Define 
\begin{equation*}
    \psi_\eps(x)\coloneqq \tfrac{1}{L_2-L_1}\tfrac{1}{K_\eps}\Big(\varrho\circ \calL(K_\eps  x)-b_1\Big)
    =\tfrac{1}{L_2-L_1}\tfrac{1}{K_\eps} \varrho(K_\eps   x-b_0)-\tfrac{1}{L_2-L_1}\tfrac{1}{K_\eps}b_1
\end{equation*}
and 
\begin{equation*}
    \psi(x)=\tfrac{L_1}{L_2-L_1}\cdot  x+\ReLU(x)\quad \tn{for any $x\in\R$.}
\end{equation*}
Clearly, $\psi_\eps\in \nnOneD[]{\varrho}{1}{1}{\R}{\R}$.  
 Furthermore, based on Equation~\eqref{eq:varrho:calL:eq:scrA20}, for any $x\in [-2M,2M]$, we have
\begin{equation*}
    \begin{split}
        \big|\psi_\eps(x)-\psi(x)\big|
        &=\Big|\tfrac{1}{L_2-L_1}\tfrac{1}{K_\eps}\Big(\varrho\circ \calL(K_\eps  x)-b_1\Big)
        -\psi(x)\Big|
        \\
        &=\Big|\tfrac{1}{L_2-L_1}\tfrac{1}{K_\eps}
        \overbrace{
        (K_\eps  x)\cdot h\big(\calL(K_\eps x)\big)
        }^{=\varrho\circ\calL(K_\eps x) -b_1
        \tn{  by   \eqref{eq:varrho:calL:eq:scrA20}}
        }
        -\psi(x)\Big|
        \\
        &=\Big|\tfrac{x\cdot h(\calL(K_\eps x))}{L_2-L_1}-\big(\tfrac{L_1}{L_2-L_1}\cdot 
 x+\ReLU(x)\big)\Big|
 \\
 &=\Big|\tfrac{h(\calL(K_\eps x))-L_1}{L_2-L_1}\cdot x-\ReLU(x)\Big|<\eps/2,
    \end{split}
\end{equation*}
where the last inequality comes from Equation~\eqref{eq:g-L1:over:L2-L1-ReLU:up}.
Then for any $x\in [-M,M]$, we have $x-M\in [-2M,0]$ and hence $\ReLU(x-M)=0$, from which we deduce
\begin{equation*}
\begin{split}
    	\eps/2 
     & > 
     \Big|\psi_\eps(x-M)-\psi(x-M)\Big|
      =\Big|\psi_\eps(x-M)-\Big(\tfrac{L_1}{L_2-L_1}\cdot (x-M)+\ReLU(x-M)\Big)\Big|\\
     & =\Big|\psi_\eps(x-M)-\tfrac{L_1}{L_2-L_1}\cdot (x-M)\Big|
     =\Big|\psi_\eps(x-M)+\tfrac{L_1M}{L_2-L_1}-\tfrac{L_1}{L_2-L_1}\cdot x\Big|.
\end{split}
\end{equation*}

Define
\begin{equation*}
	\begin{split}
	    \phi_\eps(x)\coloneqq \psi_\eps(x)-\Big(\psi_\eps(x-M)+\tfrac{L_1M}{L_2-L_1}\Big)\quad \tn{for any $x\in\R$.}
	\end{split}
\end{equation*}
It follows from $\psi_\eps\in \nnOneD[]{\varrho}{1}{1}{\R}{\R}$ that $\phi_\eps\in \nnOneD[]{\varrho}{2}{1}{\R}{\R}$. Moreover, we have
\begin{equation*}
	\begin{split}
		|\phi_\eps(x)-\ReLU(x)|
		&=  \bigg|
  \underbrace{
  \psi_\eps(x)-\Big(\psi_\eps(x-M)+\tfrac{L_1M}{L_2-L_1}\Big)
 }_{\phi_\eps}
 -\Big(
 \underbrace{
\psi(x)-\tfrac{L_1}{L_2-L_1}\cdot x
 }_{\ReLU}\Big)\bigg|\\
		& \le  \big|\psi_\eps(x)-\psi(x)\big|
  +\Big|\Big(\psi_\eps(x-M)+\tfrac{L_1M}{L_2-L_1}\Big)-\tfrac{L_1}{L_2-L_1}\cdot x\Big|\\
		&<   \eps/2+\eps/2=\eps.
	\end{split}
\end{equation*}

\mycase{3}{$\varrho\in \scrA_3$.}
Finally, let us now turn to the case of $\varrho\in \scrA_3$.
Clearly,
	 we have $\sup_{x\in\R}|\varrho(x)|<\infty$, $\varrho^\dprime(x_0)\neq 0$  for some $x_0\in\R$, and
	\begin{equation*}
		L_1=\lim_{x\to -\infty} \varrho(x)\neq L_2=\lim_{x\to\infty} \varrho(x).
	\end{equation*}
By defining 
\begin{equation*}
	g_1(x)\coloneqq \frac{\varrho(x)-L_1}{L_2-L_1}\quad \tn{for any $x\in\R$},
\end{equation*} we have 
\begin{equation*}
	\sup_{x\in\R}|g_1(x)|<\infty,\quad \lim_{x\to-\infty}g_1(x)=0,\quad \tn{and}\quad
	\lim_{x\to\infty}g_1(x)=1.
\end{equation*}
	
Since $\varrho^\dprime(x_0)\neq 0$,
	there exists $x_1$ such that 
	$\varrho^\prime(x_1)\neq 0$.
	For each $\delta\in(0,1)$, we define
	\begin{equation*}
		g_{2,\delta}(x)\coloneqq\frac{\varrho(x_1+\delta x)-\varrho(x_1)}{\delta\varrho^\prime(x_1)}\quad \tn{for any $x\in \R$.}
	\end{equation*}
    By Lemma~\ref{lem:approx:x},
	\begin{equation*}
		g_{2,\delta}(x)\rightrightarrows  x\quad \tn{as}\   \delta\to 0^+\quad \tn{for any $x\in [-M,M]$.}
	\end{equation*}
	By Lemma~\ref{lem:approx:ReLU}, there exist $K_\eps>0$ and $\delta_\eps\in (0,1)$ such that
		\begin{equation}\label{eq:g1g2delta:minus:relu:scrA3}
		\big|g_1(K_\eps x)\cdot g_{2,\delta_\eps}(x)-\ReLU(x)\big|<\eps\quad \tn{for any $x\in [-M,M]$.}
	\end{equation}
	 The fact $\sup_{x\in\R}|\varrho(x)|<\infty$ implies
	\begin{equation*}
		\begin{split}
					A=
			&\sup_{x\in [-M,M]} \max \big\{|g_1(K_\eps x)|,\,| g_{2,\delta_\eps}(x)|\big\}\\
			=
			&\sup_{x\in [-M,M]} \max \left\{\big| \tfrac{\varrho(K_\eps  x)-L_1}{L_2-L_1}\big|,\, \big| \tfrac{\varrho(x_1+\delta_\eps  x)-\varrho(x_1)}{\delta_\eps\varrho^\prime(x_1)}\big|\right\}<\infty.
		\end{split}
	\end{equation*}
	Since $\varrho^\dprime(x_0)\neq 0$, by Lemma~\ref{lem:xy}, 
	there exists
	\begin{equation*}
		\Gamma_\eta\in \nnOneD{\varrho}{3}{1}{\R^2}{\R}\quad \tn{for each $\eta\in (0,1)$}
	\end{equation*}  
	such that
	\begin{equation*}
		\Gamma_{\eta}(u,v)\rightrightarrows  uv\quad \tn{as $\eta\to0^+$}\quad \tn{for any $u,v\in[-A,A]$.}
	\end{equation*}
	Then there exists $\eta_\eps\in (0,1)$ such that 
	\begin{equation*}
		|\Gamma_{\eta_\eps}(u,v)-uv|<\eps\quad \tn{for any $u,v\in [-A,A]$,}
	\end{equation*} implying
	\begin{equation}\label{eq:gamma:g1g2delta:scrA3}
		\Big|\Gamma_{\eta_\eps}\Big(g_1(K_\eps x),\, g_{2,\delta_\eps}(x)\Big)-g_1(K_\eps x)\cdot g_{2,\delta_\eps}(x)\Big|<\eps \quad \tn{for any $x\in [-M,M]$.}
	\end{equation}	 Define
	\begin{equation*}
		\phi_\eps(x)\coloneqq \Gamma_{\eta_\eps}\Big(g_1(K_\eps x),\, g_{2,\delta_\eps}(x)\Big)\quad \tn{for any $x\in\R$.}
	\end{equation*}	
	Next, by Equations~\eqref{eq:g1g2delta:minus:relu:scrA3} and \eqref{eq:gamma:g1g2delta:scrA3}, we have
	\begin{equation*}
		\begin{split}
			&\big|\phi_\eps(x)-\ReLU(x)\big|
			= \Big|\Gamma_{\eta_\eps}\Big(g_1(K_\eps x),\, g_{2,\delta_\eps}(x)\Big)-\ReLU(x)\Big|\\
				\le \;	&\Big|\Gamma_{\eta_\eps}\Big(g_1(K_\eps x),\, g_{2,\delta_\eps}(x)\Big)-g_1(K_\eps x)\cdot g_{2,\delta_\eps}(x)\Big|
				  +\Big|g_1(K_\eps x)\cdot g_{2,\delta_\eps}(x)-\ReLU(x)\Big|\\
			< \;&\eps +\eps=2\eps
		\end{split}
	\end{equation*}	
	for any $x\in [-M,M]$, from which we deduce
		\begin{equation*}
		\phi_\eps(x)\rightrightarrows \ReLU(x)\quad \tn{as}\   \eps\to 0^+\quad \tn{for any $x\in [-M,M]$.}
	\end{equation*}
	It remains to show 
	$\phi_\eps\in \nnOneD{\varrho}{3}{2}{\R}{\R}$. By defining
	\begin{equation*}
		\bmpsi_\eps(x)\coloneqq \Big(\tfrac{\varrho(K_\eps  x)-L_1}{L_2-L_1},\, \tfrac{\varrho(x_1+\delta_\eps  x)-\varrho(x_1)}{\delta_\eps\varrho^\prime(x_1)}\Big)\quad \tn{for any $x\in\R$,}
	\end{equation*}	
	we have 
	$\bmpsi_\eps\in \nnOneD{\varrho}{2}{1}{\R}{\R^2}$ and
		\begin{equation*}
		\phi_\eps(x)= \Gamma_{\eta_\eps}\Big(g_1(K_\eps x),\, g_{2,\delta_\eps}(x)\Big)=\Gamma_{\eta_\eps}\Big(\tfrac{\varrho(K_\eps  x)-L_1}{L_2-L_1},\, \tfrac{\varrho(x_1+\delta_\eps  x)-\varrho(x_1)}{\delta_\eps\varrho^\prime(x_1)}\Big)=\Gamma_{\eta_\eps}\circ \bmpsi_\eps(x)
	\end{equation*}		
	for any $x\in\R$. Recall that $\Gamma_{\eta_\eps}\in \nnOneD{\varrho}{3}{1}{\R^2}{\R}$. 
	Hence, we can conclude that $\phi_\eps\in \nnOneD{\varrho}{3}{2}{\R}{\R}$. This result completes the proof of Proposition~\ref{prop:approx:ReLU:scrA:2:3}.
\end{proof}


\long\def\acks#1{\vskip 0.3in\noindent{\large\bf Acknowledgments}\vskip 0.2in
\noindent #1}
\acks{Jianfeng Lu was partially supported by
NSF grants CCF-1910571 and 
DMS-2012286.
Hongkai Zhao was partially supported by NSF grants DMS-2012860 and DMS-2309551.}

\let\oldhref\href
\def\href#1#2{\oldhref{#1}{\nolinkurl{#2}}}
\bibliography{references}
\end{document}